%% file: main.tex
\documentclass{article}

\usepackage[main, final]{neurips_2026}

\usepackage[utf8]{inputenc} 
\usepackage[T1]{fontenc}    
\usepackage{hyperref}       
\usepackage{url}            
\usepackage{booktabs}       
\usepackage{amsfonts}       
\usepackage{nicefrac}       
\usepackage{amsmath, amsthm, amssymb, centernot}
\usepackage{accents}
\usepackage{mathrsfs}
\usepackage[table]{xcolor} 
\usepackage{array} 
\usepackage{caption}
\usepackage{graphicx}
\usepackage{epstopdf}
\usepackage{microtype}
\usepackage{subcaption}
\usepackage{bm} 

\newcommand{\dbtilde}[1]{\accentset{\approx}{#1}}
\newtheorem{define}{Definition}
\newtheorem{theorem}{Theorem}
\newtheorem{prop}{Proposition}

\newtheorem{lemma}{Lemma}
\newtheorem{coroll}{Corollary}

\DeclareMathOperator*{\argmin}{arg\,min}

\input{figures/caption}
\input{tables/main}

\input{tables/supp}

\newcounter{contrib}
\newcommand{\contrib}{\stepcounter{contrib}(\thecontrib)~}

\title{A Biconvex Formulation for Stable Transport of Mixture Models with a Unique Solution}

\author{%
  Yeganeh Marghi, Kelly Jin, Uygar S\"{u}mb\"{u}l \\
  \\
    Allen Institute\\
    Seattle, WA, USA\\
}

\begin{document}

\maketitle

\begin{abstract}
Optimal transport (OT) provides a principled framework for mapping between probability distributions. Despite extensive progress, applying OT to large-scale data remains computationally demanding, and the resulting pointwise transport plans are often difficult to interpret. We introduce Optimal Mixture Transport (OMT), a scalable framework that shifts the transport paradigm from individual samples to mixtures of subpopulations, reformulating the transport problem as a strictly biconvex optimization with a unique global minimizer. We further establish theoretical guarantees on the stability of the OMT map, showing that bounded perturbations of the underlying distributions lead to bounded changes in the transport plan. By formulating subpopulations as exponential-family distributions, OMT decouples computational complexity from the sample size, scaling solely with the number of mixture components.
We demonstrate the effectiveness and practicality of OMT on a wide range of synthetic benchmarks and real-world datasets, including image data and large-scale single-cell RNA sequencing measurements.
\end{abstract}
\vspace{-.1in}
\section{Introduction}
\label{sec:Introduction}
\vspace{-.1in}
\input{main/1-introduction}
\vspace{-.1in}
\section{Background}
\label{sec:Background}
\vspace{-.1in}
\input{main/2-background}
%
\vspace{-.25in}
\section{Transport problem for mixture models}
\label{sec:method}
\vspace{-.1in}
\input{main/4-method}
\vspace{-.1in}
\section{Related Work}
\label{sec:related_work}
\vspace{-.1in}
\input{main/3-relatedwork}
\vspace{-.1in}
\section{Experiments}
\label{sec:experiments}
\vspace{-.1in}
\input{main/5-experiment}

\vspace{-.12in}
\section{Conclusion}
\label{sec:Conclusion}
\vspace{-.13in}
\input{main/6-conclusion}
\begin{ack}
We thank Yuan Gao for assistance with mouse visual cortex data, and Cindy van Velthoven and Zizhen Yao for insightful feedback on the analysis of mouse brain data across developmental and aging stages. This work was supported by the Allen Institute.
\end{ack}
\bibliographystyle{unsrt}
\bibliography{main}

\setcounter{theorem}{0}
\setcounter{prop}{0}
\setcounter{axiom}{0}
\setcounter{lemma}{0}
\setcounter{remark}{0}
\newpage
\appendix
{\Large\centering\bfseries Appendix\par}
\section{Proofs}
\label{sec:proofs}
\input{appendix/1-proofs}
\input{appendix/2-expfamily}
\newpage
\input{appendix/3-moreresults}
\end{document}

%% file: figures/caption.tex
\newcommand{\OMTmainFig}{
\begin{figure}[t!]
\centering
\vspace{-5mm}
\includegraphics[width=\columnwidth,trim=4pt 4pt 10pt 2pt,clip]{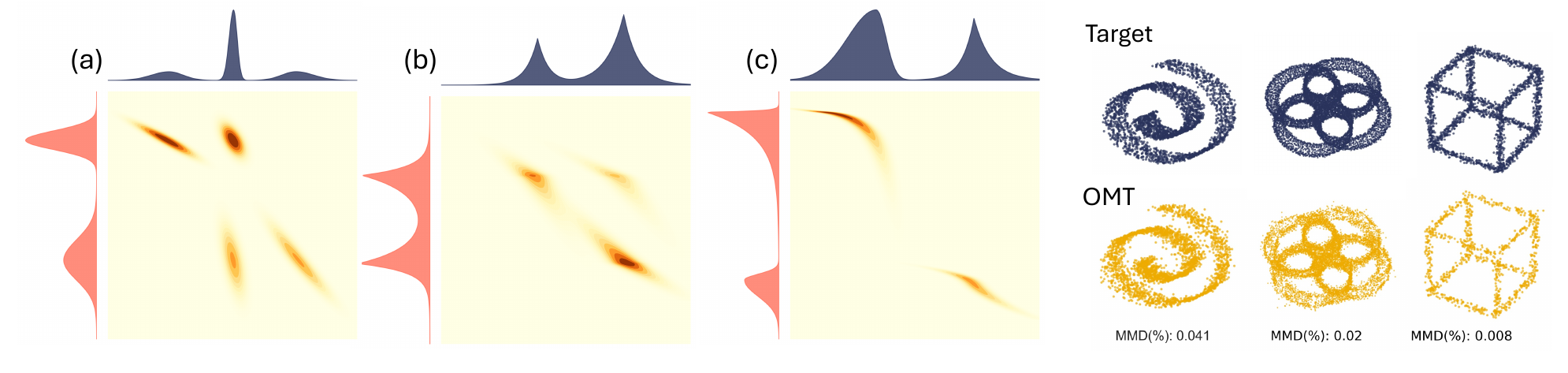}
  \caption{\footnotesize{(Left) OMT couplings obtained under three parameter configurations: (a) symmetric components with $p_i=q_i=2$, (b) symmetric components with $p_i=q_i=1$, and (c) a mixture of asymmetric components. Marginal densities corresponding to the target and source distributions are shown in blue and red, respectively. 
(Right) Sample transport from a normal source distribution to multiple target distributions using OMT.}}
\label{fig:OMTmainFig}
\vspace{-5mm}
\end{figure}
}

\newcommand{\AppBenchMark}{
\begin{figure}[t!]
\centering
\includegraphics[width=\columnwidth]{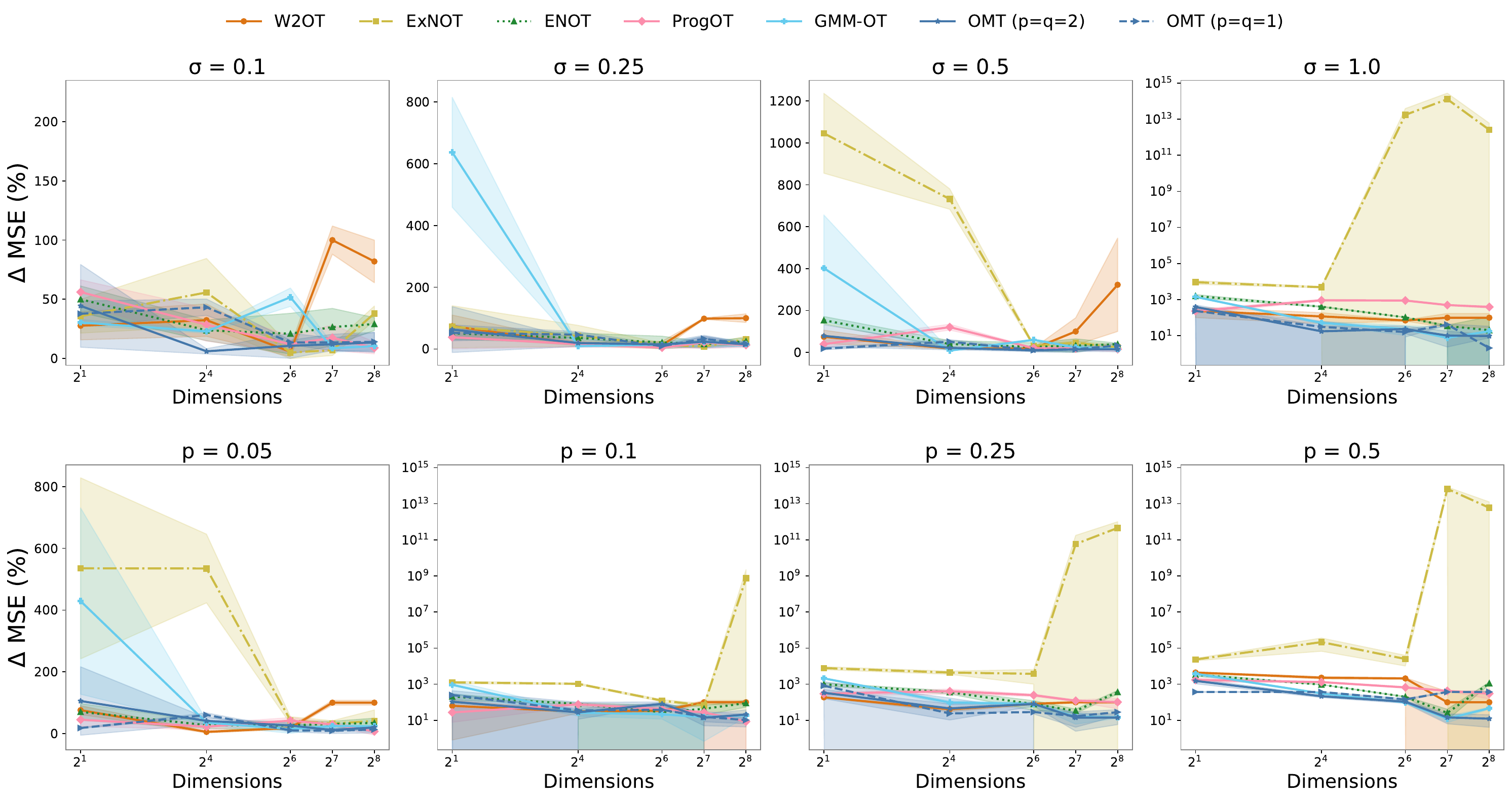}
\caption{\footnotesize{Detailed stability analysis of OT solvers across varying dimensions on the $W_2$ benchmark. Expanding upon the aggregated results in the main text, this figure illustrates the relative performance degradation ($\Delta \text{ MSE}$) as a function of data dimensionality $d \in \{2, 16, 64, 128, 256\}$. The top row presents the models' robustness against additive Gaussian noise at increasing standard deviations $\sigma \in \{0.1, 0.25, 0.5, 1.0\}$. The bottom row details performance under drop-out noise at increasing probabilities $p \in \{0.05, 0.1, 0.25, 0.5\}$. Lines and markers indicate the mean $\Delta \text{ MSE}$ across 10 random initializations (evaluated on a test set of 10,000 samples), with shaded regions denoting the standard deviation. Note that the y-axis transitions to a logarithmic scale in higher-noise regimes to display the extreme degradation observed in certain baselines (e.g., ExNOT). Consistent with the main text, the number of mixture components is $(K_0=3, K_1=15)$ for OMT ($p=q=2$) and $(K_0=10, K_1=30)$ for OMT ($p=q=1$).}}
\label{fig:AppBenchMark}
\end{figure}
}

\newcommand{\RunTime}{
\begin{figure}[t!]
\centering
\begin{minipage}{0.6\textwidth}
  \centering
\includegraphics[width=\columnwidth]{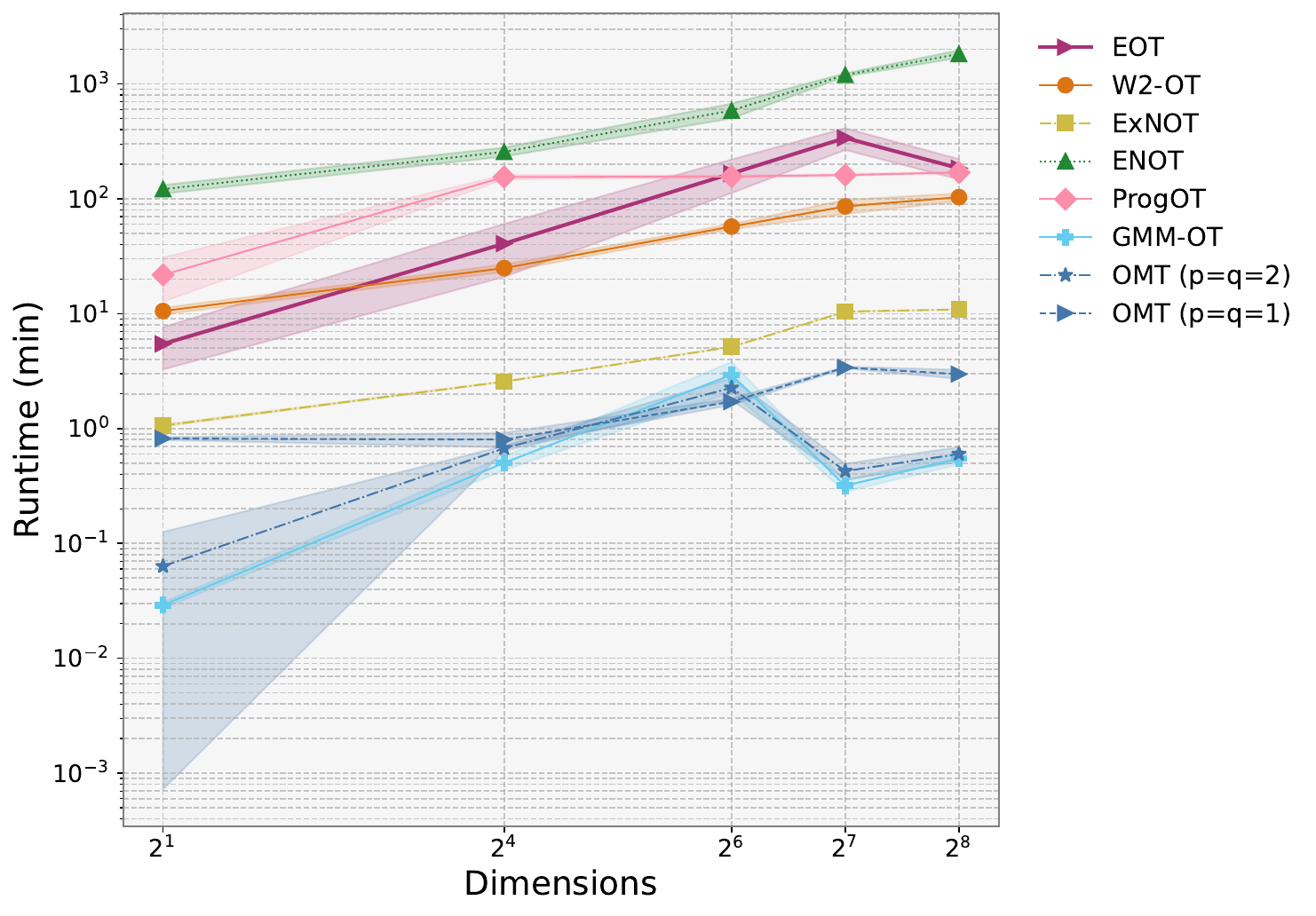}
\end{minipage}\hfill
\begin{minipage}{0.38\textwidth}
  \captionof{figure}{{\footnotesize Runtime comparison of OT solvers on the Wasserstein-2 benchmark task. Reported times correspond to the optimization of the transport plan. For OMT, runtime additionally includes the cost of fitting mixture components to the source and target measures. We further include EOT as a representative sample-to-sample method based on the Sinkhorn algorithm. All experiments are conducted on a compute node equipped with an NVIDIA A100 GPU, 4 Intel Xeon Gold 6330N CPU cores, and 128 GB RAM. Reported values are averaged over 10 independent runs.}}
\label{fig:RunTime}
\end{minipage}
\end{figure}
}

\newcommand{\scRNAseq}{
\begin{figure}[t!]
\centering
\vspace{-5mm}
    \begin{subfigure}{\columnwidth}
        \centering
\includegraphics[width=\columnwidth,trim=4pt 5pt 4pt 4pt,clip]{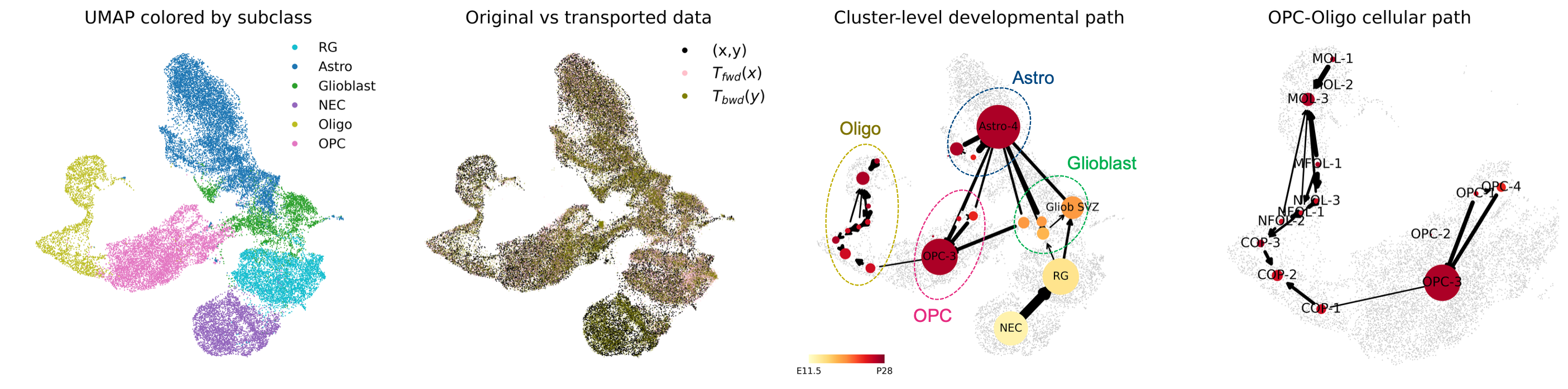}
    \end{subfigure}
    \vspace{-2mm} 
    \begin{subfigure}{\columnwidth}
        \centering
    \includegraphics[width=\columnwidth]{figures/main/opc_ageing_path.jpg}
    \end{subfigure}
\vspace{-.2in}
\caption{\footnotesize{OPC–Oligo trajectories across the mouse lifespan. Top row: developmental dataset from the mouse visual cortex. From left to right: (1) UMAP projection showing distinct neural cell subclasses. (2) The alignment between the original measured data and the transferred data using OMT. (3) The inferred global developmental trajectory at the cluster level, tracing paths from early progenitors like neuroepithelial cells (NEC) and radial glia (RG). (4) The specific cellular pathway detailing the differentiation from OPC to oligodendrocytes.
Bottom row: mouse aging dataset. From left to right: (1) UMAP projection of cell subtypes within the oligodendrocyte lineage. (2) The alignment of original and transferred data distributions. (3) The network graph illustrating the stages of the myelination cycle in aged mice. (4) The inferred aging pathway.}}
\vspace{-.2in}
\label{fig:scRNAseq}
\end{figure}
}

\newcommand{\AppMERFISHGMMs}{
\begin{figure}[t!]
\centering
\includegraphics[width=.95\columnwidth, trim=10pt 5pt 10pt 10pt,clip]{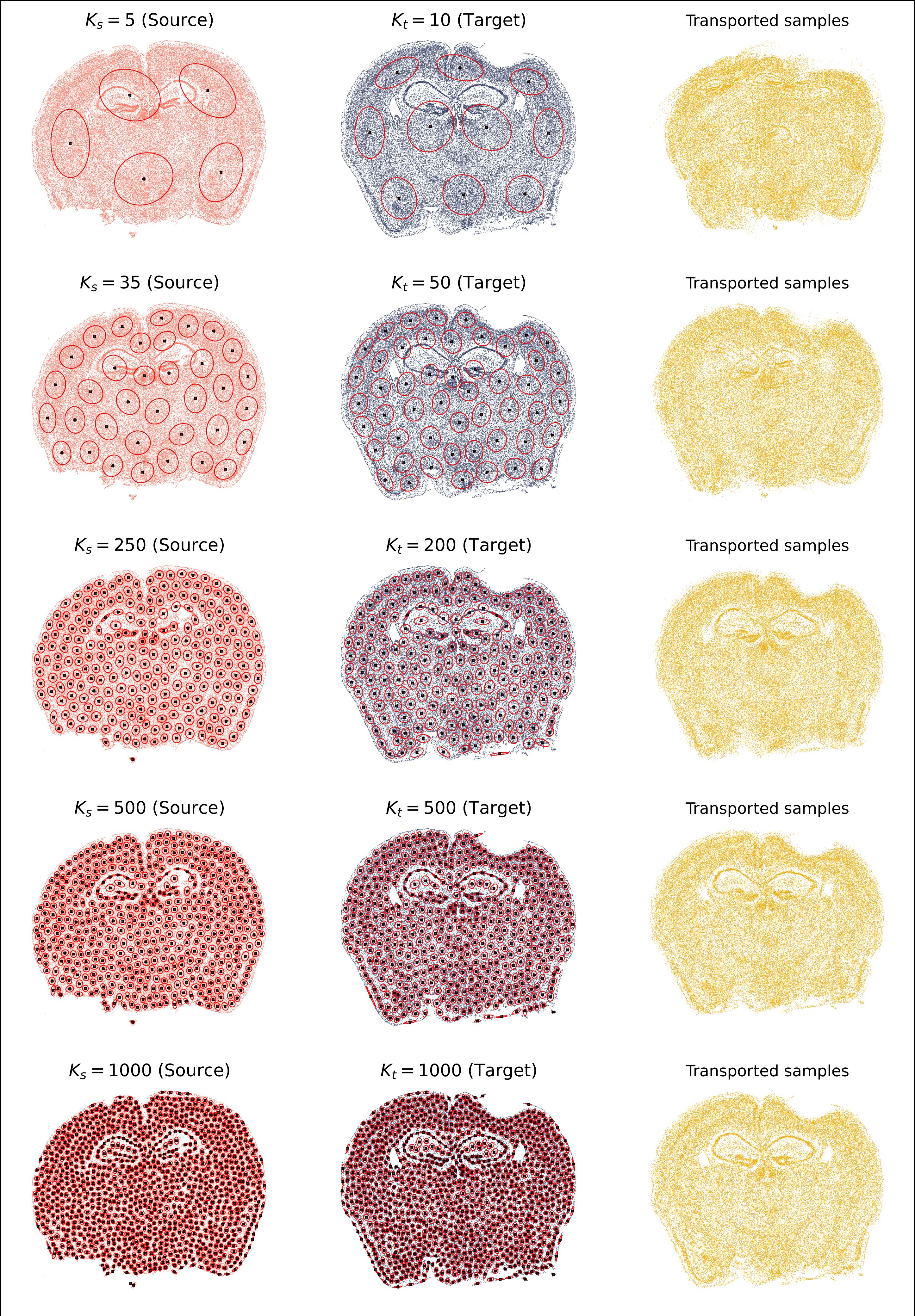}
\caption{\footnotesize{Impact of number of Gaussian components on OMT performance. The source (red) and target (blue) cell distributions are approximated using Gaussian mixtures with increasing numbers of source ($K_s$) and target ($K_t$) components. The rows show the progression of fitting accuracy as the number of components increases from $K_s=5$ (top) to $K_s=1000$ (bottom). (Left) Source cells (red dots) and (Center) target cells (blue dots) overlaid with their respective Gaussian components. (Right) The resulting transported samples (yellow dots).}}
\label{fig:AppMERFISHGMMs}
\end{figure}
}

\newcommand{\AppMERFISHgenes}{
\begin{figure}[t!]
\centering
\includegraphics[width=.95\columnwidth,trim=10pt 5pt 10pt 10pt,clip]{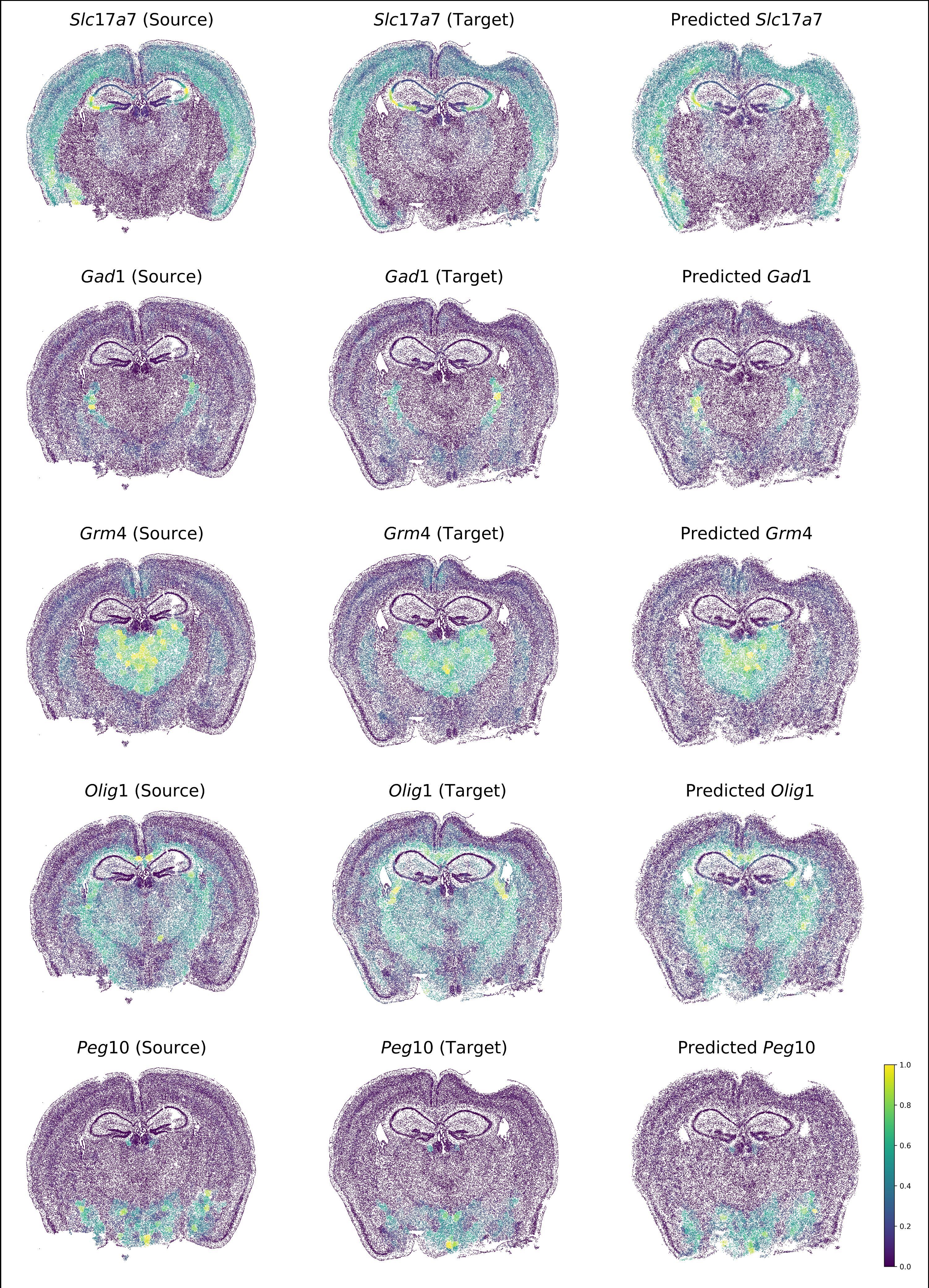}
\caption{\footnotesize{Spatial gene expression imputation using OMT with $K_s=1000$, $K_t=1000$. Rows show expression maps for five distinct genes: \textit{Slc17a7}, \textit{Gad1}, \textit{Grm4}, \textit{Olig1}, and \textit{Peg10}. From left to right, the plots illustrate the source distribution, the target ground truth, and the predicted expression after transport, demonstrating how well the transported expression (right) aligns with the ground-truth target (middle). Here, the cosine similarity exceeds $0.75$, with a total displacement cost of $107.69$.}}
\label{fig:AppMERFISHgenes}
\end{figure}
}

\newcommand{\ImageMNISTCIFAR}{
\begin{figure}[h!]
\centering
\vspace{-.1in}
\begin{minipage}{0.5\textwidth}
  \centering
\includegraphics[width=\columnwidth]{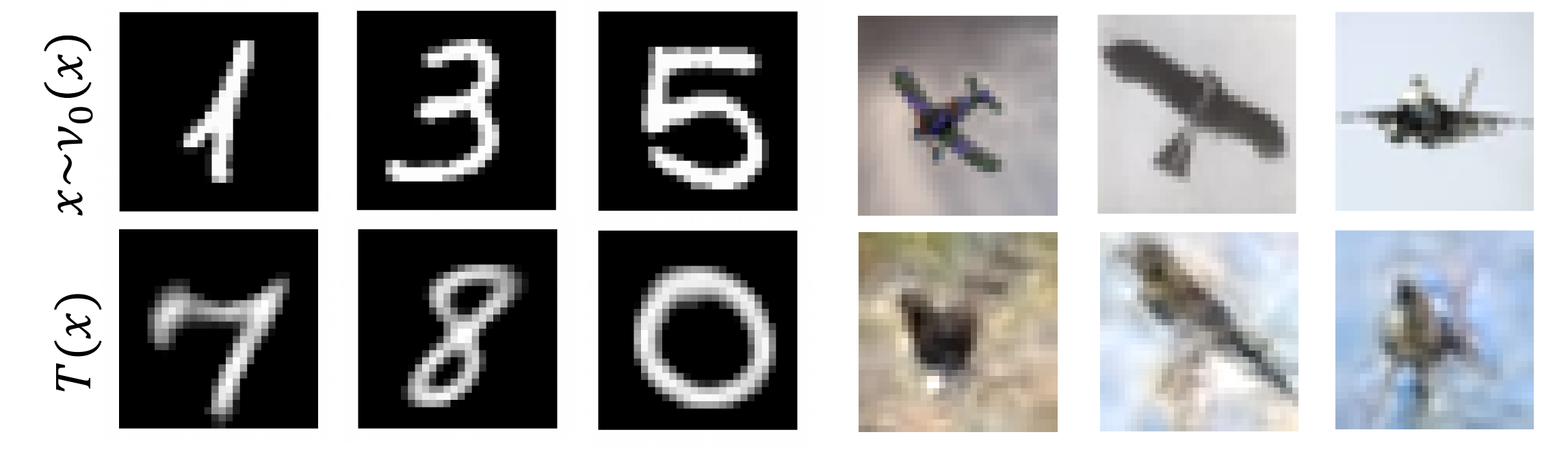}
\end{minipage}\hfill
\begin{minipage}{0.48\textwidth}
  \captionof{figure}{\footnotesize Performance of OMT  for unpaired image-to-image translation on the MNIST and CIFAR-10 datasets. For each dataset, the top row shows original samples from the source distribution, $x \sim \nu_0$, and the bottom row shows the corresponding transported images $T^{\nu_0 \rightarrow \nu_1}_{\textsc{omt}}(x)$.}
\label{fig:ImageMNISTCIFAR}
\end{minipage}
\vspace{-5mm}
\end{figure}
}

\newcommand{\AppImagTran}{
\begin{figure}[t!]
\centering
\includegraphics[width=\columnwidth,trim=4pt 10pt 4pt 4pt,clip]{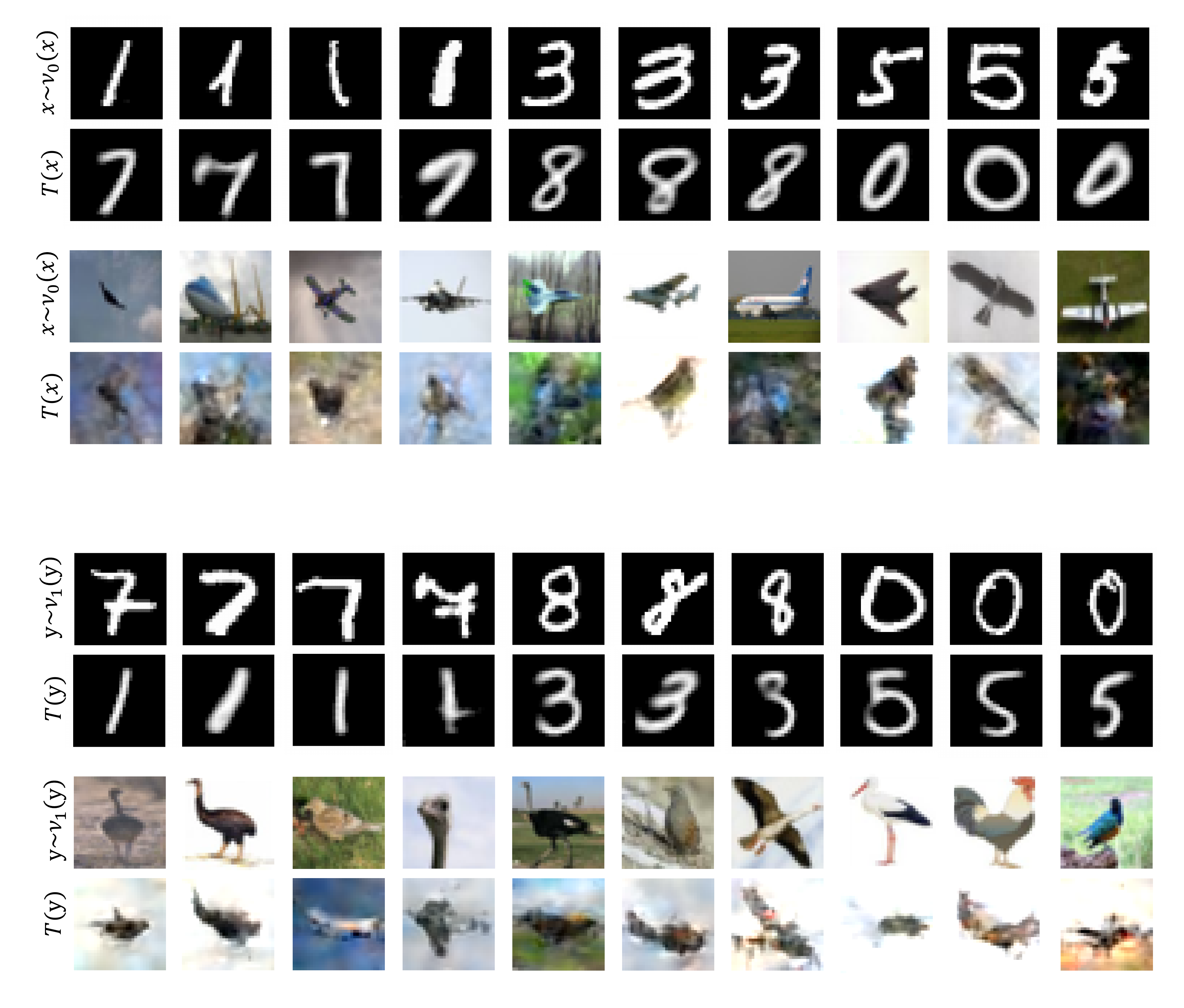}
\caption{\footnotesize{Unpaired image-to-image translation on the MNIST and CIFAR-10 datasets, showcasing OMT's bidirectional mapping capabilities. The figure is organized to demonstrate both transport directions. The top panels illustrate the forward translation path from the source to the target domain, displaying original source images ($x \sim \nu_0$) alongside their generated target counterparts ($T^{\nu_0 \rightarrow \nu_1}_{\textsc{omt}}(x)$). Conversely, the bottom panels illustrate the backward translation path, displaying original target images ($y \sim \nu_1$) and their corresponding translations back into the source domain ($T^{\nu_1 \rightarrow \nu_0}_{\textsc{omt}}(y)$). These results confirm the framework's effectiveness in learning structurally coherent mappings in both directions.}}
\label{fig:AppImagTran}
\end{figure}
}

\newcommand{\Noize}{
\begin{figure}[t!]
\centering
\vspace{-.1in}
\includegraphics[width=\columnwidth,trim=4pt 5pt 4pt 2pt,clip]{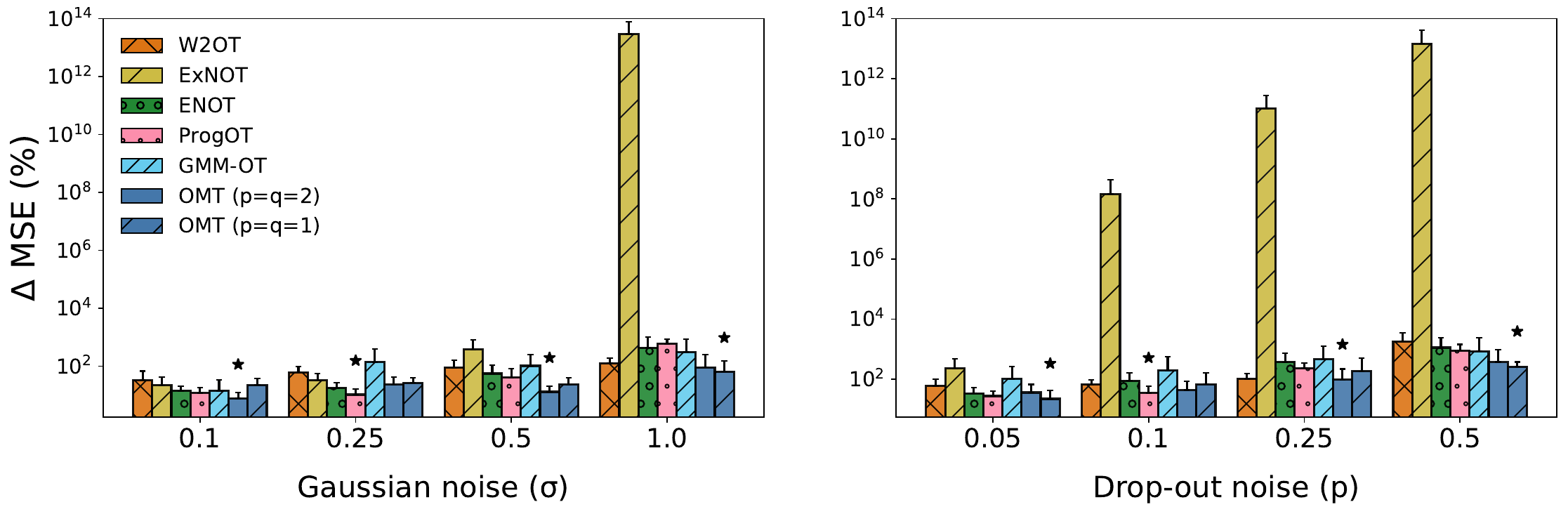}
  \captionof{figure}{{\footnotesize Stability of OT solvers under noise on the $W_2$ benchmark task. Methods are evaluated under two types of perturbations applied to the source data: Gaussian noise (left) and drop-out noise (right). The reported values measure the relative performance degradation, i.e., $\Delta \text{ MSE}$, with respect to the noise-free setting, computed on a test set of 10,000 samples and averaged over 10 random initializations. The number of components for OMT ($p=q=2$) is $(K_0=3, K_1=15)$, while for OMT ($p=q=1$) it is $(K_0=10, K_1=30)$. The asterisk ($\ast$) marks the most robust method in each noise regime.}}
\label{fig:Noise}
\vspace{-.2in}
\end{figure}
}

\newcommand{\MERFISH}{
\begin{figure}[t!]
\centering
\vspace{-.3in}
\begin{minipage}{0.6\textwidth}
  \centering
\includegraphics[width=\columnwidth]{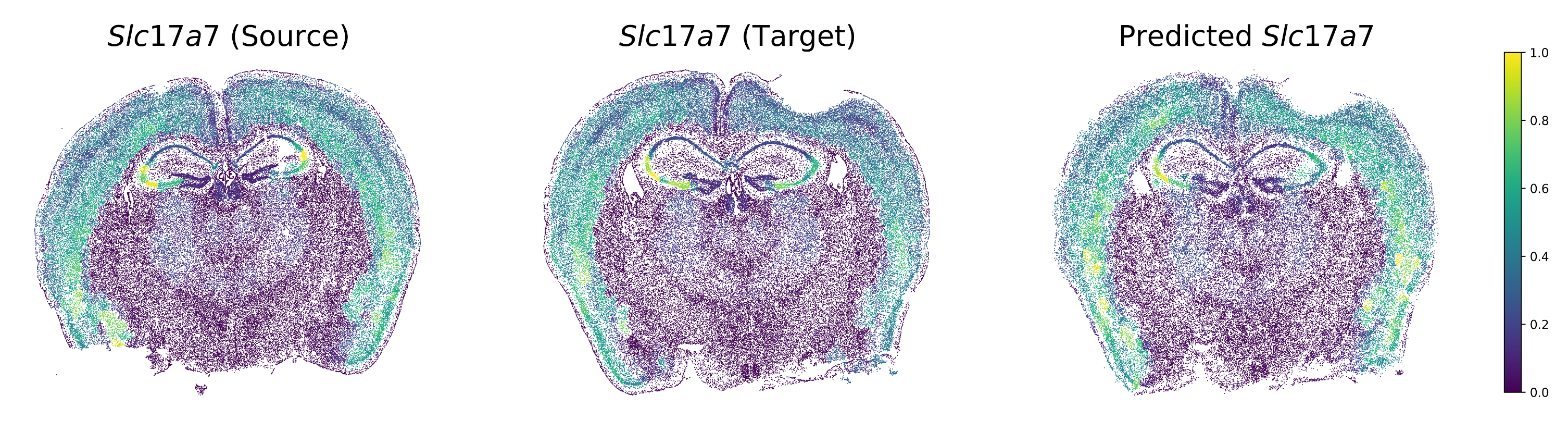}
\end{minipage}\hfill
\begin{minipage}{0.38\textwidth}
  \captionof{figure}{{\footnotesize Cellular alignment across MERFISH mouse brain data using OMT on spatial coordinates. From left to right, the expression profiles of the marker gene \textit{Slc17a7} are illustrated in the source, target, and OMT-transported cells within the brain tissue.}}
\label{fig:MERFISH}
\end{minipage}
\vspace{-.3in}
\end{figure}
}

\newcommand{\AppsciplexPCA}{
\begin{figure}[h!]
\centering
\includegraphics[width=\columnwidth]{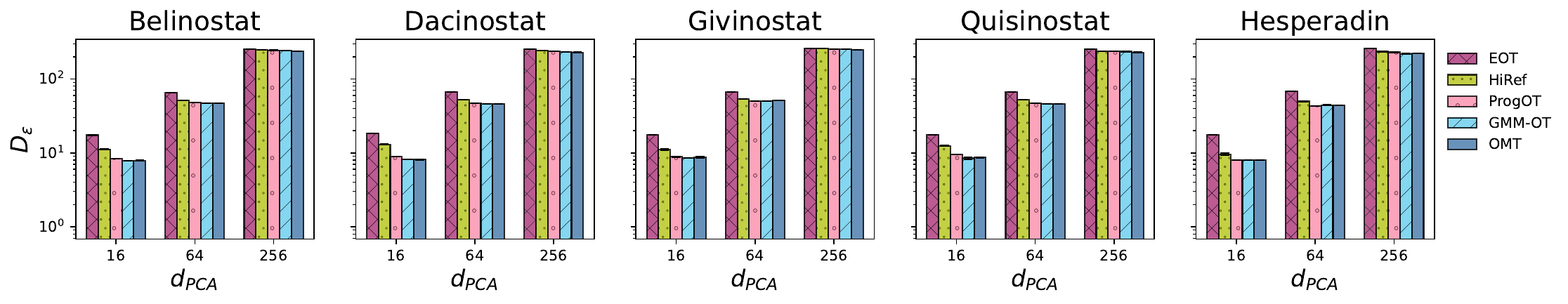}
\caption{\footnotesize{Performance comparison of OT solvers on the sci-Plex dataset across varying dimensionalities. Average $D_{\varepsilon} \downarrow$ values are computed for transported samples through forward and backward transport across five drug treatments. Evaluations are conducted across PCA dimensions $d_{PCA} \in \{16, 64, 256\}$. Each bar represents the mean of five independent runs, with error bars denoting the standard deviation. Note that the error bars are nearly imperceptible due to the logarithmic scale.}}
\label{fig:AppsciplexPCA}
\end{figure}
}

\newcommand{\AppDevPerformance}{
\begin{figure}[h!]
\centering
\includegraphics[width=\columnwidth,trim=4pt 10pt 4pt 10pt,clip]{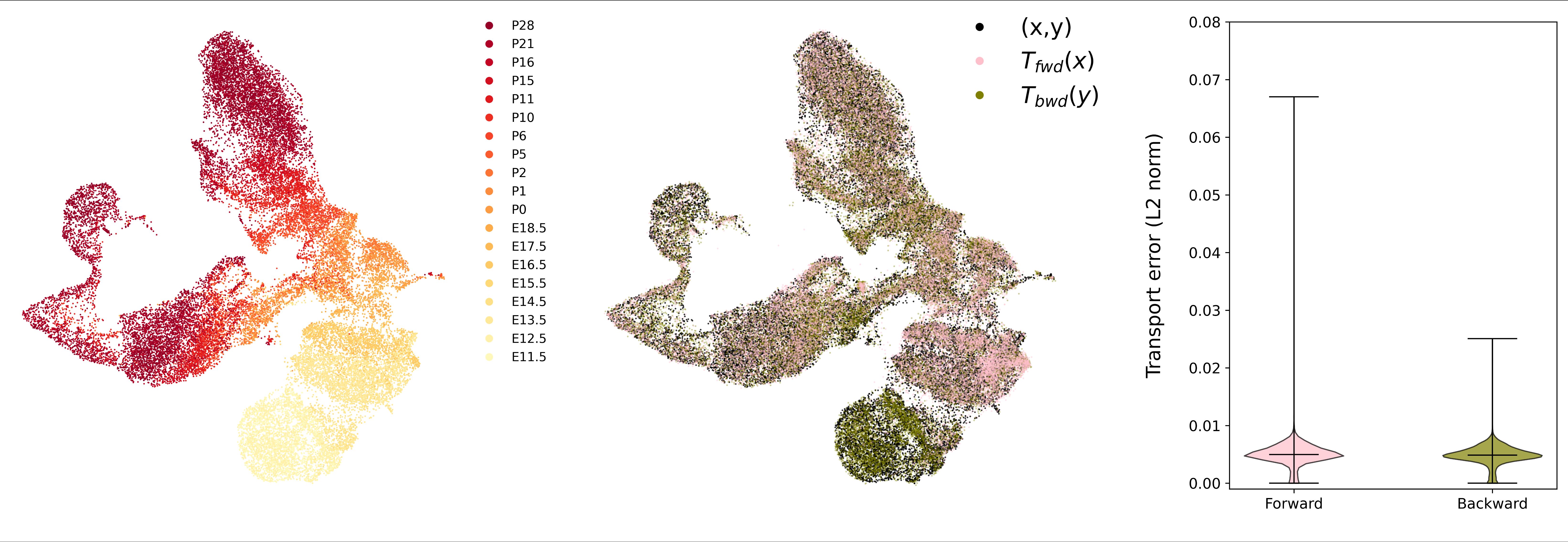}
\caption{{\footnotesize{OMT performance on the mouse visual cortex developmental dataset. The data comprises $32,998$ cells and $9,900$ HVGs~\citep{gao2024continuous}. (Left) UMAP visualization of cell distribution colored by developmental time point, from $E11.5$ to $P28$. (Middle) UMAP overlay showing the alignment between measurement, $(x,y)$ (black dots) and predicted cells from forward, $T_{fwd}(x)$ (pink dots) and backward, $T_{bwd}(y)$ (olive dots) transport. (Right) Violin plots of the transportation error (L2 norm) across all genes and cells for both directions.}}}
\label{fig:AppDevPerformance}
\end{figure}
}

\newcommand{\AppDevGenes}{
\begin{figure}[h!]
\centering
\includegraphics[width=\columnwidth,trim=4pt 10pt 10pt 10pt,clip]{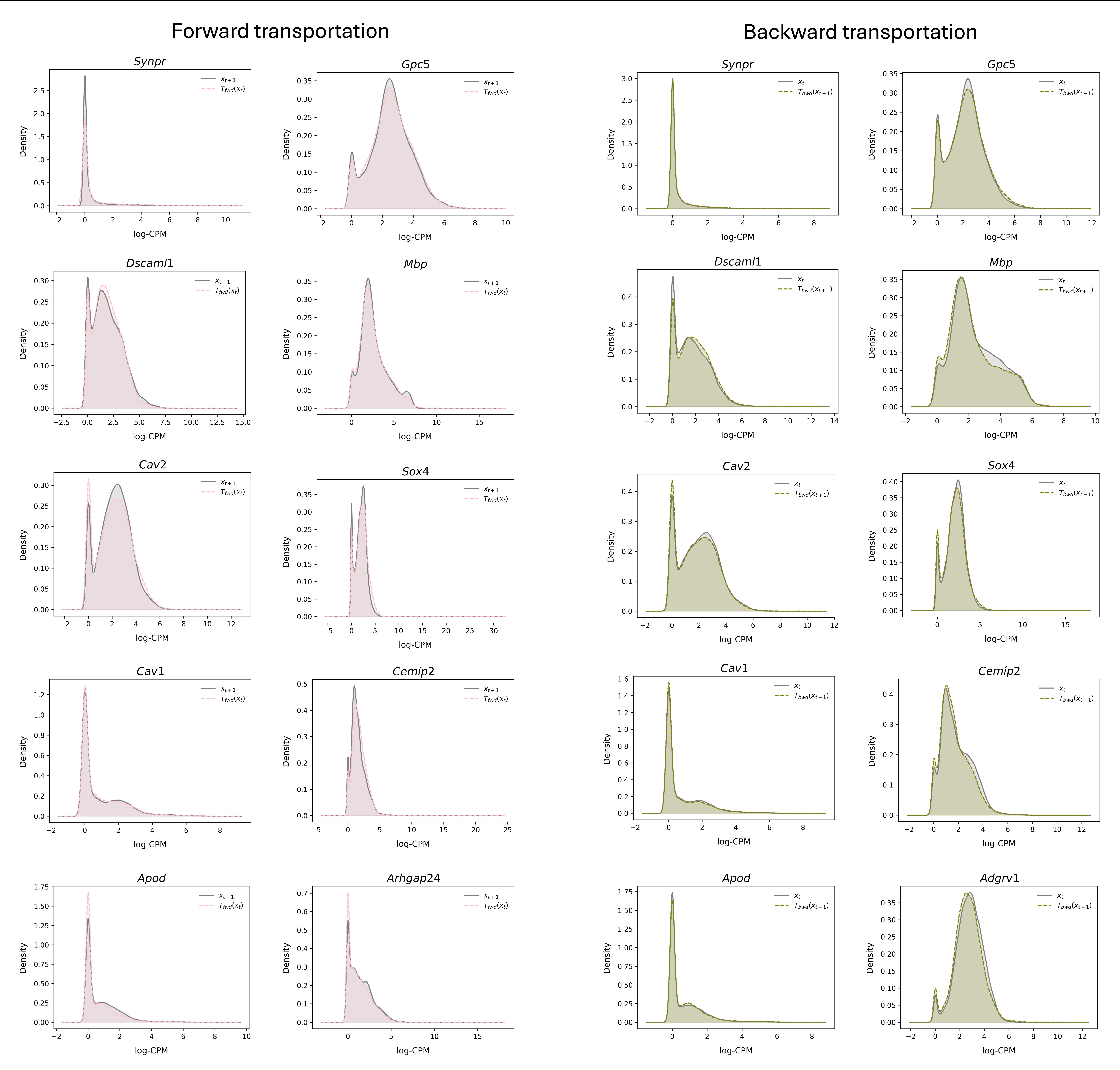}
\caption{{\footnotesize{Distributional alignment of log-CPM expression values for a subset of marker genes in non-neuronal cells across developmental time. Each subfigure shows distributions of the ground-truth expression (solid line) and the OMT-transported values (dashed line), for both the forward (Left) and backward (Right) directions.}}}
\label{fig:AppDevGenes}
\end{figure}
}

\newcommand{\AppAgePerformance}{
\begin{figure}[h!]
\centering
\includegraphics[width=\columnwidth,trim=4pt 10pt 10pt 10pt,clip]{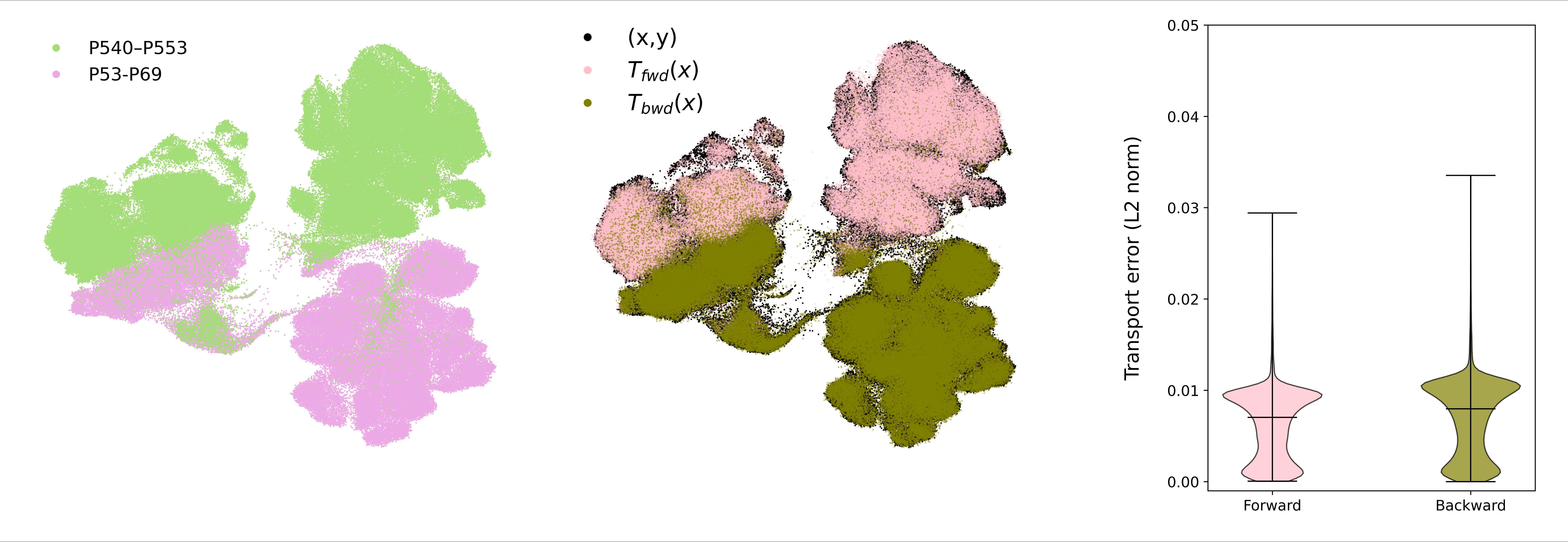}
\caption{\footnotesize{OMT performance on the mouse aging dataset. The dataset comprises $253,468$ cells and $9,359$ HVGs sampled from six brain regions~\citep{jin2025brain}. (Left) UMAP embedding of cell distributions at two time points, adult and aged. (Middle) UMAP overlay showing the alignment between collected cells, $(x, y)$ (black dots), and cells predicted by the forward, $T_{\mathrm{fwd}}$ (pink dots), and backward, $T_{\mathrm{bwd}}$ (olive dots), OMT maps. (Right) Violin plots of the transportation error (L2 norm) across all genes and cells for both directions.}}
\label{fig:AppAgePerformance}
\end{figure}
}

\newcommand{\AppAgeGenes}{
\begin{figure}[h!]
\centering
\includegraphics[width=\columnwidth,trim=4pt 10pt 10pt 10pt,clip]{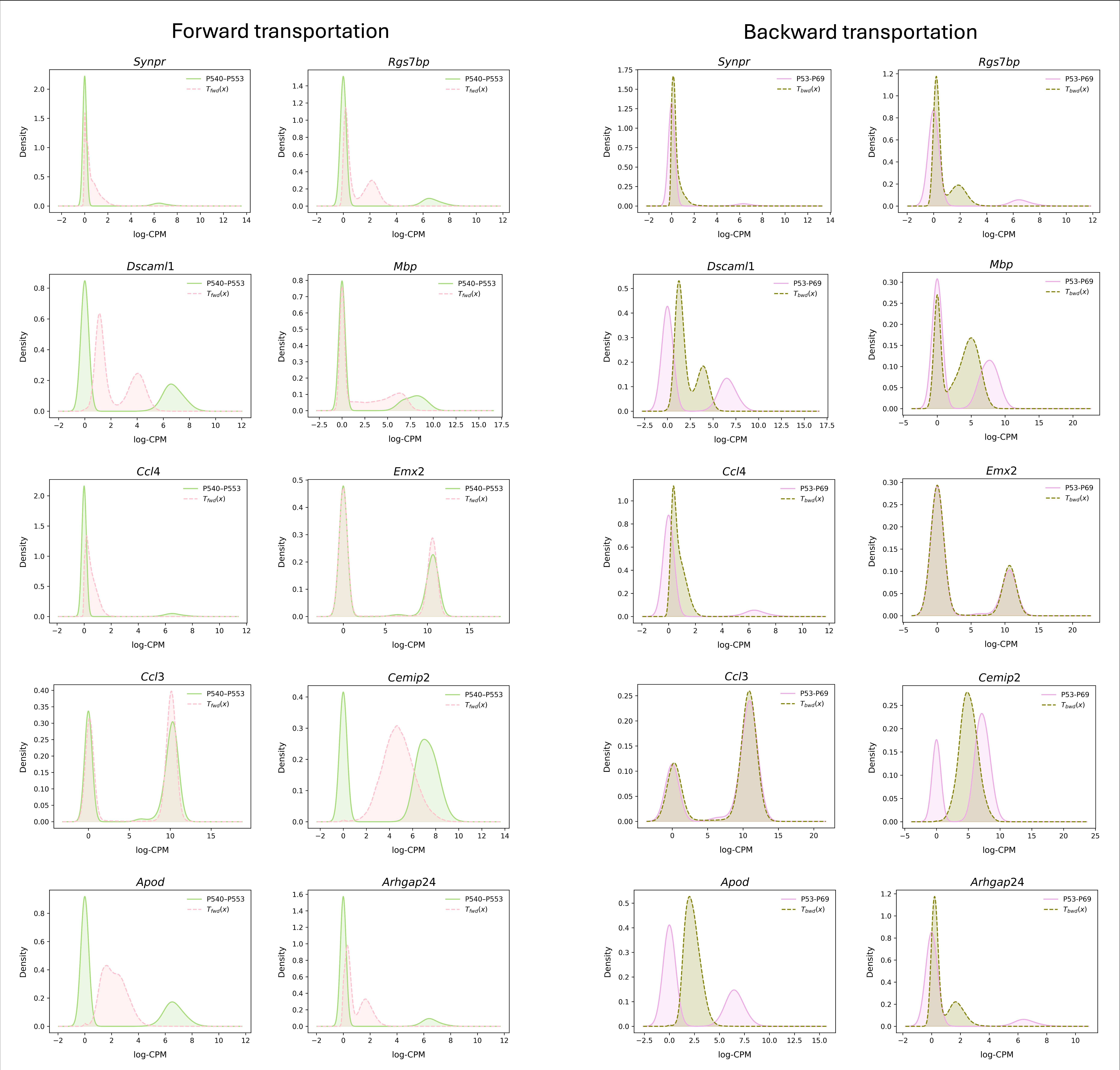}
\caption{\footnotesize{Distributional alignment of log-CPM expression values for a subset of marker genes in non-neuronal cells in the mouse aging dataset. Each subpanel shows distributions of the ground-truth expression (solid line) and the OMT-transported values (dashed line), for both the forward (Left) and backward (Right) directions.}}
\label{fig:AppAgeGenes}
\end{figure}
}

\newcommand{\OMTGMMOT}{
\begin{figure}[h!]
\centering
\includegraphics[width=\columnwidth]{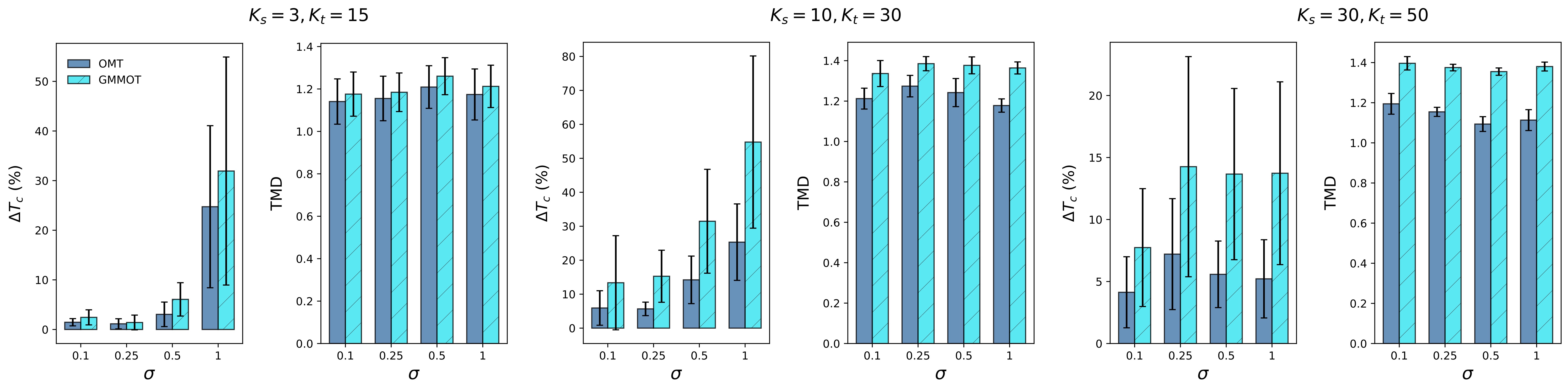}
\caption{\footnotesize{Comparing OMT and GMM-OT maps under input perturbation.
We evaluate robustness on W2-benchmark tasks ($d=2$) across varying noise levels $\sigma \in \{0.1, 0.25, 0.5, 1.0\}$ and increasing mixture complexities ($K_s, K_t$). To ensure a direct comparison of the optimization objectives, identical training samples and fitted GMMs were used for both methods in each setting. The bar charts report two stability metrics: changes of transport cost ($\Delta T_c$) and the relative transport map deviation (TMD). Error bars represent the standard deviation across $10$ random initializations. OMT maps consistently demonstrate lower sensitivity to noise compared to GMM-OT, validating the stability benefits of the biconvex formulation.}}
\label{fig:OMTGMMOT}
\end{figure}
}

\newcommand{\ExpVsGauss}{
\begin{figure}[h!]
\centering
\includegraphics[width=\columnwidth]{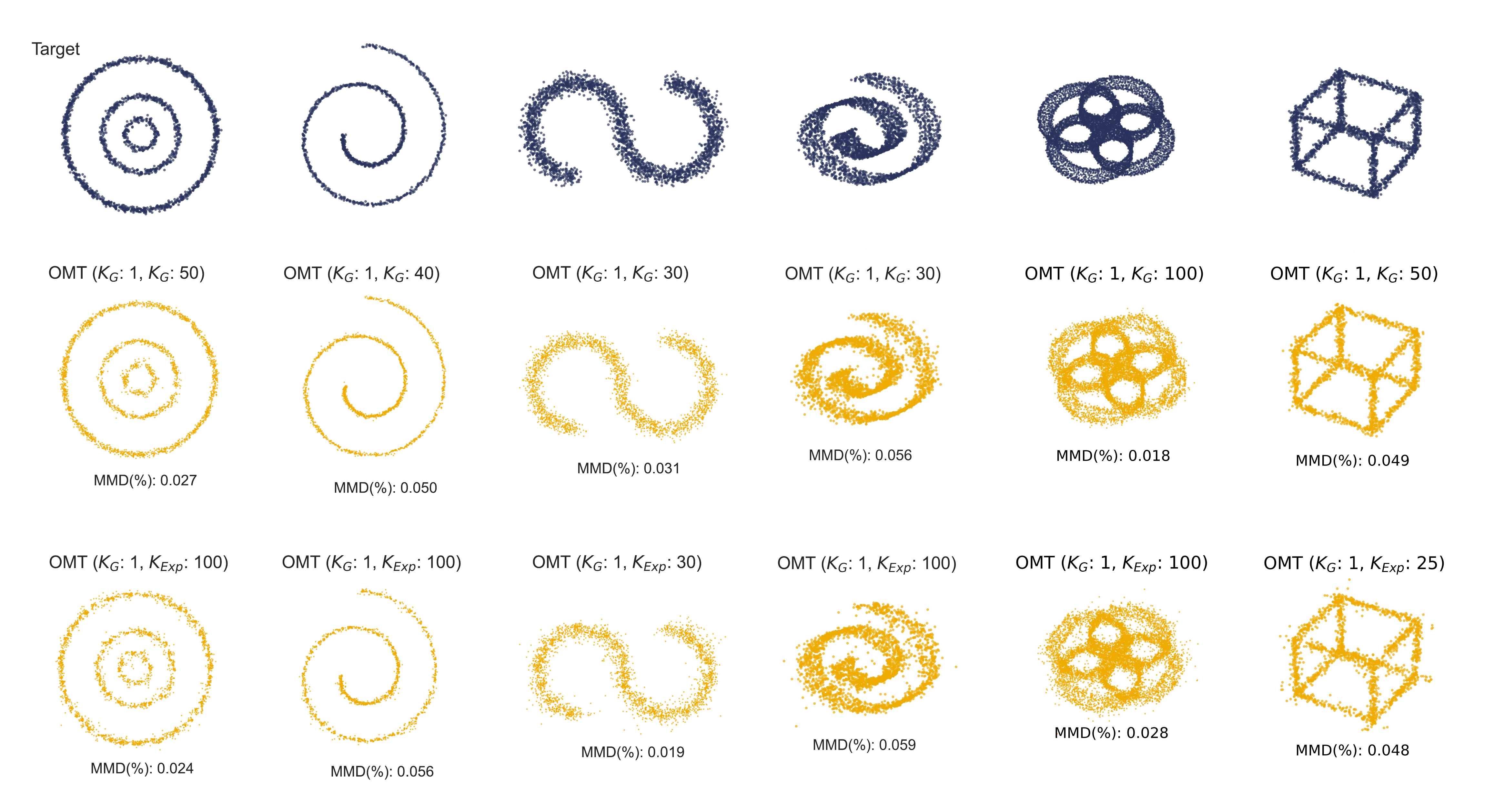}
\caption{\footnotesize{Qualitative and quantitative comparison of OMT using Gaussian (symmetric $p = q = 2$) versus factorized exponential (asymmetric $p = q = 1$) components on synthetic data. The top row shows the target sample point clouds. The middle and bottom rows show samples transported from a normal distribution to various target distributions, using different numbers of Gaussian components ($K_G$) and exponential components ($K_{\mathrm{Exp}}$), respectively.}}
\label{fig:ExpVsGauss}
\end{figure}
}

\newcommand{\PathShape}{
\begin{figure}[h!]
\centering
\includegraphics[width=\columnwidth]{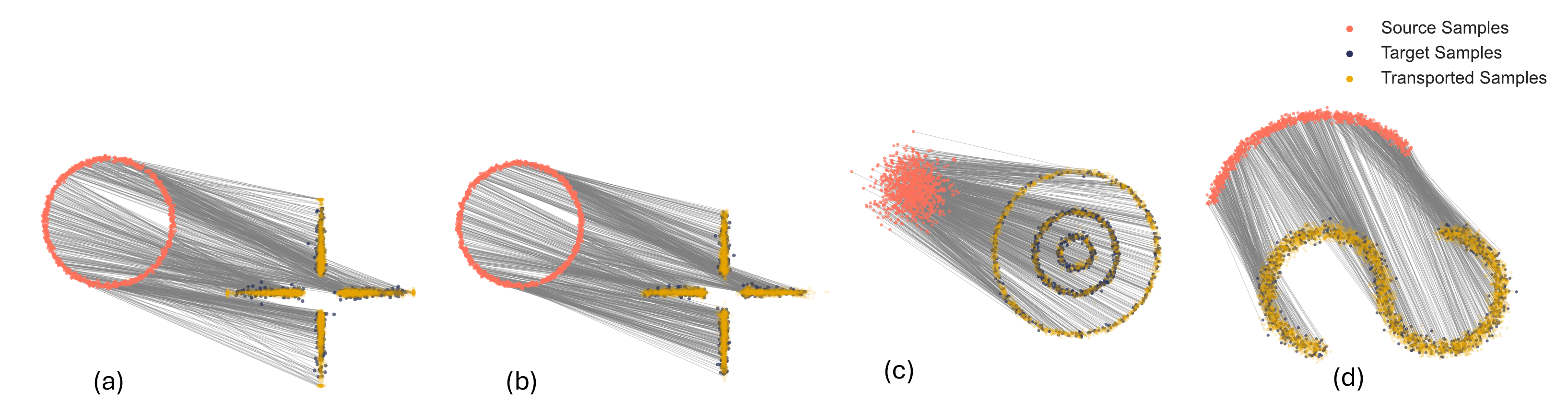}
\caption{\footnotesize{Sample-level transport paths across four synthetic topologies obtained using OMT. Gray lines indicate exact point-to-point transport paths mapping source samples (red) to transported samples (yellow), aligning them with target samples (blue).
Panels (a) and (b) compare the representational efficiency of different component families when transporting a circular source distribution to a cross-shaped target. In (a), the mapping is achieved using a compact mixture of four factorized exponential components ($p = q = 1$, $K_{\mathrm{Exp}} = 4$). In contrast, (b) requires a mixture of thirty Gaussian components ($K_G = 30$) to capture the same geometry, illustrating the trade-off between flexibility and component efficiency.
Panels (c) and (d) show transport to more complex targets: a concentric ring using $K_G = 100$ and a half-moon using $K_G = 30$, respectively, demonstrating the ability of the framework to model highly non-linear distributions.}}
\label{fig:PathShape}
\end{figure}
}

%% file: tables/main.tex
\newcommand{\TabsciPlex}{
\begin{table}[b!]
\centering
\vspace{-.25in}
\caption{{\footnotesize{Average Sinkhorn divergence, $D_{\varepsilon} \downarrow$, for forward and backward OMT maps compared to P{\scriptsize{ROG}}OT, HiRef, GMM-OT, and EOT on the human scRNA-seq dataset~\citep{srivatsan2020massively}. The results are reported as mean$\pm$std. dev. ($d_{\mathrm{PCA}} = 16$, $5$ random seeds). Additional results for other dimensions and computational costs are in Appendix~\ref{sec:AppscRNAeq}.}}}
\label{tab:sciPlex}
\begin{tabular}{llllll}
\toprule
& Belinostat & Dacinostat & Givinostat & Quisinostat & Hesperadin \\
\midrule
{\footnotesize{EOT}} & $17.4 \pm 0.01$ & $18.4 \pm 0.01$ & $17.5 \pm 0.01$  & $17.6 \pm 0.03$ & $17.5 \pm 0.01$\\
{\footnotesize{HiRef}} & $19.9 \pm 0.21$ & $24.0 \pm 0.30$ & $19.2 \pm 0.55$ & $22.8 \pm 0.2$ & $17.3 \pm 0.54$\\
{\footnotesize{P{\scriptsize{ROG}}OT}} & $8.4 \pm 0.01$ & $8.8 \pm 0.03$ & $8.8 \pm 0.04$ & $9.5 \pm 0.01$ & $8.1 \pm 0.01$\\
{\footnotesize{$\text{GMM-OT}_{(K_0=3, K_1=5)}$}} & $8.0 \pm 0.02$ & $8.9 \pm 0.17$ & $8.9 \pm 0.06$ & $8.8 \pm 0.23$ & $8.1 \pm 0.03$\\
{\footnotesize{$\text{OMT}_{(K_0=3, K_1=5)}$}} & $\bm{7.9 \pm 0.06}$ & $\bm{8.1 \pm 0.04}$ & $\bm{8.7 \pm 0.14}$ & $\bm{8.7 \pm 0.10}$ & $\bm{8.0 \pm 0.03}$\\
\bottomrule
\end{tabular}
\end{table}
}

\newcommand{\TabMERFISH}{
\begin{table}[b!]
\centering
\vspace{-.3in}
\caption{\footnotesize{Performance on the MERFISH mouse brain dataset. Cosine similarity ($\uparrow$) and spatial Euclidean cost ($\downarrow$) are reported. Cosine similarity measures the agreement between imputed and original expression profiles in the target slice using the evaluation genes from \cite{halmos2025hierarchical}, while the transport map is optimized solely based on spatial coordinates.}}
\label{tab:MERFISH}
\begin{tabular}{llllllll}
\toprule
& &
\multicolumn{1}{c}{\footnotesize{\textit{Slc17a7} ($\uparrow$)}} &
\multicolumn{1}{c}{\footnotesize{\textit{Grm4}($\uparrow$)}} &
\multicolumn{1}{c}{\footnotesize{\textit{Olig1} ($\uparrow$)}} &
\multicolumn{1}{c}{\footnotesize{\textit{Gad1} ($\uparrow$)}} &
\multicolumn{1}{c}{\footnotesize{\textit{Peg10} ($\uparrow$)}} &
\multicolumn{1}{c}{\footnotesize{Cost ($\downarrow$)}}  \\
\midrule
{{HiRef}} & & $0.81 $ & $0.80$ & $0.75$ & $0.49$ & $0.60$ & $330.33$ \\
{{$\text{GMM-OT}_{(K_0=1000, K_1=1000)}$}} & & $0.89$ & $0.90$ & $0.90$ & $0.73$ & $0.73$ & $107.20$\\
{{$\text{OMT}_{(K_0=1000, K_1=1000)}$}} & & $\mathbf{0.90}$ & $\mathbf{0.93}$ & $\mathbf{0.90}$ & $\mathbf{0.78}$ & $\mathbf{0.75}$ & $\mathbf{101.69}$\\
\bottomrule
\end{tabular}
\end{table}
}

\newcommand{\ImageNet}{
\begin{table}[h!]
\vspace{-.1in}
\centering
\begin{minipage}{0.55\textwidth}
    \centering
    \begin{tabular}{lllll}
    \toprule
    \textbf{Method} & MB~($1024$) & FRLC & HiRef &  OMT\\
    \midrule
    \textbf{Cost}~($\downarrow$) & $19.58$ & $24.12$ & $18.97$ & $\mathbf{14.30}$ \\
    \bottomrule
    \end{tabular}
\end{minipage}%
\hspace{0.5em} 
\begin{minipage}{0.43\textwidth}
    \captionof{table}{\footnotesize Cost values for the ImageNet alignment task~\citep{deng2009imagenet}. OMT is trained with 1000 components for each of the source and target.}
    \label{tab:ImageNet}
\end{minipage}
\end{table}
\vspace{-0.1in}
}

%% file: tables/supp.tex
\newcommand{\NetworkMNIST}{
\begin{table}[h!]
\centering
\caption{Architecture of the VAE for MNIST.}
\label{tab:NetworkMNIST}
\begin{tabular}{@{}lll@{}}
\toprule
\textbf{Layer} & \textbf{Configuration Details} & \textbf{Activation} \\
\midrule
\multicolumn{3}{c}{\textbf{Encoder}} \\
\midrule
Input Layer & 2 x (Batch, 1, N, N) Image & - \\
Dropout & & - \\
Conv2d (x2) & $1 \to 16$ channels, kernel=(5,5), stride=2 & ReLU \\
Norm2d & 16 features & - \\
Conv2d & $16 \to 32$ channels, kernel=(3,3), stride=2 & ReLU \\
Norm2d & 32 features & - \\
Conv2d & $32 \to 32$ channels, kernel=(3,3), stride=2 & ReLU \\
Norm2d & 32 features & - \\
Flatten & Reshapes feature map to (Batch, 128) & - \\
Linear & $128 \to 100$ units & ReLU \\
\midrule
\multicolumn{3}{c}{\textbf{Decoders}} \\
\midrule
Linear & $d_z \to 100$ units & ReLU \\
Linear & $100 \to 128$ units & ReLU \\
Unflatten & Reshapes to (Batch, 32, 2, 2) & - \\
ConvTranspose2d & $32 \to 32$ channels, kernel=(3,3), stride=2 & ReLU \\
Norm2d & 32 features & - \\
ConvTranspose2d & $32 \to 16$ channels, kernel=(5,5), stride=2 & ReLU \\
Norm2d & 16 features & - \\
ConvTranspose2d & $16 \to 1$ channel, kernel=(5,5), stride=2 & ReLU \\
Norm2d & 1 feature & - \\
ConvTranspose2d & $1 \to 1$ channel, kernel=(4,4) & ReLU \\
\bottomrule
\end{tabular}
\end{table}
}

\newcommand{\NetworkCIFAR}{
\begin{table}[h!]
\centering
\caption{Architecture of the DoubleRessNet for CIFAR-10.}
\label{tab:NetworkCIFAR}
\begin{tabular}{lll}
\toprule
\textbf{Layer} & \textbf{Configuration Details} & \textbf{Activation} \\
\midrule
\multicolumn{3}{c}{\textbf{Shared Encoder}} \\
\midrule
Input Layer & 2 x (Batch, $3$, H, W) & - \\
Conv2d & $3 \to 32$ channels, kernel=9, stride=1 & ReLU \\
Norm2d & 32 features & - \\
Conv2d & $32 \to 64$ channels, kernel=3, stride=2 & ReLU \\
Norm2d & 64 features & - \\
Conv2d & $64 \to 128$ channels, kernel=3, stride=2 & ReLU \\
Norm2d & 128 features & - \\
Residual Block & 4 stacked blocks (128 channels) & ReLU \\
Flatten & Reshapes to (Batch, $H \times 64$) & - \\
Linear & $H \times 64 \to H \times 16$ units & ReLU \\
Norm1d & $H \times 16$ features & - \\
\midrule
\multicolumn{3}{c}{\textbf{Latent Space}} \\
\midrule
Linear & $H \times 16 \to H \times 4$ units & Tanh \\
\midrule
\multicolumn{3}{c}{\textbf{Decoders}} \\
\midrule
Linear & $H \times 4 \to H \times 16$ units & - \\
Norm1d & $H \times 16$ features & - \\
Linear & $H \times 16 \to H \times 64$ units & ReLU \\
Norm1d & $H \times 64$ features & - \\
Unflatten & Reshapes to (Batch, $H$, 8, 8) & - \\
Upsample, Conv2d & $H \to 64$ channels, kernel=3, upsample=2 & ReLU \\
Norm2d & 64 features & - \\
Upsample, Conv2d & $64 \to 32$ channels, kernel=3, upsample=2 & ReLU \\
Norm2d & 32 features & - \\
Upsample, Conv2d & $32 \to 3$ channels, kernel=9, stride=1 & Linear \\
\bottomrule
\end{tabular}
\end{table}
}

\newcommand{\FID}{
\begin{table}[h!]
\centering
\begin{minipage}{0.48\textwidth}
    \centering
    \begin{tabular}{lcc}
    \toprule
    & MNIST & CIFAR-10 \\
    \midrule
    WGAN & $6.7 \pm 0.4$ & $55.2$ \\
    WGAN-GP & $7.4 \pm 0.3$ & $39.4$ \\
    OMT & $1.2 \pm 0.1$ & $56.1 \pm 1.5$\\
    \bottomrule
    \end{tabular}
\end{minipage}%
\hspace{0.5em} 
\begin{minipage}{0.5\textwidth}
    \captionof{table}{{\footnotesize FID~($\downarrow$) values for unpaired image translation on the MNIST (grayscale) and CIFAR-10 (color) datasets. Reported values for WGAN and WGAN-GP are taken from previous studies~\citep{choi2023generative, rout2021generative, qian2021self}. Results for OMT are computed over $10$ random initializations.}}
    \label{tab:fid}
\end{minipage}
\end{table}
}

%% file: main/1-introduction.tex
Optimal Transport (OT) offers a powerful mathematical framework for comparing probability distributions and finding optimal mappings between them~\citep{santambrogio2015optimal}. Its versatility has led to advances in diverse fields, including domain adaptation~\citep{grave2019unsupervised, struckmeier2023learning, chuang2023infoot,fernandes2024optimal}, data integration and alignment~\citep{demetci2022scot}, and predicting cell fates~\citep{tong2020trajectorynet, bunne2023learning, bunne2024optimal}. At its core, OT seeks to find the most cost-effective way to transform one probability distribution into another, subject to constraints on the total mass being transported~\citep{peyre2019computational, villani2008optimal}.
A major challenge in OT has been its high computational cost, requiring $O(n^3 \log{n})$. Entropy regularization was introduced to the OT objective (EOT)~\cite{cuturi2013sinkhorn} to make that solvable via the Sinkhorn algorithm with an improved sample complexity of $O(n^2)$. However, even with EOT, sample-to-sample transportation remains limited for large datasets~\citep{genevay2018learning}. To mitigate this, mini-batch strategies (MB-OT) have been developed to approximate the transport plan by operating on subsets of the data ~\citep{genevay2018learning,fatras2021minibatch, fatras2021unbalanced}. While computationally cheaper, these methods often yield suboptimal transport plans, as cost estimation from subsets can be inaccurate and satisfying the mass preservation constraint of balanced OT becomes difficult. 
To improve transport accuracy over batches, one prominent class of methods approximates the OT path by interpolation within the Wasserstein space. These techniques range from deterministic approaches, such as Progressive Optimal Transport (P{\scriptsize{ROG}}OT)~\citep{kassraie2024progressive}, to stochastic methods based on gradient flows and Schrödinger bridges~\citep{albergo2023building, albergo2024stochastic}. Stochastic methods, often employing neural networks such as those based on gradient flows~\citep{daniels2021score} or Schrödinger bridges \citep{gushchin2023building, gushchin2023entropic}, also typically necessitate inner iterations to achieve accurate transport maps. Such methods construct a sequence of intermediate distributions to bridge the source and target, often requiring numerous intermediate steps, significant memory overhead, and many inner iterations to converge to an accurate solution. Furthermore, simpler displacement strategies, like the McCann interpolation~\citep{mccann1997convexity} used in P{\scriptsize{ROG}}OT, do not always produce interpretable intermediate distributions. 
While entropic regularization-based solvers (e.g., EOT and P{\scriptsize{ROG}}OT) are known to exhibit stability under perturbations of the marginal distributions~\citep{divol2025tight,kassraie2024progressive}, such guarantees are largely absent for neural OT approaches. 
Another prominent direction is low-rank OT (LOT), which applies low-rank factorization to the coupling matrix, achieving linear sample complexity in certain regimes~\citep{scetbon2021low}. However, the performance of this group of solvers is highly limited by the choice of rank, suffers from the curse of dimensionality~\citep{halmos2025hierarchical}, and is sensitive to (non-Gaussian) noise~\citep{cao2015low}.

One effective strategy for large-scale problems is to integrate continuous and discrete transportation methods, where groups of samples are transported using discrete transport and each group is further processed through a continuous transport mechanism. Following this direction, we propose Optimal Mixture Transport (OMT), which computes EOT maps between mixture models. Our formulation recasts the global transport problem as a strictly biconvex optimization problem. We further show that this formulation converges to a unique solution and is stable under weak regularity conditions~\citep{hyvarinen2005estimation}.

\textbf{Contributions.}
\setcounter{contrib}{0}
\contrib We propose OMT, a mixture-based EOT solver that adopts combinations of continuous and discrete transport, resulting in a strictly biconvex formulation (Lemma~\ref{lemma:biconvex}).
\contrib We show that OMT admits an efficient global optimization scheme, where each subproblem has a unique solution and converges in a single step (Theorem~\ref{theorem:uniqGOP}).
\contrib We provide a stability analysis (Theorem~\ref{theorem:stability_main}), demonstrating that under weak regularity conditions, the OMT map is robust to perturbations in marginal distributions.
\contrib We leverage exponential family models to derive closed-form subpopulation transport maps in OMT, significantly reducing computational cost.
Across synthetic and real-world datasets, we show that the proposed OT solver is scalable and robust, consistently matching or outperforming state-of-the-art methods.

%% file: main/2-background.tex
{\bf Optimal transport:} For $\mathcal{X}, \mathcal{Y} \subset \mathbb{R}^d$, probability measures $\mu_0, \mu_1 \in \mathcal{P}(\mathbb{R}^{d})$ and $c: \mathcal{X} \times \mathcal{Y} \rightarrow \mathbb{R}$,
a cost function associated with transporting a unit of mass from a point in $\mathcal{X}$ to a point in $\mathcal{Y}$, the minimum total cost of transport can be obtained as
\begin{equation}
\label{eq:monge}
      \displaystyle{\inf_{T{\sharp \mu_0 = \mu_1}}} \int_{\mathcal{X}} c(\mathbf{x}, T(\mathbf{x})) \ d\mu_0(\mathbf{x}),
\end{equation}
where $T_{\sharp} \mu_0 = \mu_1$ denotes the pushforward of $\mu_0$ by $T$, defined as $\mu_1(A) = \mu_0(T^{-1}(A))$, $\forall A \subset \mathcal{Y}$, ensuring mass conservation. Problem \eqref{eq:monge} is known as the Monge problem~\citep{peyre2019computational} which seeks a map $T: \ \mathbb{R}^d \rightarrow \mathbb{R}^d$ referred to as the transport map between $\mu_0$ and $\mu_1$. The Monge formulation is often problematic because the optimization is over a non-convex set of maps, and a deterministic map $T$ may not exist.
To bypass this, Kantorovich presented a relaxed formulation, which seeks a distribution $\pi \in \mathbb{R}^d \times \mathbb{R}^d$ referred to as
‘‘coupling’’ of $\mu_0$ and $\mu_1$, as
\begin{equation}
\label{eq:staticKP}
      \displaystyle{\inf_{\pi \in \prod_{\mu_0, \mu_1}}} \int_{\mathcal{X} \times \mathcal{Y}} c(\mathbf{x}, \mathbf{y}) \ d\pi(\mathbf{x},\mathbf{y}),
\end{equation}
where {$\prod_{\mu_0, \mu_1} := \{\pi \in \mathcal{P}(\mathcal{X} \times \mathcal{Y}) \ | \ P_{\mathcal{X}}: \pi \rightarrow \mu_0,  P_{\mathcal{Y}}: \pi \rightarrow \mu_1 \}$}. 

When $\mathcal{X}=\mathcal{Y}$ and $c(\mathbf{x}, \mathbf{y}) = d(\mathbf{x},\mathbf{y})^p$, $p \geq 1$, where $d$ is a distance on $\mathcal{X}$, the Kantorovich problem is equivalent to the Wasserstein p-distance between probability measures, $\mathcal{W}_p(\mu_0, \mu_1)$. Specifically, for $c(\mathbf{x}, \mathbf{y}) = \|\mathbf{x}-\mathbf{y}\|_2^2$, the Kantorovich problem yields the squared Wasserstein-2 distance:
\begin{equation}
\label{eq:KP_Wasserstein}
    \mathcal{W}^2_2(\mu_0, \mu_1) = \displaystyle{\inf_{\pi \in \prod(\mu_0, \mu_1)}} \int_{\mathcal{X} \times \mathcal{Y}} \| \mathbf{x}-\mathbf{y} \|^2_2 \ d\pi(\mathbf{x},\mathbf{y}).
\end{equation}
%

{\bf Entropic optimal transport:} Entropic optimal transport introduces an entropic regularization term to~\eqref{eq:KP_Wasserstein}, transforming the problem into a strictly convex optimization that can be efficiently solved using algorithms like the Sinkhorn-Knopp method~\citep{cuturi2013sinkhorn,janati2020entropic}. For a regularization parameter $\varepsilon > 0$, the entropic optimal transport cost is defined as:
\begin{equation}
\label{eq:EOT}
    d_{\varepsilon}(\mu_0, \mu_1) = \inf_{\pi \in \Pi_{\mu_0, \mu_1}} \left\{ \int_{{{\mathcal{X} \times \mathcal{Y}}}} \| \mathbf{x}-\mathbf{y} \|_2^2 \ d\pi(\mathbf{x},\mathbf{y}) - 2\varepsilon H(\pi) \right\}.
\end{equation}
%
%
%

%% file: main/4-method.tex
Let $\nu(\mathbf{x}) \in M_{K}(\mathbb{R}^d)$ denote a mixture model, $\nu(\mathbf{x}) = \sum_{i=0}^{K} \alpha_i \mu_i(\mathbf{x})$, $\mathbf{x} \in \mathbb{R}^d$ with $K$ components. Here \(\mu_i\) are probability measures and \(\sum_i \alpha_i=1\), \(\alpha_i\geq 0, \forall i\).  
%
%

Given two measures $\nu_0 \in M_{K_0}(\mathbb{R}^d)$ and $\nu_1 \in M_{K_1}(\mathbb{R}^d)$, we define the mixture transport coupling as follows:
\begin{eqnarray}
\label{eq:GMM_Wasserstein}
     \pi^*_{\mathcal{M}} &:=& \displaystyle{\argmin_{\pi \in \prod(\nu_0, \nu_1) \cap M_{K}(\mathbb{R}^{2d})}} \int_{\mathcal{X} \times \mathcal{Y}} \| \mathbf{x}-\mathbf{y} \|^2_2 \ d\pi(\mathbf{x},\mathbf{y}) 
\end{eqnarray}
where $\prod(\nu_0, \nu_1) := \{\pi \in \mathcal{P}(\mathcal{X} \times \mathcal{Y}) \ | \ P_{\mathcal{X}}: \pi \rightarrow \nu_0(\mathbf{x}),  P_{\mathcal{Y}}: \pi \rightarrow \nu_1(\mathbf{y}) \}$ and $K = K_0 K_1$.
Therefore, the transport policy belongs to the mixture model family and can be expressed as $ d\pi(\mathbf{x}, \mathbf{y}) = {\sum^K_{i=1}} \omega_{i} dp_i(\mathbf{x}, \mathbf{y}), \hspace{.1in} \text{where} \ \forall i, \ p_i \in \mathcal{P}(\mathbb{R}^{2d})$.
Note that the trivial choice of $dp_1=\ldots=dp_K$, $\omega_1=\ldots=\omega_K=1/K$ makes the solver in Eq.~\ref{eq:GMM_Wasserstein} equal to ${\mathcal{W}_2}(\nu_0, \nu_1)$ in Eq.~\ref{eq:KP_Wasserstein}. We next constrain that marginals of the components of the transport coupling are the same as the components of the source and target functions:
\begin{align}
\label{eq:omt}
\mathcal{D}_{\mathcal{M}}(\nu_0, \nu_1) = \displaystyle{\min_{\Omega, P}} &\sum_{i,j} \omega_{ij} \int_{\mathcal{X} \times \mathcal{Y}} \| \mathbf{x}-\mathbf{y} \|^2_2 \ dp_{ij}(\mathbf{x},\mathbf{y}), \nonumber \\
\text{s.t.} \hspace{.5in}&  \mathbf{1^T}\Omega = \boldsymbol{\alpha}_0, \hspace{.2in} \Omega^T\mathbf{1} = \boldsymbol{\alpha}_1  \\ 
\forall i, \ \displaystyle{\sum_j \int_{\mathcal{Y}} dp_{ij}(\mathbf{x}, \mathbf{y}) = \mu_{0_{i}}}&,  \hspace{.2in} \forall j, \ \displaystyle{\sum_i \int_{\mathcal{X}} dp_{ij}(\mathbf{x}, \mathbf{y}) = \mu_{1_{j}}}, \nonumber
\end{align}
where \(\Omega=[w_{ij}]\) denotes the matrix of mixture weights. With this constraint, $\mathcal{D}_{\mathcal{M}}(\nu_0, \nu_1) \geq {\mathcal{W}_2}(\nu_0, \nu_1)$ and equality is achieved in the $K\to\infty$ limit when the source and target functions are over-parametrized. 
The problem in Eq.\ref{eq:omt} is similar to minimizing the aggregated Wasserstein distance~\citep{chen2019aggregated}.
\vspace{-.1in}
\subsection{Optimal mixture transport}
\label{sec:Reg_OMT}
\vspace{-.1in}
A common approach to ensure uniqueness of solutions in optimal transport problems is to introduce an entropy regularization term, which makes the objective function strictly convex and improves the numerical stability of optimization. In the context of mixture transport optimization in \ref{eq:omt}, we adopt a similar approach by incorporating a weighted average entropy term as a regularizer.
\begin{define}[Optimal Mixture Transport]
Let $\nu_0(\mathbf{x}) = \sum_{i=1}^{K_0} \alpha_{0,i} \mu_{0,i}(\mathbf{x})$ and $\nu_1(\mathbf{y}) = \sum_{j=1}^{K_1} \alpha_{1,j} \mu_{1,j}(\mathbf{y})$ be two mixture models. The OMT objective, defined in Eq.~\ref{eq:emot}, incorporates two levels of regularization: (i) a component-wise entropy $H(p_{ij})$ and (ii) a mixing-matrix entropy $H(\Omega)$, controlled by hyperparameters $\varepsilon_1, \varepsilon_2 > 0$. 
\begin{align}
\label{eq:emot}
\mathcal{L}_{\varepsilon_1, \varepsilon_2}(\Omega, P) =  
\sum_{i,j}^K \omega_{ij}
D_{\varepsilon_1}(p_{ij})
- \varepsilon_2 H(\Omega) \\
 \text{where} \ D_{\varepsilon_1}(p_{ij}) = \int_{\mathcal{X} \times \mathcal{Y}}
    \|\mathbf{x}-\mathbf{y}\|_2^2 \,
    dp_{ij}(\mathbf{x},\mathbf{y}) - \varepsilon_1 H(p_{ij}) \ , \nonumber
\end{align}
and represents the entropic Wasserstein cost between mixture components. Here $\Omega = [\omega_{ij}]_{K_0 \times K_1}$ denotes the component-coupling matrix and $P = [p_{ij}]$ is the set of local transport plans.
\end{define}
Accordingly OMT coupling, defined as $\pi_{\textsc{omt}}(\mathbf{x}, \mathbf{y})={\sum_{i,j}} {\omega}_{ij} p_{ij}(\mathbf{x}, \mathbf{y})$, is obtained by the following constrained optimization.
\begin{gather}
\pi_{\textsc{omt}}(\mathbf{x}, \mathbf{y}) = \displaystyle{\argmin_{\Omega, P}  \mathcal{L}_{\varepsilon_1, \varepsilon_2}(\Omega, P) }  \nonumber \\
{{\text{s.t.} \hspace{.2in} \mathbf{1}\Omega = \boldsymbol{\alpha}_0, 
\hspace{.1in} \Omega^T\mathbf{1} = \boldsymbol{\alpha}_1, } }\nonumber \\
{{ \int_{\mathcal{Y}} \sum_j \frac{\omega_{ij}}{\alpha_{0_i}}dp_{ij}(\mathbf{x}, \mathbf{y}) = {\mu}_{0_{i}}(\mathbf{x})},  \hspace{.1in} {\int_{\mathcal{X}} \sum_i \frac{\omega_{ij}}{\alpha_{1_j}} dp_{ij}(\mathbf{x}, \mathbf{y}) = {\mu}_{1_{j}}}(\mathbf{y}),}
\label{eq:Reg-OMT_2}
\end{gather}
where $\Omega \in \mathcal{S}^{K -1}$ and $P \in \mathcal{P}^K(\mathcal{X} \times \mathcal{Y})$. Solving the optimization in Eq.~\ref{eq:Reg-OMT_2} returns the OMT transport map, $T^{\nu_0 \rightarrow \nu_1}_{\textsc{omt}}(\mathbf{x})$:
\begin{gather}
\label{eq:T_omt}
       T_{\textsc{omt}}(\mathbf{x}) =  \displaystyle{\sum_{i,j}} \tilde{\omega}_{ij}(\mathbf{x}) T_{ij}(\mathbf{x}) \ , \hspace{.1in}\tilde{\omega}_{ij}(\mathbf{x}) = \displaystyle{\frac{\omega_{ij} \ \mu_{0_{i}}(\mathbf{x})}{\nu_0(\mathbf{x})}} \ .
\end{gather}
Detailed derivations are provided in Appendix~\ref{sec:app_omtmap}.
Although Eq.~\ref{eq:Reg-OMT_2} no longer defines a convex loss, as we show now in Lemma~\ref{lemma:biconvex}, the objective is biconvex. Moreover, while biconvex problems don't have unique solutions generally, Eq.~\ref{eq:Reg-OMT_2} has a unique minimizer that can be obtained efficiently (Theorem~\ref{theorem:uniqGOP}).
\begin{lemma}
\label{lemma:biconvex}
For any $\varepsilon_1, \varepsilon_2 > 0$, $\mathcal{L}_{\varepsilon_1, \varepsilon_2}(\Omega, P)$ is strictly biconvex. \textnormal{(Proof in Appendix~\ref{sec:proofs})}
\end{lemma}
~\citep{floudas1990global} proposed the \textit{Global Optimization Algorithm} (GOP) to solve constrained biconvex problems. It decomposes the optimization into disjoint blocks similar to the \textit{Alternate Convex Search} (ACS) method and exploits the convex substructure of the problem by a primal-relaxed dual
approach~\citep{gorski2007biconvex}. The GOP algorithm is guaranteed to terminate after a finite number of steps for an $\epsilon$-global optimum solution, for any $\epsilon > 0$ ({Theorem 4.11}, {Corollary 4.12}~\citep{gorski2007biconvex}). As mentioned above, the uniqueness of this global optimum is not guaranteed in the general case. However, by exploiting the structure of Eq.~\ref{eq:Reg-OMT_2}, we show that its solution is unique and obtained in a single iteration:
\begin{theorem}
\label{theorem:uniqGOP}
For the optimization problem defined in Eq.~\ref{eq:Reg-OMT_2}, the GOP algorithm converges to a unique solution in a single iteration. \textnormal{(Proof in Appendix~\ref{sec:proofs})}
\end{theorem}
The practical utility of OT maps often depends on the stability of the transport plan with respect to perturbations of the source or target distributions~\citep{divol2025tight}. Such stability guarantees are generally not available for neural OT methods. In contrast, among non-neural OT approaches, EOT–based methods provide formal stability guarantees~\citep{divol2025tight, kassraie2024progressive}. For many applications, stability of transport maps under noise is essential. Since OMT relies on multiple transport operations, a trivial stability result doesn't exist.
In what follows, we introduce regularity conditions on mixture densities ensuring that the transport maps remain well-behaved. We then introduce a formal notion of stability and prove that OMT yields a stable transport map under regularity conditions.
\begin{define}[Regularity conditions]
\label{def:regularity_conditions}
Let $\nu_0(\mathbf{x}) = \sum_{i=1}^{K_0} \alpha_{0_i} \mu_{0_i}(\mathbf{x})$ and $\nu_1(\mathbf{y}) = \sum_{j=1}^{K_1} \alpha_{1_j} \mu_{1_j}(\mathbf{y})$ be continuous densities in $\mathcal{X}, \mathcal{Y} \subset \mathbb{R}^d$. The ordered pair $(\nu_0(\mathbf{x}), \nu_1(\mathbf{y}))$ satisfies the regularity conditions if for all $i,j$ the following hold:
\begin{itemize}
\vspace{-.1in}
    \item $\mu_{0_i}$ and $\mu_{1_j}$ have finite second moments.
    \item $\forall \mathbf{x}\in \mathcal{X}$, and $\forall \mathbf{y}\in \mathcal{Y}$ , $\mu_{0_i}(\mathbf{x}) / \nu_0(\mathbf{x})\|\mathbb{E}_{p_{ij}(\mathbf{y} |\mathbf{x})}[\mathbf{y}\mathbf{y}^T]\|_{\textsc{op}}$ exists and is finite.
    \item $\forall \mathbf{x} \in \mathcal{X}$, $ \mu_{0_i}(\mathbf{x})/\nu_0(\mathbf{x}) \|\nabla \log{\mu_{0_i}(\mathbf{x})}\|_2$ exists and is finite.
\end{itemize}
\end{define}

\begin{prop}
\label{prop:local_lipsch_OMT}
Let $\nu_0(\mathbf{x}) \in M_{K_0}(\mathbb{R}^d)$, $\nu_1(\mathbf{y}) \in M_{K_1}(\mathbb{R}^d)$ be mixture densities satisfying the regularity conditions in Definition~\ref{def:regularity_conditions}. Then $T^{\nu_0 \rightarrow \nu_1}_{\textsc{omt}}(\mathbf{x})$ is Lipschitz continuous.~\textnormal{(Proof in Appendix~\ref{sec:app_localLip})}
\end{prop}
While Lipschitz continuity ensures that the map is well behaved over the joint space, we further require a characterization of how the map responds to perturbations in the underlying measure. This sensitivity is captured by tilt stability, defined as follows.
\begin{define}[Tilt stability (Entropic stability), ~\citep{chen2022localization, bauerschmidt2024stochastic}]
\label{def:tilt}
Let $\mu$ be a probability measure
on $\mathcal{X} \subset \mathbb{R}^d$, with finite second moment, a function $f : \mathbb{R}^d \times \mathbb{R}^d \to \mathbb{R}_+$ and $\alpha > 0$, we say that $\mu$ is $\alpha$-entropically stable with respect to $f$ if
\begin{gather}
    f\left(\mathbb{E}_{\mathcal{T}_{\mathbf{v}}\mu}[\mathbf{x}], \mathbb{E}_{\mu}[\mathbf{x}] \right) \le \alpha D_{KL}(T_{\mathbf{v}}\mu || \mu), \quad \forall \mathbf{v} \in \mathbb{R}^d
\end{gather}
where $\mathcal{T}_{\mathbf{v}}\mu$ denotes the exponential tilt of the probability measure $\mu$ through vector $\mathbf{v} \in \mathbb{R}^d$ such that
\begin{gather*}
  {d \mathcal{T}_{\mathbf{v}}\mu(\mathbf{x})} := \frac{e^{\langle \mathbf{v}, \mathbf{x} \rangle}}{\int e^{\langle \mathbf{v}, \mathbf{z} \rangle} d\mu(\mathbf{z})} d\mu(\mathbf{x}).  
\end{gather*}
\end{define}
\begin{prop}
\label{prop:OMT_tilt_cnt}
 Suppose $\nu_0(\mathbf{x}) \in M_{K_0}(\mathbb{R}^d), \nu_1(\mathbf{y}) \in M_{K_1}(\mathbb{R}^d)$ are mixture models admitting regularity conditions in Definition~\ref{def:regularity_conditions}. Then, the OMT map between these two probability measures is tilt-stable. \textnormal{(Proof in Appendix~\ref{sec:app_tiltstability})}
\end{prop}
\begin{theorem}[Stability of OMT under perturbation]
\label{theorem:stability_main}
Let $\nu_0 \in M_{K_0}(\mathbb{R}^d)$, $\nu_0^{\prime} \in M_{K^{\prime}_0}(\mathbb{R}^d)$, and $\nu_1 \in M_{K_1}(\mathbb{R}^d)$ be mixture probability densities that admit density regularity conditions in Definition~\ref{def:regularity_conditions}. Then there exist constants $a_0, b_0 \in \mathbb{R}_+$ such that
\begin{align*}
    \mathbb{E}_{\nu_0} \left[\|T^{\nu_0 \rightarrow \nu_1}_{\textsc{omt}}(\mathbf{x}) - T^{\nu^{\prime}_0 \rightarrow \nu_1}_{\textsc{omt}}(\mathbf{x}) \| \right] & \le a_0 \mathcal{W}_2 (\nu_{0}, \nu^{\prime}_{0}) + b_0 \mathcal{W}^{1/2}_2 (\nu_{0}, \nu^{\prime}_{0}) \ .
\end{align*}
\textnormal{(Proof in Appendix~\ref{sec:app_tiltstability})}
\end{theorem}
This result establishes that the OMT map remains robust to small variations in the data. 
We remark that the regularity conditions in Definition \ref{def:regularity_conditions} are satisfied by a broad class of distributions, including finite mixtures of sub-Gaussian, Gaussian, and, more generally, exponential-family distributions. In the subsequent sections, we focus primarily on exponential-family mixtures, as they can act as universal approximators for a wide range of density functions.
\OMTmainFig
\vspace{-.1in}
\subsection{Mixture of exponential families}
\label{sec:mix_exp_family}
\vspace{-.1in}
We introduce a general class of mixture models parameterized as $\nu(\mathbf{x})= \sum_{i}^K \alpha_i \mu_i(\mathbf{x})$, where $\mathbf{x} \in \mathbb{R}^d$. Each component, $\mu_i(\mathbf{x})$ is an asymmetric continuous density constructed through a piecewise exponential formulation defined as
\begin{align}
\label{eq:exp_family}
\mu_{i}(\mathbf{x}) = \begin{cases} a_{i} e^{- \|\mathbf{x} - \mathbf{b}_i\|^{p_i}_{\text{C}^T_i \text{C}_i}}  & \text{if } \mathbf{u}^T_i(\mathbf{x} - \mathbf{b}_i) \geq 0 \\
a_{i} e^{-\| \mathbf{x} - \mathbf{b}_i\|^{q_i}_{\text{D}^T_{i} \text{D}_{i}}}  & \text{otherwise}
\end{cases} \ ,  
\end{align}
where $a_i \in \mathbb{R}$ is a normalization constant; $\mathbf{b}_i \in \mathbb{R}^d$ denotes the location parameter; $\mathbf{u}_i \in \mathbb{R}^d$ is a unit normal vector; and $\text{C}_i, \text{D}_i \in \mathbb{R}^{d \times d}$ are scaling matrices governing the right and left tails along the normal direction, respectively. The exponents $p_i, q_i > 0$ control the shape and decay rate of the tail on each side. Here, $\| \mathbf{x} \|^p_{\text{W}} = \left( \sqrt{\mathbf{x}^T \text{W} \mathbf{x} }\right)^p$.
This formulation provides the flexibility necessary to capture complex data exhibiting skewness or heterogeneous decay rates and recovers several classical distributions as special cases through appropriate choices of the shape parameters:

(i) For $p_i = q_i = 1$, the component reduces to a double-exponential (asymmetric Laplace) form.

(ii) For $p_i = q_i = 2$, the exponent is quadratic and reduces to a Gaussian form under symmetry.

Here, we focus on the cases $p_i, q_i \in \{1, 2\}$, as they uniquely admit closed-form expressions for cross-component maps. By deriving a closed-form distance between component pairs, the computational complexity of OMT can be reduced to ${O}(K_0 K_1)$.
\vspace{-.1in}
\subsubsection{Gaussian Measures}
\label{sec:gaussian_measures}
\vspace{-.1in}
Building upon the formulation in Eq.~\ref{eq:exp_family}, when the shape parameters are $p_{0_i}=q_{0_i}=p_{1_j}=q_{1_j}=2$ and the scale matrices are symmetric, i.e., $\mathbf{C}_{0_i}=\mathbf{D}_{0_i}$ and $\mathbf{C}_{1_j}=\mathbf{D}_{1_j}$, across all components, the mixture reduces to a GMM. Specifically, the density functions for source component $i$ and target component $j$ become $\mu_{0_i}(\mathbf{x}) = \mathcal{N}(\mathbf{b}_{0_i}, (\mathbf{C}^T_{0_i}\mathbf{C}_{0_i})^{-1})$ and $\mu_{1_j}(\mathbf{y}) = \mathcal{N}(\mathbf{b}_{1_j}, (\mathbf{D}^T_{1_j}\mathbf{D}_{1_j})^{-1})$, respectively.
In this regime, the component-wise transport maps $T_{ij}(\mathbf{x})$ (Eq.~\ref{eq:T_omt}) and the coupling distributions $p_{ij}(\mathbf{x}, \mathbf{y})$ (Eq.~\ref{eq:Reg-OMT_2}) admit exact, closed-form solutions \citep{janati2020entropic}. Notably, this formulation relaxes the requirement for strictly positive-definite covariance matrices; the OMT map remains well-defined and stable even in the presence of singular covariances \citep{janati2020entropic}. 
\begin{coroll}
\label{coroll:GMM_EOMT}
Let \( \nu_0 \in G_{K_0}(\mathbb{R}^d) \) and \( \nu_1 \in G_{K_1}(\mathbb{R}^d) \) be two Gaussian mixture models (GMMs) in \( \mathbb{R}^d \) with \( K_0 \) and \( K_1 \) components, respectively. Then, the optimal mixture transport map between \( \nu_0 \) and \( \nu_1 \) is itself a Gaussian mixture model with $K$ components, where $K = K_0K_1$.~
\textnormal{(Proof in Appendix~\ref{sec:app_GOMT})}
\end{coroll}
Accordingly, under this framework, the OMT coupling remains strictly within the GMM family, as expressed in Eq.~\ref{eq:omt_finalform} (Figure~\ref{fig:OMTmainFig}a).
\vspace{-.1in}
\subsubsection{Factorized asymmetric exponential measures}
\label{sec:exp_subsection}
\vspace{-.1in}
For the cases where $p_i, q_i \in \{1, 2\}$, we now focus on scenarios where the features are uncorrelated, i.e., $\text{C}_i$ and $\text{D}_i$ are diagonal. Under this structural independence, the multivariate component density can be factorized into a product of independent 1D asymmetric distributions as
\begin{align}
\label{eq:exp_dist_def}
    \mu_i(\mathbf{x}) = a_i \prod_{k=1}^d \exp \left( - \begin{cases} (c_{k_i} (x_k - b_{k_i}))^{p_i} & \text{if } x_k \geq b_{k_i} \\ (d_{k_i} (b_{k_i} - x_k))^{q_i} & \text{otherwise} \end{cases} \right) \ .
\end{align}
This factorization draws a direct theoretical connection to Sliced OT \citep{bonneel2015sliced}. Sliced OT approximates the Wasserstein distance by projecting distributions onto random 1D lines, solving the transport by integrating over all directions. Consequently, $\mathcal{W}_2^2$, decomposes additively across dimensions \citep{villani2021topics}. 
Using this property, the local transport map $T_{ij}(\mathbf{x})$ between source component $i$ and target component $j$ can be computed element-wise via the cumulative distribution function of the source, $F_{0_i}$, and the quantile function of the target, $F_{1_j}$ \citep{villani2021topics}:
\begin{align}
\label{eq:local_1d_map}
    T_{ij}(\mathbf{x}) = \Big[ F^{-1}_{1_{j}}\big(F_{0_{i}}(x_1)\big), \dots, F^{-1}_{1_{j}}\big(F_{0_{i}}(x_d)\big) \Big]^T 
\end{align}
The global transport map for the full mixture is then obtained via a barycentric projection (weighted sum) of these element-wise local maps (Figures~\ref{fig:OMTmainFig}b,c). Complete derivations and closed-form expressions are provided in Appendix~\ref{sec:mix_exp_derivations}.
\setcounter{theorem}{1}
\setcounter{prop}{1}

%% file: main/3-relatedwork.tex
To address the limitations of Sinkhorn-based methods, recent work has turned to deep learning, giving rise to Neural Optimal Transport (Neural OT)~\citep{makkuva2020optimal, korotin2022neural}, which uses neural networks to learn a continuous mapping between distributions, while enforcing theoretical constraints~\citep{genevay2018learning, buzun2024expectile}.
A related line of work directly parameterizes the transport map with neural networks, learning transformations that push samples from a source distribution to a target distribution. This approach is widely used in domain adaptation and generative modeling, where models such as normalizing flows learn invertible maps from simple to complex distributions. A prominent example of neural OT connects diffusion models with the theory of Schrödinger Bridges~\citep{de2021diffusion, shi2023diffusion, gushchin2024entropic, gushchin2024light}, a classic stochastic transport problem. This establishes a learning framework for diffusion models equivalent to solving a Schrödinger Bridge problem, which can be viewed as a form of neural EOT~\citep{gushchin2024entropic}. Broadly, stochastic and neural OT methods require extensive training on large datasets and are often unstable, especially under non-Gaussian noise (stability analysis, Figure~\ref{fig:Noise}).

Alongside neural methods, several non-neural OT solvers have demonstrated competitive performance. P{\scriptsize{ROG}}OT~\citep{kassraie2024progressive} constructs the transport map sequentially, backed by theoretical stability bounds. While this approach can be parallelized using libraries like OTT-Jax~\citep{cuturi2022optimal}, it remains memory-intensive and faces scalability challenges on massive datasets. Alternatively, LOT methods~\citep{scetbon2021low, scetbon2022low, halmos2024low, halmos2025hierarchical} are proposed to enhance efficiency by factorizing the coupling matrix into low-rank components. Within this category, HiRef \citep{halmos2025hierarchical} utilizes a hierarchical refinement strategy across multi-scale partitions to achieve state-of-the-art results among LOT solvers. Nonetheless, it is restricted to learning bijective Monge maps between datasets of equal sample size; it cannot compute transport couplings involving mass splitting or merging, which is a critical requirement for many real-world applications. Moreover, HiRef cannot guarantee a unique solution when the cost matrix is not \textit{r-Monge separable}~\citep{halmos2025hierarchical}.

GMM-based solvers (GMM-OT) represent another class of non-neural methods. These approaches simplify the transport problem by assuming that data distributions belong to a GMM family. This body of work, often formulated as the aggregated Wasserstein distance \citep{chen2019aggregated} or $G\mathcal{W}_2$ \citep{delon2020wasserstein, yachimura2024scegot, fernandes2024optimal}, approximates the global distance by considering transport between individual Gaussian components. Since optimal transport between Gaussians admits closed-form solutions, the computational complexity scales primarily with the number of mixture components.
One problem in $G\mathcal{W}_2$ lies in optimizing component weights, which reduces to a discrete OT problem that often suffers from non-uniqueness and numerical instability, particularly when component covariances are singular or ill-conditioned (see Appendix \ref{sec:OMTvsGMMOT}). The solver proposed in \citep{yachimura2024scegot} addresses this issue using entropic regularization. However, while transport between individual Gaussian components enjoys well-understood stability properties, the stability of the resulting coupling maps under the $G\mathcal{W}_2$ framework remains largely unexplored.
In this work, we extend mixture-based OT to an entropic mixture transport setting by adopting exponential-family distributions for each mixture component. We demonstrate that our formulation is strictly biconvex, ensuring a unique solution for both the transport plan over mixing weights and the component-level distributions. Furthermore, we address a key theoretical gap in existing mixture-based frameworks by showing that the resulting transport map is stable and bounded under perturbations of the marginal distributions.

%% file: main/5-experiment.tex
{\bf Impact of mixture parametrization.} 
We evaluate the proposed method on synthetic 2D and 3D tasks with diverse target distributions to demonstrate its capability and flexibility. Figures \ref{fig:OMTmainFig}, \ref{fig:ExpVsGauss}, and \ref{fig:PathShape} (Appendix~\ref{sec:Appsynth_data}) show that OMT successfully recovers the target shapes across all evaluated scenarios.
Specifically, we evaluate the representational efficiency of OMT under different power terms. Given sufficient components, the framework consistently transports samples to accurately match the target measure. As shown in Figures \ref{fig:ExpVsGauss} and \ref{fig:PathShape}, while the underlying geometry can influence efficiency, e.g., exponential mixtures may require fewer components for structures with sharp boundaries and Gaussians for smoother shapes, both families act as robust universal approximators.

{\bf Stability under noise.} 
To evaluate the robustness of OMT, we conduct an ablation study where noise is injected into the source distribution during training, while evaluation is performed on clean data. We consider two types of perturbations: additive Gaussian noise with variance $\sigma^2$ and dropout noise with probability $p$, capturing both Gaussian and non-Gaussian settings.
The stability of the computed transport maps is assessed on the W2-Benchmark task~\citep{korotin2021neural}. We compare OMT against representative state-of-the-art approaches across different categories: amortized W2-OT~\citep{amos2022amortizing}, ExNOT~\citep{buzun2024enot}, ENOT~\citep{gushchin2024entropic}, P{\scriptsize{ROG}}OT~\citep{kassraie2024progressive}, and GMM-OT~\citep{delon2020wasserstein}. We exclude LOT solvers (HiRef) as they are tailored to matching problems rather than learning transport maps that can be used for unseen samples.
Figures~\ref{fig:Noise} and~\ref{fig:AppBenchMark} (Appendix) summarize the results in terms of relative mean squared error degradation under noise, compared to the noise-free setting, across dimensions $d \in \{2, 16, 64, 128, 256\}$. We report results for OMT with $p,q \in \{1,2\}$. OMT exhibits stronger stability than competing methods, particularly under non-Gaussian perturbations. Notably, while methods such as P{\scriptsize{ROG}}OT operate at the sample level, OMT, similar to neural OT approaches, estimates a continuous transport map at the distribution level, which contributes to improved robustness and computational efficiency (Figure~\ref{fig:RunTime}). Additional experimental results are detailed in Appendix~\ref{sec:AppBenchmark}. Moreover, we provide the stability distinctions between GMM-OT and OMT ($p=q=2$) in Appendix~\ref{sec:OMTvsGMMOT}.

For the remaining experiments, we use the symmetric quadratic OMT formulation, which typically requires fewer mixture components, especially for PCA- or nonlinearly projected, normalized embeddings.

\Noize

\TabsciPlex

{\bf Single-cell Data.} OT has emerged as a powerful tool in computational biology, with applications such as aligning cell populations across conditions and inferring their trajectories over time~\citep{tong2020trajectorynet, bunne2023learning, bunne2024optimal}. Here, we focus on two single-cell modalities, (i) single-cell RNA sequencing (scRNA-seq) and (ii) single-cell transcriptomic profiling with multiplexed error-robust fluorescence \textit{in situ} hybridization (MERFISH). scRNA-seq generates high-dimensional molecular profiles by measuring the expression of thousands of genes at single-cell resolution, whereas MERFISH provides spatially resolved transcriptomic profiles, capturing the organization of cells within tissue.
We consider three scRNA-seq datasets: one human dataset, sci-Plex~\citep{srivatsan2020massively}, which is used for benchmarking~\citep{cuturi2023monge,janati2020entropic}, and two recent 10x Genomics mouse brain datasets, one collected during development~\citep{gao2024continuous} and the other during aging~\citep{jin2025brain}. 
For the sci-Plex dataset, following prior work, we focus on a subset comprising three cell lines (A549, K562, and MCF7) exposed to five cancer treatments for 24 hours. Following the preprocessing steps recommended in \cite{cuturi2023monge}, the final dataset contains $77,920$ cells and $34,636$ genes~(see Appendix~\ref{sec:AppscRNAeq}) . Similarly, for the transport analysis, we perform dimensionality reduction using PCA, retaining the same number of PCs as in~\cite{kassraie2024progressive}.
Table~\ref{tab:sciPlex} and Figure~\ref{fig:AppsciplexPCA} summarize the performance of OMT against comparable non-neural solvers, demonstrating that OMT consistently outperforms these baselines across all treatment conditions.
%
\scRNAseq
\TabMERFISH
We note that, here, the data subset for each task is relatively small ($ \ll 10^5$). While this scale is computationally feasible for sample-based approaches like P{\scriptsize{ROG}}OT, it does not represent the large datasets in modern single-cell studies. 
For further validation, we extend our analysis by applying OMT to large-scale mouse brain scRNA-seq datasets spanning the full lifespan from development to aging.

For brain development, we use data from the visual cortex spanning a wide period from embryonic days to postnatal days~(E11.5-P28)~\citep{gao2024continuous}. 
For aging, we consider data from ~\cite{jin2025brain} collected from $108$ mice, spanning six brain regions at two timepoints: adult (P53–69) and aged (P540–553).
Our analysis focuses on the cellular dynamics of the oligodendrocyte lineage, including oligodendrocyte precursor cells (OPCs) and mature oligodendrocytes (Oligos). These glial cells, which are responsible for myelinating axons to facilitate neural communication, exhibit significant heterogeneity in their lifespan and function, making them a suitable candidate for studying time-dependent cellular transitions~\citep{marques2016oligodendrocyte, jin2025brain}. 
After preprocessing (Appendix~\ref{sec:AppscRNAeq}), the data includes $32,998$ cells and $9,900$ highly variable genes (HVGs) from the developmental data, alongside $253,468$ cells and $9,359$ HVGs from the aging dataset. 
We utilized a VAE model to learn a compressed representation of the cells. The OMT model was then trained on these low-dimensional embeddings~($d_{z}=10$). 
OMT is applied across 11 consecutive time pairs between E11.5 and P28 for the developmental data, and between adult and aged time points for aging data. 
\MERFISH
Figure~\ref{fig:scRNAseq} summarizes the analysis of the mouse datasets. The UMAP plots show that the cell population transported by the model, whether forward or backward in time, closely mirrors the empirical cell distribution at the target timepoints. This demonstrates the model's ability to learn the global distribution across cell subclasses.
The right panels of the figure illustrate the developmental and aging trajectories revealed by our method. The transport map reveals the known developmental pathway, beginning with neuroepithelial cells (NECs) that mature into radial glia (RG). 
See Appendix~\ref{sec:AppscMouseFigs} for additional details and evaluations.

We next extend our study to another data modality, with much larger single-cell samples, using MERFISH data from the Vizgen MERFISH Mouse Brain Receptor Map dataset. We select the same two slices studied in \cite{halmos2025hierarchical}, containing $85,958$ and $84,172$ cells, respectively. The alignment is formulated as a transport problem using only the spatial coordinates of the cells. Using the computed transport map, we impute gene expression in aligned cells based on the source slice data~(Figures~\ref{fig:MERFISH}, \ref{fig:AppMERFISHGMMs}, \ref{fig:AppMERFISHgenes}). Table~\ref{tab:MERFISH} benchmarks OMT against low-rank OTs, and GMM-OT. Here, we report performance across $5$ marker genes based on cosine similarity and transport cost. The former measures the agreement between imputed gene expression at the transported locations and the ground-truth expression of spatially adjacent cells in the target slice. OMT consistently outperforms all baselines in both metrics. See Appendices~\ref{sec:scTraining} and \ref{sec:AppscMERFISHFigs} for additional details and assessments.

\ImageNet
{\bf Image Datasets.} To evaluate the scalability of OMT in high-dimensional settings, we perform an alignment task on ImageNet \citep{deng2009imagenet} following the setup of~\cite{halmos2025hierarchical} (Table~\ref{tab:ImageNet}). We further assess OMT on unpaired image-to-image translation tasks using MNIST \citep{lecun1998mnist} and CIFAR-10 \citep{krizhevskycifar} (see Figures~\ref{fig:ImageMNISTCIFAR},~\ref{fig:AppImagTran}, and Table~\ref{tab:fid} in the Appendix). Additional experimental results, qualitative visualizations, and implementation details are provided in Appendix~\ref{sec:AppImageTasks}. These results demonstrate the applicability and effectiveness of OMT in high-dimensional image tasks.
\ImageMNISTCIFAR

%% file: main/6-conclusion.tex
This work introduced OMT, a framework that leverages entropic mixture transport to enhance discrete sample-based transportation through mixture-level transport. OMT achieves high computational efficiency through a strictly biconvex formulation and by exploiting the closed-form structure of exponential-family distributions in solving the mixture transport problem. We demonstrated, both theoretically and empirically, that OMT transport maps remain stable and bounded under perturbations of the underlying distributions. Across diverse data modalities, OMT consistently outperformed established non-neural OT solvers.
Although OMT is currently less suited for high-resolution image generation, its strong stability and computational efficiency make it a promising framework for high-dimensional alignment and biological trajectory inference.
Several directions remain for future work, including extending OMT to unbalanced OT settings for applications with unequal component masses and generalizing the framework from static transport plans to dynamic transport maps for modeling continuous-time trajectories.

%% file: appendix/1-proofs.tex
\begin{lemma}
For any $\varepsilon_1, \varepsilon_2 > 0$, $\mathcal{L}_{\varepsilon_1, \varepsilon_2}(\Omega, P)$ is strictly biconvex.
\end{lemma}
\begin{proof}
A function \( f: \mathcal{X} \times \mathcal{Y} \rightarrow \mathbb{R} \) is called \textit{biconvex} if, for fixed \( x \in \mathcal{X} \), the function \( f(x, y) \) is convex in \( y \), and for fixed \( y \in \mathcal{Y} \), it is convex in \( x \). According to \textbf{Theorem 3.1}~\cite{gorski2007biconvex}, \( f(x, y) \) is biconvex if and only if for all \( (x_1, y_1), (x_1, y_2), (x_2, y_1), (x_2, y_2) \in \mathcal{X} \times \mathcal{Y} \) and all \( \lambda, \tau \in [0,1] \), the following inequality holds:
\begin{equation*}
f(x_{\lambda}, y_{\tau}) \leq \lambda \tau f(x_1, y_1) + (1-\lambda) \tau f(x_2, y_1) + \lambda (1 - \tau) f(x_1, y_2) + (1-\lambda)(1 - \tau) f(x_2, y_2),
\end{equation*}
where \( (x_{\lambda}, y_{\tau}) := \left( \lambda x_1 + (1 - \lambda) x_2,\; \tau y_1 + (1 - \tau) y_2 \right) \). 

Then, for given $\omega_{ij_{\lambda}} = \lambda \tilde{\omega}_{ij} + (1- \lambda) \dbtilde{\omega}_{ij}$ and $dp_{ij_{\tau}} = \tau \tilde{dp}_{ij} + (1- \tau) \dbtilde{dp}_{ij}$, the following inequality must hold.
\begin{equation}
\label{eq:strick_biconvex}
\mathcal{L}_{\varepsilon_1, \varepsilon_2}(\Omega_{\lambda}, P_{\tau}) < \lambda \tau \mathcal{L}_{\varepsilon_1, \varepsilon_2}(\tilde{\Omega}, \tilde{P}) + (1-\lambda)\tau \mathcal{L}_{\varepsilon_1, \varepsilon_2}(\dbtilde{\Omega}, \tilde{P}) + \lambda (1- \tau)\mathcal{L}_{\varepsilon_1, \varepsilon_2}(\tilde{\Omega}, \dbtilde{P})  + (1-\lambda)(1- \tau)\mathcal{L}_{\varepsilon_1, \varepsilon_2}(\dbtilde{\omega}, \dbtilde{P})
\end{equation}
\begin{eqnarray}
    \mathcal{L}_{\varepsilon_1, \varepsilon_2}(\Omega_{\lambda}, P_{\tau}) &=& \displaystyle{\sum^K_{i,j} \left(\lambda \tilde{\omega}_{ij} + (1- \lambda) \dbtilde{\omega}_{ij} \right) \left[\int_{\mathcal{X} \times \mathcal{Y}} \| \mathbf{x}-\mathbf{y} \|^2_2 \ \left(\tau \tilde{dp}_{ij}(\mathbf{x},\mathbf{y}) + (1- \tau) \dbtilde{dp}_{ij}(\mathbf{x},\mathbf{y}) \right) \right]} + \nonumber \\
    && \varepsilon_2 D_{KL}(\Omega_{\lambda} \| \boldsymbol{\alpha}_0 \otimes \boldsymbol{\alpha}_1) + \displaystyle{\sum^K_{i,j} \left(\lambda \tilde{\omega}_{ij} + (1- \lambda) \dbtilde{\omega}_{ij} \right) \varepsilon_1 D_{KL}(p_{{ij}_{\tau}} \| \mu_{0_i} \otimes \mu_{1_j}) } \nonumber \\
    &=& \displaystyle{ \underbrace{\lambda \tau \sum_{i,j} \tilde{\omega}_{ij} \int \| \mathbf{x} - \mathbf{y}\|^2_2 \tilde{dp}_{ij}(\mathbf{x}, \mathbf{y})}_{A}  + \underbrace{(1-\lambda)\tau \sum_{i,j} \dbtilde{\omega}_{ij} \int \| \mathbf{x} - \mathbf{y}\|^2_2 \tilde{dp}_{ij}(\mathbf{x}, \mathbf{y})}_{B}} + \nonumber \\
    && \displaystyle{\underbrace {\lambda (1-\tau) \sum_{i,j} \tilde{\omega}_{ij} \int \| \mathbf{x} - \mathbf{y}\|^2_2 \dbtilde{dp}_{ij}(\mathbf{x}, \mathbf{y})}_{C} + \underbrace{(1-\lambda)(1-\tau) \sum_{i,j} \dbtilde{\omega}_{ij} \int \| \mathbf{x} - \mathbf{y}\|^2_2 \dbtilde{dp}_{ij}(\mathbf{x}, \mathbf{y})}_{D}} + \nonumber \\
    && \displaystyle{\varepsilon_1  \sum_{i,j} \left(\lambda \tilde{\omega}_{ij} + (1- \lambda) \dbtilde{\omega}_{ij} \right) \underbrace{D_{KL}(p_{{ij}_{\tau}} \| \mu_{0_i} \otimes \mu_{1_j})}_{E} + \varepsilon_2 \underbrace{D_{KL}(\Omega_{\lambda} \| \boldsymbol{\alpha}_0 \otimes \boldsymbol{\alpha}_1)}_{F}}.
\end{eqnarray}
\begin{eqnarray}
\lambda \tau \mathcal{L}_{\varepsilon_1, \varepsilon_2}(\tilde{\Omega}, \tilde{P}) &+& (1-\lambda)\tau \mathcal{L}_{\varepsilon_1, \varepsilon_2}(\dbtilde{\Omega}, \tilde{P}) + \lambda (1- \tau)\mathcal{L}_{\varepsilon_1, \varepsilon_2}(\tilde{\Omega}, \dbtilde{P})  + (1-\lambda)(1- \tau)\mathcal{L}_{\varepsilon_1, \varepsilon_2}(\dbtilde{\omega}, \dbtilde{P}) = \nonumber \\
&& \displaystyle{\underbrace{\lambda \tau \sum_{i,j} \tilde{\omega}_{ij} \int \| \mathbf{x} - \mathbf{y}\|^2_2 \tilde{dp}_{ij}(\mathbf{x}, \mathbf{y})}_{A}  + \varepsilon_1 \sum_{i,j} \lambda \tilde{\omega}_{ij} \left(\tau D_{KL}(\tilde{p}_{{ij}_{\tau}} \| \mu_{0_i} \otimes \mu_{1_j}) \right)} + \nonumber \\
&& \varepsilon_2 \tau \lambda D_{KL}(\tilde{\Omega} \| \boldsymbol{\alpha}_0 \otimes \boldsymbol{\alpha}_1) + \displaystyle{\underbrace{(1-\lambda)\tau \sum_{i,j} \dbtilde{\omega}_{ij} \int \| \mathbf{x} - \mathbf{y}\|^2_2 \tilde{dp}_{ij}(\mathbf{x}, \mathbf{y})}_{B}} + \nonumber \\
&& \varepsilon_1 \sum_{i,j} (1-\lambda) \dbtilde{\omega}_{ij} \left(\tau D_{KL}(\tilde{p}_{{ij}_{\tau}} \| \mu_{0_i} \otimes \mu_{1_j}) \right) + \varepsilon_2 \tau (1-\lambda) D_{KL}(\dbtilde{\Omega} \| \boldsymbol{\alpha}_0 \otimes \boldsymbol{\alpha}_1) + \nonumber \\
&& \displaystyle{\underbrace{\lambda (1-\tau) \sum_{i,j} \tilde{\omega}_{ij} \int \| \mathbf{x} - \mathbf{y}\|^2_2 \dbtilde{dp}_{ij}(\mathbf{x}, \mathbf{y})}_{C}} + \nonumber 
\end{eqnarray}
\begin{eqnarray}
&& \displaystyle{\varepsilon_1 \sum_{i,j} \lambda \tilde{\omega}_{ij} \left((1-\tau) D_{KL}(\dbtilde{p}_{{ij}_{\tau}} \| \mu_{0_i} \otimes \mu_{1_j}) \right)} + \varepsilon_2 (1-\tau) \lambda D_{KL}(\tilde{\Omega} \| \boldsymbol{\alpha}_0 \otimes \boldsymbol{\alpha}_1) + \nonumber \\
&& \displaystyle{\underbrace{(1-\lambda)(1-\tau) \sum_{i,j} \dbtilde{\omega}_{ij} \int \| \mathbf{x} - \mathbf{y}\|^2_2 \dbtilde{dp}_{ij}(\mathbf{x}, \mathbf{y}) }_{D}+ \varepsilon_1 \sum_{i,j} (1- \lambda) \dbtilde{\omega}_{ij} \left( (1-\tau) D_{KL}(\dbtilde{p}_{ij} \| \mu_{0_i} \otimes \mu_{1_j}) \right)} +  \nonumber \\
&& \displaystyle{\varepsilon_2 (1-\tau)(1- \lambda) D_{KL}(\dbtilde{\Omega} \| \boldsymbol{\alpha}_0 \otimes \boldsymbol{\alpha}_1)}, \nonumber \\ [.25in]
&=& A+B+C+D + \varepsilon_2 \left(\underbrace{\lambda  D_{KL}(\tilde{\Omega} \| \boldsymbol{\alpha}_0 \otimes \boldsymbol{\alpha}_1) + (1-\lambda) D_{KL}(\dbtilde{\Omega} \| \boldsymbol{\alpha}_0 \otimes \boldsymbol{\alpha}_1)}_{G} \right) + \nonumber \\
&& \varepsilon_1 \displaystyle{\sum_{i,j} \left(\lambda \tilde{\omega}_{ij} + (1-\lambda )\dbtilde{\omega}_{ij}\right) \left( \underbrace{\tau D_{KL}(\tilde{p}_{ij} \| \mu_{0_i} \otimes \mu_{1_j}) + (1-\tau) D_{KL}(\dbtilde{p}_{ij} \| \mu_{0_i} \otimes \mu_{1_j})}_{H} \right)}.
\end{eqnarray}
For any fixed $q$, $D_{KL}(p || q)$ is strictly convex in $p$. Consequently, we have $E < H$ and $F < G$, which together imply that inequality \ref{eq:strick_biconvex} holds.
\end{proof}
\begin{theorem}
\label{theorem:app_uniqGOP}
For the optimization problem defined in \ref{eq:Reg-OMT_2}, the GOP algorithm converges to a unique solution in a single iteration.
\end{theorem}
\begin{proof}
Consider the biconvex optimization problem defined as follows:
\[
\min{\{\mathcal{L}_{\varepsilon_1 \varepsilon_2}(\Omega, P), \ (\Omega, P)\in \Lambda \}}.
\]
Let's begin by selecting an arbitrary initial point $Z_0 = (\Omega_0, P_0) \in \Lambda$, $\Omega_0 = [\omega_{0_{ij}}]$ and set the iteration index $s = 0$. Without loss of generality, we assume that $\omega_{0_{ij}} >0$, for all $i,j$.
We then solve the following convex optimization problem with respect to $P$, keeping $\Omega_s$ fixed.
\begin{gather}
\label{eq:min_p}
\displaystyle{\min_{P} \sum_{i,j} \omega_{0_{ij}} \left[\int_{\mathcal{X} \times \mathcal{Y}} \| \mathbf{x}-\mathbf{y} \|^2_2 \ dp_{ij}(\mathbf{x},\mathbf{y}) + \varepsilon_1 D_{KL}(p_{ij} \| \mu_{0_i} \otimes \mu_{1_j}) \right]} \nonumber \\
\text{s.t.} \hspace{.25in} \displaystyle{\int_{\mathcal{Y}} dP(\mathbf{x}, \mathbf{y}) = \boldsymbol{\mu}_{0}}, \hspace{.1in} \displaystyle{\int_{\mathcal{X}} dP(\mathbf{x}, \mathbf{y}) = \boldsymbol{\mu}_{1}}
\end{gather}
Since the objective function in \ref{eq:min_p} is convex in $P$ for any fixed $\Omega$ and Slater’s condition is satisfied~\cite{floudas1993primal}, strong duality holds. Accordingly, problem \ref{eq:min_p} admits the following strong dual formulation:
\begin{gather}
\label{eq:min_p_dual}
    \displaystyle{\max_{\Phi ,\Psi} \min_{P} \sum_{i,j} \omega_{0_{ij}} \left[\int_{\mathcal{X} \times \mathcal{Y}} \| \mathbf{x}-\mathbf{y} \|^2_2 \ dp_{ij}(\mathbf{x},\mathbf{y}) + \varepsilon_1 D_{KL}(p_{ij} \| \mu_{0_i} \otimes \mu_{1_j}) \right]} - \\  
    \displaystyle{\sum_{i,j} \omega^0_{ij} \left[ \int_{\mathcal{X}} \varphi_{ij}(\mathbf{x}) \left(\int_{\mathcal{Y}} dp_{ij}(\mathbf{x}, \mathbf{y}) - d\mu_{0_i}(\mathbf{x}) \right) + 
 \int_{\mathcal{Y}} \psi_{ij}(\mathbf{y}) \left(\int_{\mathcal{X}} dp_{ij}(\mathbf{x}, \mathbf{y}) - d\mu_{1_j}(\mathbf{y}) \right) \right]}.
\end{gather}
Here, $\phi_{ij}, \psi_{ij} \geq 0$ are the Lagrange multipliers associated with the marginal constraints.

Denoting the inner minimization problem in \ref{eq:min_p_dual} by $\displaystyle{\min_{P}{f(P, \Phi, \Psi)}}$, we seek the optimal policy that minimizes the loss in \ref{eq:min_p_dual}. To do so, we compute the functional derivative of the loss with respect to $dp_{ij}(\mathbf{x}, \mathbf{y})$.
\begin{eqnarray}
\label{eq:p_der}
    \displaystyle{\frac{d {f}}{dp_{ij}(\mathbf{x}, \mathbf{y})}} &=& \omega_{0_{ij}} \left(\| \mathbf{x}-\mathbf{y} \|^2_2 + \varepsilon_1 \log{\displaystyle{\frac{dp_{ij}(\mathbf{x}, \mathbf{y})}{d\mu_{0_i}(\mathbf{x}) d\mu_{1_j}(\mathbf{y})}}} - \varphi_{ij}(\mathbf{x}) - \psi_{ij}(\mathbf{y}) \right). 
\end{eqnarray}
Since the objective is strictly convex in $P$, it admits a unique solution that is independent of $\Omega$.
\begin{eqnarray}
\label{eq:p_opt}
    dp_{ij}^*(\mathbf{x}, \mathbf{y}) &=& \displaystyle{\exp{\left(\frac{\varphi_{ij}(\mathbf{x}) + \psi_{ij}(\mathbf{y}) - \| \mathbf{x}-\mathbf{y} \|^2_2}{\varepsilon_1} \right)}} d\mu_{0_i}(\mathbf{x}) d\mu_{1_j}(\mathbf{y}). \\ \nonumber
\end{eqnarray}
Substituting $dp^*_{ij}$ in \ref{eq:min_p_dual}, the dual form can be written as:
\begin{eqnarray}
\label{eq:max_p_dual}
\displaystyle{\max_{\Phi, \Psi} \sum_{i,j} \omega_{0_{ij}} \left[\int_{\mathcal{X}} \varphi_{ij}(\mathbf{x}) d\mu_{0_i}(\mathbf{x}) - \int_{\mathcal{Y}} \psi_{ij}(\mathbf{y}) d\mu_{0_j}(\mathbf{y}) - \varepsilon_1 \left(\int_{\mathcal{X} \times \mathcal{Y}} dp^*_{i,j}(\mathbf{x}, \mathbf{y}) -1 \right) \right]}.
\end{eqnarray}
To determine the optimal Lagrange multipliers that maximize the dual objective, we differentiate the loss function in \ref{eq:max_p_dual}, denoted as $g(\Phi, \Psi, P^*)$ with respect to each multiplier as follows.
\begin{eqnarray}
\displaystyle{\frac{d g}{d\varphi_{ij}(\mathbf{x})}} &=&  \displaystyle{\omega_{0_{ij}} \left(d\mu_{0_i}(\mathbf{x}) -  \int_{\mathcal{Y}}\exp{(\frac{\varphi_{ij}(\mathbf{x}) + \psi_{ij}(\mathbf{y}) - \|\mathbf{x} - \mathbf{y} \|^2_2}{\varepsilon_1})} d\mu_{0_i}(\mathbf{x}) d\mu_{1_j}(\mathbf{y}) \right)} \ \\ 
\displaystyle{\frac{d g}{d\psi_{ij}(\mathbf{y})}} &=&  \displaystyle{\omega_{0_{ij}} \left(d\mu_{1_j}(\mathbf{y})  -  \int_{\mathcal{X}}\exp{(\frac{\varphi_{ij}(\mathbf{x}) + \psi_{ij}(\mathbf{y}) - \|\mathbf{x} - \mathbf{y} \|^2_2}{\varepsilon_1})} d\mu_{0_i}(\mathbf{x}) d\mu_{1_j}(\mathbf{y}) \right) } \  \\ \nonumber
\end{eqnarray}
Assuming that for all $i,j$, the measures $\mu_{0_i}$ and $\mu_{1_j}$ have finite second-order moments, the pair $(\varphi_{ij}, \psi_{ij})$ is optimal if and only if the following conditions are satisfied.
\begin{equation}
\label{eq:dual_form_condition}
\int_{\mathcal{Y}}\exp{(\frac{\varphi_{ij}(\mathbf{x}) + \psi_{ij}(\mathbf{y}) - \|\mathbf{x} - \mathbf{y} \|^2_2}{\varepsilon_1})} d\mu_{1_j}(\mathbf{y}) = 1, \hspace{.25in}
\int_{\mathcal{X}}\exp{(\frac{\varphi_{ij}(\mathbf{x}) + \psi_{ij}(\mathbf{y}) - \|\mathbf{x} - \mathbf{y} \|^2_2}{\varepsilon_1})} d\mu_{0_i}(\mathbf{x}) = 1,
\end{equation}
which is equivalent to the following expressions for the optimal multipliers:
\begin{eqnarray}
\varphi_{ij}(\mathbf{x}) &=& \displaystyle{- \varepsilon_1 \log{\int_{\mathcal{Y}} \exp{\left(\frac{\psi_{ij}(\mathbf{y}) - \| \mathbf{x}-\mathbf{y} \|^2_2}{\varepsilon_1} \right)}d\mu_{1_j}(\mathbf{y})}}, \label{eq:phi_opt}\\
 \psi_{ij}(\mathbf{y}) &=& \displaystyle{- \varepsilon_1 \log{\int_{\mathcal{X}} \exp{\left(\frac{\varphi_{ij}(\mathbf{x}) - \| \mathbf{x}-\mathbf{y} \|^2_2}{\varepsilon_1} \right)}d\mu_{0_i}(\mathbf{x})}}. \label{eq:psi_opt}\\ \nonumber
\end{eqnarray}
As observed, the optimal Lagrange multipliers are also independent of $\Omega$. Therefore, the unique optimal solution of \ref{eq:min_p_dual} remains the same for any fixed $\Omega_s$, implying
\begin{eqnarray}
\label{eq:p_conditions}
\forall s, \hspace{.1in} \mathcal{L}_{\varepsilon_1 \varepsilon_2}(\Omega_s, P^*) &\leq& \mathcal{L}_{\varepsilon_1 \varepsilon_2}(\Omega_s, P), \nonumber \\ [.1in ]
\forall s \neq s', \hspace{.1in} \mathcal{L}_{\varepsilon_1 \varepsilon_2}(\Omega_s, P^*) &=& \mathcal{L}_{\varepsilon_1 \varepsilon_2}(\Omega_{s'}, P^*), \nonumber \\ [.1in ]
\forall s \neq s', \hspace{.1in} \mathcal{L}_{\varepsilon_1 \varepsilon_2}(\Omega_s, P^*) &=& \mathcal{L}_{\varepsilon_1 \varepsilon_2}(\Omega_{s'}, P_{s'}), \ \text{iff} \ P_{s'}=P^*
\end{eqnarray}
Proceeding to the next step, we set $s=1$, which gives $P_1=P^*$. For a given $\varepsilon_2 > 0$ , we solve the following strictly convex optimization problem with respect to $\Omega=[\omega_{ij}]$, keeping $P$ fixed.
\begin{gather}
\label{eq:min_omega}
\displaystyle{\min_{\Omega} \sum_{i,j} \omega_{ij} \left[\int_{\mathcal{X} \times \mathcal{Y}} \| \mathbf{x}-\mathbf{y} \|^2_2 \ dp^*_{ij}(\mathbf{x},\mathbf{y}) + \varepsilon_1 D_{KL}(p^*_{ij} \| \mu_{0_i} \otimes \mu_{1_j}) \right]} + \varepsilon_2 D_{KL}(\Omega \| \boldsymbol{\alpha}_0 \otimes \boldsymbol{\alpha}_1) \\
\text{s.t.} \hspace{.25in} \mathbf{1}\Omega = \boldsymbol{\alpha}_0, \hspace{.25in} \Omega^T\mathbf{1} = \boldsymbol{\alpha}_1 
\end{gather}
This formulation, similar to equation \ref{eq:min_p}, admits the following dual form:
\begin{gather}
\label{eq:min_omega_dual}
    \displaystyle{\max_{\boldsymbol{\lambda}, \boldsymbol{\tau}} \min_{\Omega} \sum_{i,j} \omega_{ij} \left[\int_{\mathcal{X} \times \mathcal{Y}} \| \mathbf{x}-\mathbf{y} \|^2_2 \ dp^*_{ij}(\mathbf{x},\mathbf{y}) + \varepsilon_1 D_{KL}(p^*_{ij} \| \mu_{0_i} \otimes \mu_{1_j}) \right]} + \varepsilon_2 D_{KL}(\Omega \| \boldsymbol{\alpha}_0 \otimes \boldsymbol{\alpha}_1) - \\
     \displaystyle{\sum_i \lambda_i \left(\sum_j \omega_{ij} - \alpha_{0_i} \right) - \sum_j \tau_j \left(\sum_i \omega_{ij} - \alpha_{1_j}\right)}.
\end{gather}
The inner minimization in \ref{eq:min_omega_dual}, denoted $f^{\prime}(\Omega, \boldsymbol{\lambda}, \boldsymbol{\tau})$, admits a unique solution for the optimal weights. These weights can be derived by computing the functional derivative of the objective with respect to $\omega_{ij}$, yielding:
\begin{eqnarray}
\label{eq:omega_der}
\displaystyle{\frac{d f^{\prime}}{d\omega_{ij}}} &=& \displaystyle{ \underbrace{\int_{\mathcal{X} \times \mathcal{Y}} \| \mathbf{x}-\mathbf{y} \|^2_2 \ dp^*_{ij}(\mathbf{x},\mathbf{y}) + \varepsilon_1 D_{KL}(p^*_{ij} \| \mu_{0_i} \otimes \mu_{1_j})}_{\mathcal{L}_{p^*_{ij}}} + \varepsilon_2 \log{\frac{\omega_{ij}}{\alpha_{0_i} \alpha_{1_j}}} - \lambda_i - \tau_j} \\ 
\omega^*_{ij} &=& \displaystyle{\exp{\left(\frac{\lambda_i+ \tau_j - \mathcal{L}_{p^*_{ij}}}{\varepsilon_2}\right)}} \alpha_{0_i} \alpha_{1_j} \label{eq:omega_opt} 
\end{eqnarray}
To obtain the optimal Lagrangian multipliers that maximize the loss in \ref{eq:min_omega_dual}, denoted $\displaystyle{\max_{\boldsymbol{\lambda}, \boldsymbol{\tau}}g^{\prime}(\boldsymbol{\lambda}, \boldsymbol{\tau}, \Omega^*)}$, we compute the partial derivatives of $g^{\prime}$ with respect to each multiplier.
\begin{eqnarray}
\label{eq:lambda_tau_der}
\displaystyle{\frac{d g^{\prime}}{d\lambda_i}} &=& \displaystyle{ \alpha_{0_i} - \sum_{j} \exp{\left(\frac{\lambda_i + \tau_j - \mathcal{L}_{p^*_{ij}}}{\varepsilon_2}\right)} \alpha_{0_i} \alpha_{1_j}} \\ 
\displaystyle{\frac{d g^{\prime}}{d\tau_j}} &=& \displaystyle{ \alpha_{1_j} - \sum_{i} \exp{\left(\frac{\lambda_i + \tau_j - \mathcal{L}_{p^*_{ij}}}{\varepsilon_2}\right)} \alpha_{0_i} \alpha_{1_j}}
\end{eqnarray}
Solving these yields the optimal multipliers as:
\begin{gather}
\label{eq:lambda_tau_opt}
\lambda_i = -\varepsilon_2 \displaystyle{\log{ \sum_{j} \exp{\left(\frac{\tau_j - \mathcal{L}_{p^*_{ij}}}{\varepsilon_2} \right)} \alpha_{1_j}}}, \hspace{.25in}  \tau_j = -\varepsilon_2 \displaystyle{\log{ \sum_{i} \exp{\left(\frac{\lambda_i - \mathcal{L}_{p^*_{ij}}}{\varepsilon_2} \right)} \alpha_{0_i}}}
\end{gather}
Since the optimal weight in \ref{eq:omega_opt} minimizes $\mathcal{L}_{\varepsilon_1 \varepsilon_2}(\Omega, P_1)$ uniquely for the fixed choice $P_1=P^*$, it follows that:
\begin{gather}
\mathcal{L}_{\varepsilon_1 \varepsilon_2}(\Omega^*, P_1) \leq \mathcal{L}_{\varepsilon_1 \varepsilon_2}(\Omega_0, P_1) \leq \mathcal{L}_{\varepsilon_1 \varepsilon_2}(\Omega_0, P_0)
\end{gather}
Now, let's update the optimization by setting $\Omega_1 = \Omega^*$, and advancing to step $s=2$. According to \ref{eq:p_conditions}, the next update satisfies:
\begin{gather}
\mathcal{L}_{\varepsilon_1 \varepsilon_2}(\Omega_1, P^*) 
 = \displaystyle{\min_{P} \mathcal{L}_{\varepsilon_1 \varepsilon_2}(\Omega_1, P)}. 
\end{gather}
Since we find that $P_2=P_1=P^*$ and consequently, $\Omega_2=\Omega_1=\Omega^*$, the stopping criterion of the overall alternating optimization procedure is met after just a single iteration.
\end{proof}
\subsection{Transport map in OMT}
\label{sec:app_omtmap}
By the definition of the OMT problem in \ref{eq:emot}, the corresponding transport map satisfies
\begin{gather*}
   T^{\nu_0 \rightarrow \nu_1}_{\textsc{omt}}(\mathbf{x}) =  \int \mathbf{y} \pi_{\textsc{omt}}(\mathbf{y} | \mathbf{x}) d\mathbf{y},
\end{gather*}
where $\pi_{\textsc{omt}}(\mathbf{x}, \mathbf{y}) \in M_{K}(\mathbb{R}^d)$ (Eq.~\ref{eq:omt_finalform}). The conditional distribution can be expressed as
\begin{align}
\label{eq:OMT_cpl}
   \pi_{\textsc{omt}}(\mathbf{y} | \mathbf{x}) &= \displaystyle{\frac{\pi_{\textsc{omt}}(\mathbf{x}, \mathbf{y})}{\int_{\mathbf{y}^{\prime}} \pi_{\textsc{omt}}(\mathbf{x}, \mathbf{y}^{\prime}) d\mathbf{y}^{\prime}}} \nonumber \\[.1in] 
   &= \displaystyle{\frac{ \displaystyle \sum_{i,j} \omega_{ij} p_{ij}(\mathbf{x}, \mathbf{y})}{\displaystyle \sum_{r,k}
  {\omega}_{rk} \int_{\mathbf{y}^{\prime}} p_{rk}(\mathbf{x}, \mathbf{y}^{\prime}) d\mathbf{y}^{\prime}}} = \displaystyle{\frac{ \displaystyle \sum_{i,j} \omega_{ij} p_{i}(\mathbf{x}) p_{ij}(\mathbf{y} | \mathbf{x})}{\displaystyle \sum_{r,k} {\omega}_{rk} \int_{\mathbf{y}^{\prime}} p_{rk}(\mathbf{x}, \mathbf{y}^{\prime}) d\mathbf{y}^{\prime}}} \nonumber \\[.1in]
   &= \displaystyle{\sum_{i,j}} \underbrace{\displaystyle{\frac{\omega_{ij} p_{i}(\mathbf{x})}{\displaystyle \sum_{r,k}
  {\omega}_{rk} p_{r}(\mathbf{x})}}}_{\tilde{\omega}_{ij}(\mathbf{x})} p_{ij}(\mathbf{y} | \mathbf{x}) = \displaystyle \sum_{i,j} \tilde{\omega}_{ij}(\mathbf{x}) p_{ij}(\mathbf{y} | \mathbf{x}), 
\end{align}
where $\displaystyle \sum_{i,j} \tilde{\omega}_{ij}(\mathbf{x}) = 1$. Consequently, the transport map can be decomposed as
\begin{align}
   T^{\nu_0 \rightarrow \nu_1}_{\textsc{omt}}(\mathbf{x}) &=  \displaystyle{\sum_{i,j}} \tilde{\omega}_{ij}(\mathbf{x}) \int \mathbf{y} p_{ij}(\mathbf{y} | \mathbf{x}) d\mathbf{y} = \displaystyle{\sum_{i,j}} \tilde{\omega}_{ij}(\mathbf{x}) T^{{\mu}_{i_0} \to {\mu}_{j_1}}(\mathbf{x}), 
\label{eq:omtMap}
\end{align}
where $T^{{\mu}_{i_0} \to {\mu}_{j_1}}(\mathbf{x})$ denotes the entropic map transporting the $i$-th component of $\nu_0$ to the $j$-th component of $\nu_1$. For notational simplicity, in what follows we denote $T^{\nu_0 \rightarrow \nu_1}_{\textsc{omt}}(\mathbf{x}) := T_{\textsc{omt}}(\mathbf{x})$, with
\begin{gather}
\label{eq:app_T_omt}
       T_{\textsc{omt}}(\mathbf{x}) =  \displaystyle{\sum_{i,j}} \tilde{\omega}_{ij}(\mathbf{x}) T_{ij}(\mathbf{x}) \ .
\end{gather}
\subsection{Covariance of $\pi_{\textsc{omt}}(\mathbf{y} | \mathbf{x})$}
\label{sec:app_cov}
The conditional covariance of $\pi_{\textsc{omt}}(\mathbf{y}|\mathbf{x})$ (Eq.~\ref{eq:OMT_cpl}) is expressed as:
\begin{align*}
\text{Cov}_{\pi_{\textsc{omt}}}(Y|X=\mathbf{x}) = E_{\pi_{\textsc{omt}}}[\mathbf{y}\mathbf{y}^T] - \overline{\mathbf{m}}_{\mathbf{y}|\mathbf{x}}\overline{\mathbf{m}}_{\mathbf{y}|\mathbf{x}}^T
\end{align*}
where $\overline{\mathbf{m}}_{\mathbf{y}|\mathbf{x}}$ denotes the global conditional mean, defined as
\begin{gather}
\label{eq:cond_mean}
\overline{{\mathbf{m}}}_{\mathbf{y}|\mathbf{x}} := \sum_{i,j} \tilde{\omega}_{ij}(\mathbf{x}) \mathbf{m}^{ij}_{\mathbf{y}|\mathbf{x}} \ 
\end{gather}
Moreover, the total second moment is calculated as the weighted sum of the local second moments:
\begin{align*}
E_{\pi_{\textsc{omt}}}[\mathbf{y}\mathbf{y}^T] = \displaystyle \sum_{i,j} \tilde{\omega}_{ij}(\mathbf{x}) \int \mathbf{y}\mathbf{y}^T p_{ij}(\mathbf{y}|\mathbf{x}) d\mathbf{y} = \sum_{i,j} \tilde{\omega}_{ij}(\mathbf{x}) E_{ij}[\mathbf{y}\mathbf{y}^T] \ .
\end{align*}
By substituting the identity for the local covariance of each component, $E_{ij}[\mathbf{y}\mathbf{y}^T] = \text{Cov}^{ij}_{\mathbf{y}|\mathbf{x}} + \mathbf{m}^{ij}_{\mathbf{y}|\mathbf{x}} (\mathbf{m}^{ij}_{\mathbf{y}|\mathbf{x}})^T$, into the expression for the total second moment, we obtain:
\begin{align}
\label{eq:total_cov}
    \text{Cov}_{\pi_{\textsc{omt}}}(Y|X=\mathbf{x}) &= \displaystyle \sum_{i,j} \tilde{\omega}_{ij}(\mathbf{x}) \text{Cov}^{ij}_{\mathbf{y}|\mathbf{x}}  + \displaystyle \sum_{i,j} \tilde{\omega}_{ij}(\mathbf{x}) \left(\mathbf{m}^{ij}_{\mathbf{y}|\mathbf{x}} -\overline{\mathbf{m}}_{\mathbf{y}|\mathbf{x}} \right)\left(\mathbf{m}^{{ij}}_{\mathbf{y}|\mathbf{x}} - \overline{\mathbf{m}}_{\mathbf{y}|\mathbf{x}} \right)^T,
\end{align}
where $\text{Cov}^{ij}_{\mathbf{y}|\mathbf{x}}$ and $\mathbf{m}^{ij}_{\mathbf{y}|\mathbf{x}}$ represent the conditional covariance and mean of the local transport plan $p_{ij}(\mathbf{y}|\mathbf{x})$, respectively:
\begin{gather}
\label{eq:cond_cov_ij}
\text{Cov}^{ij}_{\mathbf{y}|\mathbf{x}} := \text{Cov}_{{p_{ij}(\mathbf{y}|\mathbf{x})}}(Y|X=\mathbf{x}) \ 
\end{gather}
Applying the operator norm and the triangle inequality gives
 \begin{align}
 \label{eq:overallCov_u_v}
     \|\text{Cov}_{\pi_{\textsc{omt}}}(Y|X=\mathbf{x})\|_{\text{op}} &= \| \displaystyle \sum_{i,j} \tilde{\omega}_{ij}(\mathbf{x}) \text{Cov}^{ij}_{\mathbf{y}|\mathbf{x}}  + \displaystyle \sum_{i,j} \tilde{\omega}_{ij}(\mathbf{x}) \left(\mathbf{m}^{ij}_{\mathbf{y}|\mathbf{x}} -\overline{\mathbf{m}}_{\mathbf{y}|\mathbf{x}} \right)\left(\mathbf{m}^{{ij}}_{\mathbf{y}|\mathbf{x}} - \overline{\mathbf{m}}_{\mathbf{y}|\mathbf{x}} \right)^T \|_{\text{op}} \nonumber \\
     & \le \displaystyle \|  \sum_{i,j} \tilde{\omega}_{ij}(\mathbf{x}){\text{Cov}^{ij}_{\mathbf{y}|\mathbf{x}}} \|_{\text{op}} + \| \displaystyle \sum_{i,j} \tilde{\omega}_{ij}(\mathbf{x}) \left(\mathbf{m}^{ij}_{\mathbf{y}|\mathbf{x}} -\overline{\mathbf{m}}_{\mathbf{y}|\mathbf{x}} \right)\left(\mathbf{m}^{{ij}}_{\mathbf{y}|\mathbf{x}} - \overline{\mathbf{m}}_{\mathbf{y}|\mathbf{x}} \right)^T \|_{\text{op}} \nonumber \\
     & \le \displaystyle \underbrace{\sum_{i,j} \tilde{\omega}_{ij}(\mathbf{x})  \|{\text{Cov}^{ij}_{\mathbf{y}|\mathbf{x}}} \|_{\text{op}}}_{u(\mathbf{x})} + \underbrace{\sum_{i,j} \tilde{\omega}_{ij}(\mathbf{x}) \| \left(\mathbf{m}^{ij}_{\mathbf{y}|\mathbf{x}} -\overline{\mathbf{m}}_{\mathbf{y}|\mathbf{x}} \right)\left(\mathbf{m}^{{ij}}_{\mathbf{y}|\mathbf{x}} - \overline{\mathbf{m}}_{\mathbf{y}|\mathbf{x}} \right)^T \|_{\text{op}} }_{v(\mathbf{x})} \\
     &\leq u_{\text{max}} + v_{\text{max}} \ .
     \label{eq:sup_total_cov}
\end{align}
where 
\begin{gather}
\label{eq:U_max}
u_{\text{max}}:= \displaystyle \sup_{\mathbf{x}} \{ \sum_{i,j} \tilde{\omega}_{ij}(\mathbf{x}) \|{\text{Cov}^{ij}_{\mathbf{y}|\mathbf{x}}} \|_{\text{op}} \}
\end{gather}
and
\begin{align}
\label{eq:bound_m_m_ij}
     v_{\text{max}} & = \displaystyle \sup_{\mathbf{x}} \  \{ \displaystyle \sum_{i,j}  \tilde{\omega}_{ij}(\mathbf{x})  \|\left(\mathbf{m}^{ij}_{\mathbf{y}|\mathbf{x}} -\bar{\mathbf{m}}_{\mathbf{y}|\mathbf{x}} \right)\left(\mathbf{m}^{{ij}}_{\mathbf{y}|\mathbf{x}} - \bar{\mathbf{m}}_{\mathbf{y}|\mathbf{x}} \right)^T \|_{\text{op}} \} \nonumber \\
     & = \displaystyle \sup_{\mathbf{x}} \  \{ \displaystyle \sum_{i,j} \tilde{\omega}_{ij}(\mathbf{x}) \left( \mathbf{m}^{{ij}}_{\mathbf{y}|\mathbf{x}} - \bar{\mathbf{m}}_{\mathbf{y}|\mathbf{x}}\right)^T \left( \mathbf{m}^{{ij}}_{\mathbf{y}|\mathbf{x}} - \bar{\mathbf{m}}_{\mathbf{y}|\mathbf{x}}\right)  \} 
\end{align}
where the normalized weights are defined as
\begin{align}
    \tilde{\omega}_{ij}(\mathbf{x}) &= \displaystyle{\frac{\omega_{ij} \ \mu_{0_{i}}(\mathbf{x})}{\nu_0(\mathbf{x})}} \ ,  \hspace{.5in} \text{where }\displaystyle \sum_{i,j} \omega_{ij}=1 \nonumber \\
\end{align}
and 
\begin{align}
\text{Cov}^{ij}_{\mathbf{y}|\mathbf{x}} =& \mathbb{E}_{p_{ij}(\mathbf{y} | \mathbf{x})} \left[\mathbf{y}\mathbf{y}^T \right] - \mathbb{E}_{p_{ij}(\mathbf{y} | \mathbf{x})} \left[\mathbf{y}\right] \mathbb{E}^T_{p_{ij}(\mathbf{y} | \mathbf{x})} \left[\mathbf{y}\right] \ , \nonumber \\
=& \mathbb{E}_{p_{ij}(\mathbf{y} | \mathbf{x})} \left[\mathbf{y}\mathbf{y}^T \right] - \mathbf{m}^{{ij}}_{\mathbf{y}|\mathbf{x}} {\mathbf{m}^{ij}}^T_{\mathbf{y}|\mathbf{x}} \ .
\end{align}
Under the regularity conditions specified in Definition~\ref{def:regularity_conditions}, the weighted second moments are bounded as $ \displaystyle \frac{\mu_{0_i}(\mathbf{x})}{\nu_0(\mathbf{x})} \| \mathbb{E}_{p_{ij}(\mathbf{y}|\mathbf{x})} \left[\mathbf{y}\mathbf{y}^T \right]\|_{\text{op}} < C_i$. 

For a random vector $\mathbf{y} \in \mathbb{R}^d$, we have:
$$\mathbb{E}[\|\mathbf{y}\|^2] = \mathbb{E}[\text{Tr}(\mathbf{y}\mathbf{y}^T)] = \text{Tr}(\mathbb{E}[\mathbf{y}\mathbf{y}^T]) \le d \| \mathbb{E}[\mathbf{y} \mathbf{y}^T]\|_{\text{op}} $$ 
By Jensen's inequality, it further holds that $\|\mathbb{E}[\mathbf{y}]\|^2 \le \mathbb{E}[\|\mathbf{y}\|^2]$. Consequently, the magnitudes of both the local mean and the covariance are bounded as follows.
\begin{gather}
\label{eq:bound_m_ij}
 \displaystyle \frac{\mu_{0_i}(\mathbf{x})}{\nu_0(\mathbf{x})} \| \mathbf{m}^{ij}_{\mathbf{y}|\mathbf{x}} {\mathbf{m}^{ij}}^T_{\mathbf{y}|\mathbf{x}}\|_{\text{op}} = \displaystyle \frac{\mu_{0_i}(\mathbf{x})}{\nu_0(\mathbf{x})} \| \mathbf{m}^{ij}_{\mathbf{y}|\mathbf{x}} \|^2 < \infty  \\
\displaystyle \frac{\mu_{0_i}(\mathbf{x})}{\nu_0(\mathbf{x})} \|\text{Cov}^{ij}_{\mathbf{y}|\mathbf{x}} \|_{\text{op}} < \infty \ .
\end{gather}
Consequently, the operator norm of the OMT conditional covariance is uniformly bounded as follows.
\begin{gather}
\label{eq:omt_sup_cov}
    \displaystyle \sup_{\mathbf{x}} \ \{ \|  \text{Cov}_{\pi_{\textsc{omt}}}(Y|X=\mathbf{x}) \|_{\text{op}} \} = u_{\text{max}} + v_{\text{max}}=C_{\pi_{\textsc{omt}}} \ < \infty
\end{gather}
\subsection{Lipschitz continuity of $T_{\textsc{omt}}(\mathbf{x})$}
\label{sec:app_localLip}
\begin{define}[Lipschitzness]
\label{define:Lipschitzness}
    A function \(f : \mathbb{R}^d \to \mathbb{R}^m\) is Lipschitz continuous if there exists a constant \(L > 0\) such that
\begin{gather*}
    \forall x, x' \in \mathbb{R}^d, \quad \| f(x) - f(x') \| \le L \| x - x' \|.
\end{gather*}
\end{define}
\begin{prop}
Let $\nu_0(\mathbf{x}) \in M_{K_0}(\mathbb{R}^d)$, $\nu_1(\mathbf{y}) \in M_{K_1}(\mathbb{R}^d)$ be mixture densities satisfying the regularity conditions in Definition~\ref{def:regularity_conditions}. Then $T^{\nu_0 \rightarrow \nu_1}_{\textsc{omt}}(\mathbf{x})$ is Lipschitz continuous.
\end{prop}

\begin{proof}
To establish the Lipschitz continuity of the transport map, it is sufficient to show that the operator norm of its Jacobian is bounded. We therefore begin by computing the Jacobian of the transport map.
\begin{gather}
\label{eq:DerivOMT}
    \nabla T_{\textsc{omt}}(\mathbf{x}) = \displaystyle \sum_{i,j} \tilde{\omega}_{ij}(\mathbf{x}) \nabla T_{ij}(\mathbf{x}) + T_{ij}(\mathbf{x}) \nabla^T \tilde{\omega}_{ij}(\mathbf{x})  ,
\end{gather}
where 
\begin{gather}
\label{eq:DerivT}
    \forall i,j \hspace{.1in} \nabla T_{ij}(\mathbf{x}) = \displaystyle \int \mathbf{y} \nabla_{x} dp_{ij}(\mathbf{y} | \mathbf{x}), \hspace{.15in} \nabla \tilde{\omega}_{ij}(\mathbf{x}) = \nabla \left(\displaystyle{\frac{\omega_{ij} \int_{\mathbf{y}} dp_{ij}(\mathbf{x}, \mathbf{y})}{\displaystyle \sum_{r,k} \omega_{rk} \int_{\mathbf{y}^{\prime}} dp_{rk}(\mathbf{x}, \mathbf{y}^{\prime})}} \right) \ .
\end{gather}
For notational convenience, define $dp_{ij}(\mathbf{y} | \mathbf{x}) = p_{ij}(\mathbf{y} | \mathbf{x}) d\mathbf{y}$, $d\mu_{0_i}(\mathbf{x}) = \mu_{0_i}(\mathbf{x}) d\mathbf{x}$, and $d\mu_{1_j}(\mathbf{y}) = \mu_{1_j}(\mathbf{y}) d\mathbf{y}$.

Following the derivation in proof of Theorem~\ref{theorem:app_uniqGOP}, we obtain
\begin{gather*}
    p_{ij}(\mathbf{x} , \mathbf{y}) = \displaystyle{\exp{\left(\frac{\varphi_{ij}(\mathbf{x}) + \psi_{ij}(\mathbf{y}) - \| \mathbf{x}-\mathbf{y} \|^2_2}{\varepsilon_1} \right)}} \mu_{0_i}(\mathbf{x}) \mu_{1_j}(\mathbf{y}), \\
    \varphi_{ij}(\mathbf{x}) = \displaystyle{- \varepsilon_1 \log{\int_{\mathcal{Y}} \exp{\left(\frac{\psi_{ij}(\mathbf{y}) - \| \mathbf{x}-\mathbf{y} \|^2_2}{\varepsilon_1} \right)}\mu_{1_j}(\mathbf{y}) d\mathbf{y}}}.
\end{gather*}
This implies
\begin{align}
    p_{ij}(\mathbf{y} | \mathbf{x}) &= \displaystyle{\frac{p_{ij}(\mathbf{x}, \mathbf{y})}{\int_{\mathbf{y}^{\prime}} {p_{ij}(\mathbf{x}, \mathbf{y}^{\prime}) d\mathbf{y}^{\prime}}}} \\
    &= \displaystyle{\frac{\exp{\left(\frac{\varphi_{ij}(\mathbf{x})}{\varepsilon_1} \right)} \mu_{0_i}(\mathbf{x}) \exp{\left(\frac{\psi_{ij}(\mathbf{y}) - \| \mathbf{x}-\mathbf{y} \|^2_2}{\varepsilon_1} \right)} \mu_{1_j}(\mathbf{y})}{\exp{\left(\frac{\varphi_{ij}(\mathbf{x})}{\varepsilon_1} \right)} \mu_{0_i}(\mathbf{x}) \int_{\mathbf{y}^{\prime}} {\exp{\left(\frac{\psi_{ij}(\mathbf{y}^{\prime}) - \| \mathbf{x}-\mathbf{y}^{\prime} \|^2_2}{\varepsilon_1} \right)} \mu_{1_j}(\mathbf{y}^{\prime}) d\mathbf{y}^{\prime}}}} \nonumber \\[.1in]
    & = \displaystyle{\frac{\exp{\left(\frac{\psi_{ij}(\mathbf{y}) - \| \mathbf{x}-\mathbf{y} \|^2_2}{\varepsilon_1} \right)} \mu_{1_j}(\mathbf{y})}{\int_{\mathbf{y}^{\prime}} {\exp{\left(\frac{\psi_{ij}(\mathbf{y}^{\prime}) - \| \mathbf{x}-\mathbf{y}^{\prime} \|^2_2}{\varepsilon_1} \right)} \mu_{1_j}(\mathbf{y}^{\prime})d\mathbf{y}^{\prime}}}}
\label{eq:condtranspmap}
\end{align}
\begin{align}
    \log{p_{ij}(\mathbf{y} | \mathbf{x})} &= \log{\left(\exp{\left(\frac{\psi_{ij}(\mathbf{y}) - \| \mathbf{x}-\mathbf{y} \|^2_2}{\varepsilon_1} \right)} \mu_{1_j}(\mathbf{y})\right)} - \log{\left(\int_{\mathbf{y}^{\prime}} {\exp{\left(\frac{\psi_{ij}(\mathbf{y}^{\prime}) - \| \mathbf{x}-\mathbf{y}^{\prime} \|^2_2}{\varepsilon_1} \right)} \mu_{1_j}(\mathbf{y}^{\prime})d\mathbf{y}^{\prime}}\right)} \nonumber \\[.1in]
    &= \left(\frac{\psi_{ij}(\mathbf{y}) - \| \mathbf{x}-\mathbf{y} \|^2_2}{\varepsilon_1} \right) + \log{\mu_{1_j}(\mathbf{y})} - \log{\left(\int_{\mathbf{y}^{\prime}} {\exp{\left(\frac{\psi_{ij}(\mathbf{y}^{\prime}) - \| \mathbf{x}-\mathbf{y}^{\prime} \|^2_2}{\varepsilon_1} \right)} \mu_{1_j}(\mathbf{y}^{\prime})d\mathbf{y}^{\prime}}\right)} \nonumber \\[.1in]
    &= \left(\frac{\psi_{ij}(\mathbf{y}) - \| \mathbf{x}-\mathbf{y} \|^2_2}{\varepsilon_1} \right) + \log{\mu_{1_j}(\mathbf{y})} + \varepsilon^{-1} \varphi_{ij}(\mathbf{x}) \label{eq:cond_pij}
\end{align}
Differentiating with respect to $\mathbf{x}$,
\begin{align}
    \nabla_{\mathbf{x}} \log{p_{ij}(\mathbf{y} | \mathbf{x})}
    &= -2 \varepsilon^{-1} (\mathbf{x} - \mathbf{y})^{T} +\varepsilon^{-1} \nabla \varphi_{ij}(\mathbf{x}) . \label{eq:JacLogCondP}
\end{align}
Next, using the definition of $\varphi_{ij}(\mathbf{x})$,
\begin{align}
    \nabla \varphi_{ij}(\mathbf{x})
    &= 2 \displaystyle{\frac{\int_{\mathbf{y}} (\mathbf{x} - \mathbf{y})^T\exp{\left(\frac{\psi_{ij}(\mathbf{y}) - \| \mathbf{x}-\mathbf{y} \|^2_2}{\varepsilon_1} \right)} \mu_{1_j}(\mathbf{y}) d\mathbf{y}}{\int_{\mathbf{y}^{\prime}} {\exp{\left(\frac{\psi_{ij}(\mathbf{y}^{\prime}) - \| \mathbf{x}-\mathbf{y}^{\prime} \|^2_2}{\varepsilon_1} \right)} \mu_{1_j}(\mathbf{y}^{\prime}) d\mathbf{y}^{\prime}}}} \nonumber \\[.1in]
    &= 2\mathbf{x}^T - 2 \displaystyle \int_{\mathbf{y}} \mathbf{y}^{T} \displaystyle{\frac{\exp{\left(\frac{\psi_{ij}(\mathbf{y}) - \| \mathbf{x}-\mathbf{y} \|^2_2}{\varepsilon_1} \right)} \mu_{1_j}(\mathbf{y}) d\mathbf{y}}{\int_{\mathbf{y}^{\prime}} {\exp{\left(\frac{\psi_{ij}(\mathbf{y}^{\prime}) - \| \mathbf{x}-\mathbf{y}^{\prime} \|^2_2}{\varepsilon_1} \right)} \mu_{1_j}(\mathbf{y}^{\prime}) d\mathbf{y}^{\prime}}}} \nonumber \\
    &= 2\mathbf{x}^T - 2 \displaystyle \int_{\mathbf{y}} \mathbf{y}^T p_{ij}(\mathbf{y} | \mathbf{x}) d\mathbf{y} \nonumber \\
    &= 2(\mathbf{x} - T_{ij}(\mathbf{x}))^T \label{eq:JacPhi}
\end{align}
Substituting Eq.~\ref{eq:JacPhi} into Eq.~\ref{eq:JacLogCondP} yields
\begin{align}
    \nabla_{\mathbf{x}} \log{p_{ij}(\mathbf{y} | \mathbf{x})}
    &= -2 \varepsilon_{1}^{-1} (\mathbf{x} - \mathbf{y})^{T} + 2\varepsilon_{1}^{-1} (\mathbf{x} - T_{ij}(\mathbf{x}))^T \nonumber \\
    &= 2\varepsilon_{1}^{-1} (\mathbf{y} - T_{ij}(\mathbf{x}))^T
\end{align}
Applying the logarithmic derivative formula 
\begin{align}
\label{eq:nabla_p_ij}
    \nabla_{\mathbf{x}} p_{ij}(\mathbf{y} | \mathbf{x}) &= \nabla_{x} \log{p_{ij}(\mathbf{y} | \mathbf{x})} p_{ij}(\mathbf{y} | \mathbf{x}) \nonumber \\
    &= 2\varepsilon_{1}^{-1} (\mathbf{y} - T_{ij}(\mathbf{x}))^{T} p_{ij}(\mathbf{y} | \mathbf{x}).
\end{align}
Hence, the Jacobian of $T_{ij}$ can be defined as 
\begin{align}
    \nabla_{\mathbf{x}} T_{ij}(\mathbf{x}) &= \displaystyle \int \mathbf{y} \nabla_{x} p_{ij}(\mathbf{y} | \mathbf{x}) d\mathbf{y} \nonumber \\
    &= 2\varepsilon_{1}^{-1} \displaystyle \int \mathbf{y} (\mathbf{y} - T_{ij}(\mathbf{x}))^{T} p_{ij}(\mathbf{y} | \mathbf{x}) d\mathbf{y} \nonumber\\
    &=  2\varepsilon_{1}^{-1} \left(\displaystyle \int \mathbf{y} \mathbf{y}^T p_{ij}(\mathbf{y} | \mathbf{x}) d\mathbf{y} - \int \mathbf{y} p_{ij}(\mathbf{y} | \mathbf{x}) T^T_{ij}(\mathbf{x}) d\mathbf{y}\right)  \nonumber \\
    &= 2\varepsilon_{1}^{-1} \left(\displaystyle \int \mathbf{y} \mathbf{y}^T p_{ij}(\mathbf{y} | \mathbf{x}) d\mathbf{y} - T_{ij}(\mathbf{x}) T^T_{ij}(\mathbf{x}) \right) \nonumber \\
    &= 2\varepsilon_{1}^{-1} \left(\mathbb{E}_{p_{ij}(\mathbf{y} | \mathbf{x})} \left[\mathbf{y}\mathbf{y}^T \right] - \mathbb{E}_{p_{ij}(\mathbf{y} | \mathbf{x})} \left[\mathbf{y}\right] \mathbb{E}^T_{p_{ij}(\mathbf{y} | \mathbf{x})} \left[\mathbf{y}\right] \right) = 2 \varepsilon_{1}^{-1} \text{Cov}_{{p_{ij}(\mathbf{y}|\mathbf{x})}}({Y} | {X}=\mathbf{x}) = 2\varepsilon_{1}^{-1} \text{Cov}^{ij}_{\mathbf{y}|\mathbf{x}}
    \label{eq:TijDeriv}
\end{align}
Accordingly,
\begin{align}
    T_{ij}(\mathbf{x}) &= \mathbf{E}_{p_{ij}(\mathbf{y}|\mathbf{x})}[\mathbf{y}] = \mathbf{m}^{ij}_{\mathbf{y}|\mathbf{x}} 
    \label{eq:TGauss}\\
    \nabla T_{ij}(\mathbf{x}) &= 2 \varepsilon_{1}^{-1} \text{Cov}^{ij}_{\mathbf{y}|\mathbf{x}} \label{eq:dTGauss}
\end{align}
Now, we focus on the derivative of the mixture weights $\tilde{\omega}_{ij}$ in Eq.~\ref{eq:OMT_cpl}. Recall that
\begin{align*}
    \tilde{\omega}_{ij}(\mathbf{x}) &=  \left(\displaystyle{\frac{\omega_{ij} \int_{\mathbf{y}} p_{ij}(\mathbf{x}, \mathbf{y}) d\mathbf{y}}{\displaystyle \sum_{r,k} \omega_{rk} \int_{\mathbf{y}^{\prime}} p_{rk}(\mathbf{x}, \mathbf{y}^{\prime}) d\mathbf{y}^{\prime}}} \right)
\end{align*}
    Following the derivation in earlier section, Eq.~\ref{eq:omega_opt}, the optimality conditions yield
\begin{gather}
    \omega_{ij} = \displaystyle{\exp{\left(\frac{\lambda_i+ \tau_j - \mathcal{L}_{p_{ij}}}{\varepsilon_2}\right)}} \alpha_{0_i} \alpha_{1_j} \label{eq:omega_opt_2} 
\end{gather}
where
\begin{gather*}
    \mathcal{L}_{p_{ij}} = \mathbb{E}_{p_{ij}}\left[\varphi_{ij}(\mathbf{x}) \right] + \mathbb{E}_{p_{ij}}\left[\psi_{ij}(\mathbf{y}) \right] = \displaystyle \int \int \left(\varphi_{ij}(\mathbf{x}) + \psi_{ij}(\mathbf{y}) \right) p_{ij}(\mathbf{x}, \mathbf{y}) d\mathbf{x} d\mathbf{y}
\end{gather*}
and the corresponding dual potentials are given by
\begin{gather}
    \lambda_i = -\varepsilon_2 \displaystyle{\log{ \sum_{j} \exp{\left(\frac{\tau_j - \mathcal{L}_{p_{ij}}}{\varepsilon_2} \right)} \alpha_{1_j}}}, \hspace{.25in}  \tau_j = -\varepsilon_2 \displaystyle{\log{ \sum_{i} \exp{\left(\frac{\lambda_i - \mathcal{L}_{p_{ij}}}{\varepsilon_2} \right)} \alpha_{0_i}}}.
    \label{eq:discrete_potentials} 
\end{gather}
Expanding the component distributions $p_{ij}$ in exponential form and substituting the optimality conditions from Eq.~\ref{eq:omega_opt_2}-\ref{eq:discrete_potentials}, we obtain
\begin{align*}
    \tilde{\omega}_{ij}(\mathbf{x}) &=  \displaystyle{\frac{\omega_{ij} \mu_{0_i}(\mathbf{x}) \displaystyle{\exp{\left(\frac{\varphi_{ij}(\mathbf{x})}{\varepsilon_1}\right)}}  \int_\mathbf{y} \displaystyle{\exp{\left(\frac{\psi_{ij}(\mathbf{y}) - \| \mathbf{x} - \mathbf{y} \|^2_2}{\varepsilon_1}\right)}} \mu_{1_j}(\mathbf{y}) d\mathbf{y}}{\displaystyle \sum_{r,k} \displaystyle \displaystyle{\exp{\left(\frac{\lambda_r + \tau_k - \mathcal{L}_{p_{rk}}}{\varepsilon_2}\right)}} \alpha_{0_r}  \alpha_{1_k} \mu_{0_r}(\mathbf{x}) \displaystyle{\exp{\left(\frac{\varphi_{rk}(\mathbf{x})}{\varepsilon_1}\right)}} \int_{\mathbf{y}^{\prime}} \displaystyle{\exp{\left(\frac{\psi_{rk}(\mathbf{y}^{\prime}) - \| \mathbf{x} - \mathbf{y}^{\prime} \|^2_2}{\varepsilon_1}\right)}} \mu_{1_j}(\mathbf{y}^{\prime}) d\mathbf{y}^{\prime}}}.
\end{align*}
Having exponential terms cancel out, yielding the simplified expression.
\begin{align*}
    \tilde{\omega}_{ij}(\mathbf{x}) &= \displaystyle{\frac{\omega_{ij}  \mu_{0_i}(\mathbf{x})\displaystyle{\exp{\left(\frac{\varphi_{ij}(\mathbf{x})}{\varepsilon_1}\right)}} \displaystyle{\exp{\left(\frac{-\varphi_{ij}(\mathbf{x})}{\varepsilon_1}\right)}} }{\displaystyle \sum_{r} \displaystyle{\exp{\left(\frac{\lambda_r}{\varepsilon_2}\right)}} \alpha_{0_r} \mu_{0_r}(\mathbf{x}) \sum_{k}  \displaystyle{\exp{\left(\frac{\tau_k - \mathcal{L}_{p_{rk}}(\mathbf{x}, \mathbf{y})}{\varepsilon_2}\right)}} \displaystyle{\exp{\left(\frac{\varphi_{rk}(\mathbf{x})}{\varepsilon_1}\right)}} \displaystyle{\exp{\left(\frac{-\varphi_{rk}(\mathbf{x})}{\varepsilon_1}\right)}} \alpha_{1_k}}} \\[.1in]
    &= \displaystyle{\frac{ \omega_{ij} \mu_{0_i}(\mathbf{x})}{\displaystyle \sum_{r} \displaystyle{\exp{\left(\frac{\lambda_r}{\varepsilon_2}\right)}} \alpha_{0_r} \mu_{0_r}(\mathbf{x}) \underbrace{\sum_{k}  \displaystyle{\exp{\left(\frac{\tau_k - \mathcal{L}_{p_{rk}}(\mathbf{x}, \mathbf{y})}{\varepsilon_2}\right)}} \alpha_{1_k} }_{\displaystyle{\exp{\left(\frac{-\lambda_r}{\varepsilon_2}\right)}}}}} \\
    &= \displaystyle{\frac{\omega_{ij} \mu_{0_i}(\mathbf{x})}{\displaystyle \sum_{r} \alpha_{0_r} \mu_{0_r}(\mathbf{x})}} = \displaystyle{\frac{\omega_{ij} \mu_{0_i}(\mathbf{x})}{\nu_{0}(\mathbf{x})}} \ .
\end{align*}
Differentiating $\tilde{\omega}_{ij}(\mathbf{x})$ with respect to $\mathbf{x}$ gives
\begin{align}
\label{eq:OmegaDeriv}
    \nabla  \tilde{\omega}_{ij}(\mathbf{x}) &= \displaystyle{\frac{\omega_{ij}}{\nu_0(\mathbf{x})}} \nabla \mu_{0_{i}}(\mathbf{x}) - \displaystyle{\frac{\tilde{\omega}_{ij}(\mathbf{x})}{\nu_0(\mathbf{x})}} \nabla \nu_{0}(\mathbf{x}) \nonumber \\
    &= \displaystyle{\frac{\omega_{ij}}{\nu_0(\mathbf{x})}} \nabla \mu_{0_{i}}(\mathbf{x}) - \displaystyle{\frac{\tilde{\omega}_{ij}(\mathbf{x})}{\nu_0(\mathbf{x})}} \sum_{r,k} \omega_{rk} \nabla \mu_{0_{r}}(\mathbf{x}) \nonumber \\
    &= \tilde{\omega}_{ij}(\mathbf{x}) \left( \underbrace{\nabla \log \mu_{0_{i}}(\mathbf{x})}_{g_i(\mathbf{x})} - \underbrace{\sum_{r,k} \tilde{\omega}_{rk}(\mathbf{x}) \nabla \log \mu_{0_{r}}(\mathbf{x})}_{\overline{g(\mathbf{x})}} \right) 
\end{align}
By substituting the expressions in Eqs.~\ref{eq:TijDeriv}, \ref{eq:TGauss}, and \ref{eq:OmegaDeriv} into Eq.~\ref{eq:DerivOMT}, the gradient can be compactly expressed as
\begin{align*}
\nabla T_{\textsc{omt}}(\mathbf{x}) &= \displaystyle \sum_{i,j} 2 \varepsilon_{1}^{-1} \tilde{\omega}_{ij}(\mathbf{x})   \ \text{Cov}^{ij}_{\mathbf{y}|\mathbf{x}} + \tilde{\omega}_{ij}(\mathbf{x}) \mathbf{m}^{ij}_{\mathbf{y|\mathbf{x}}}  \left(g_i(\mathbf{x}) - \overline{g(\mathbf{x})}\right)^T
\end{align*}
Taking norms on both sides and applying the triangle inequality, we obtain
\begin{align}
\label{eq:Jac_norm_b1}
    \| \nabla T_{\textsc{omt}}(\mathbf{x}) \|_{\text{op}} &= \|2 \varepsilon_{1}^{-1}  \displaystyle \sum_{i,j} \tilde{\omega}_{ij}(\mathbf{x})   \ \text{Cov}^{ij}_{\mathbf{y}|\mathbf{x}} + \displaystyle \sum_{i,j} \tilde{\omega}_{ij}(\mathbf{x}) \mathbf{m}^{ij}_{\mathbf{y|\mathbf{x}}}  \left(g_i(\mathbf{x}) - \overline{g(\mathbf{x})}\right)^T   \|_{\text{op}}  \nonumber \\
    &  \leq \displaystyle 2 \varepsilon_{1}^{-1}  \sum_{i,j} \tilde{\omega}_{ij}(\mathbf{x}) \| {\text{Cov}}^{ij}_{\mathbf{y}|\mathbf{x}} \|_{\text{op}} +   \sum_{i,j} \tilde{\omega}_{ij}(\mathbf{x}) \| \mathbf{m}^{ij}_{\mathbf{y|\mathbf{x}}} \left(g_i(\mathbf{x}) - \overline{g(\mathbf{x})}\right)^T \|_{\text{op}} \nonumber \\
    & = \displaystyle 2 \varepsilon_{1}^{-1}  \sum_{i,j} \tilde{\omega}_{ij}(\mathbf{x}) \| {\text{Cov}}^{ij}_{\mathbf{y}|\mathbf{x}} \|_{\text{op}} +   \underbrace{\sum_{i,j} \tilde{\omega}_{ij}(\mathbf{x}) \| \mathbf{m}^{ij}_{\mathbf{y|\mathbf{x}}}  \| \  \|\left(g_i(\mathbf{x}) - \overline{g(\mathbf{x})}\right)^T \| }_{s(\mathbf{x})}
\end{align}
Regularity conditions ensure that all weighted norms exist and are finite. Consequently,
\begin{align}
\label{eq:l_nu}
  L_{\nu_0} &:= \displaystyle \sup_{\mathbf{x}} \| \nabla T_{\textsc{omt}}(\mathbf{x}) \| = 2 \varepsilon_1^{-1} u_{\text{max}} + s_{\text{max}} < \infty  
\end{align}
where $s_{\max}  = \displaystyle \sup_{\mathbf{x}} \{ s(\mathbf{x})\}$, and $u_{\max}$ is defined in Eq.~\ref{eq:U_max}.
\end{proof}
\subsection{Tilt Stability of $T_{\textsc{omt}}$}
\label{sec:app_tiltstability}
%
%
\begin{lemma}[Lemma 40 \cite{chen2022localization}]
\label{lem:tilt-stable}
Let $\nu$ be a probability measure on $\mathbb{R}^d$ with finite second moment. Suppose that for every $\mathbf{v} \in \mathbb{R}^d$ one has
$$ \|\textup{Cov}(\mathcal{T}_{\mathbf{v}} \nu) \|_{\textup{op}}\le \alpha $$
Then $\nu$ is  tilt-stable with constant $\alpha \in \mathbb{R}^+$ with respect to the function $f(x, y) = \frac{1}{2} \|\mathbf{x} - \mathbf{y}\|_2^2$.
\end{lemma} 
\vspace{.1in}
\begin{coroll}
\label{coroll:comp_tilt_stability}
Let $\mu_{i}(\mathbf{x})$ and $\mu_{j}(\mathbf{y})$ be two measures with finite second moments.  
Then the entropic transport map from $\mu_i$ to $\mu_j$ is tilt-stable with stability constant
\begin{gather}
\label{eq:comp_tilt_cnt}
c_{ij} = \sup_{\mathbf{x}} \| \textsc{Cov}^{ij}_{\mathbf{y}|\mathbf{x}}\|_{\textsc{op}} \ .
\end{gather}
\end{coroll}
\begin{proof}
Our derivation follows the same proof for Lemma 2.2 and Corollary 2.3 in \cite{divol2025tight}.
\begin{align*}
    p_{ij}(\mathbf{y} | \mathbf{x}) &= \displaystyle \exp\left({\frac{\varphi_{ij}(\mathbf{x}) +\psi_{ij}(\mathbf{y}) - \| \mathbf{x}-\mathbf{y} \|^2_2}{\varepsilon_1}}\right)\mu_{1_j}(\mathbf{y}) \\
    &= \displaystyle \exp{\left(\frac{\langle \mathbf{x}, \mathbf{y} \rangle - \tilde{\varphi}_{ij}(\mathbf{x}) - \tilde{\psi}_{ij}(\mathbf{y})}{1/2\varepsilon_1}\right)}\mu_{1_j}(\mathbf{y}) 
\end{align*}
where $\tilde{\varphi}_{ij}(\mathbf{x}) := 1/2(\|x\|^2 - \varphi_{ij}(\mathbf{x})) $ and $\tilde{\psi}_{ij}(\mathbf{y}) := 1/2(\|y\|^2 - \psi_{ij}(\mathbf{y}))$, which are also known as Brenier potentials. 

\begin{align*}
    \mathbb{E}_{p_{ij}(\mathbf{y}|\mathbf{x})}[e^{\langle \mathbf{v}, \mathbf{y} \rangle}] &= \displaystyle \exp{\frac{- 2\tilde{\varphi}_{ij}(\mathbf{x})}{\varepsilon_1}} \int \exp{\left(\frac{\langle \mathbf{x}, \mathbf{y} \rangle + 1/2\varepsilon_1\langle \mathbf{v}, \mathbf{y} \rangle - \tilde{\psi}_{ij}(\mathbf{y})}{1/2\varepsilon_1}\right)} \mu_{1_j}(\mathbf{y}) d\mathbf{y} \\
    &= \displaystyle \exp{\frac{- 2\tilde{\varphi}_{ij}(\mathbf{x})}{\varepsilon_1}} \underbrace{\int \exp{\left(\frac{\langle \mathbf{x} + 1/2\varepsilon_1\mathbf{v}, \mathbf{y} \rangle  - \tilde{\psi}_{ij}(\mathbf{y})}{1/2\varepsilon_1}\right)} \mu_{1_j}(\mathbf{y}) d\mathbf{y}}_{\exp{\left( \frac{\tilde{\varphi}_{ij}(\mathbf{x} + 1/2\varepsilon_1 \mathbf{v})}{1/2 \varepsilon_1} \right)}} \\
    &= \displaystyle \exp{\left( \frac{\tilde{\varphi}_{ij}(\mathbf{x} + 1/2\varepsilon_1 \mathbf{v}) - \tilde{\varphi}_{ij}(\mathbf{x})}{1/2 \varepsilon_1} \right)} = h_{{\varepsilon_1}_{ij}}
\end{align*}
Based on Definition~\ref{def:tilt}, the tilted coupling measure is
\begin{align}
\label{eq:tilt_v_measure_ij}
\mathcal{T}_{\mathbf{v}}p_{ij}(\mathbf{y}|\mathbf{x}) &= \displaystyle \frac{1}{h_{{\varepsilon_1}_{ij}}} \exp{\left( \frac{\langle \mathbf{x} + 1/2\varepsilon_1\mathbf{v}, \mathbf{y} \rangle - \tilde{\varphi}_{ij}(\mathbf{x}) - \tilde{\psi}_{ij}(\mathbf{y})}{1/2 \varepsilon_1} \right)} \mu_{1_j}(\mathbf{y}) \nonumber \\
&= \displaystyle \exp{\left( \frac{\langle \mathbf{x} + 1/2\varepsilon_1\mathbf{v}, \mathbf{y} \rangle - \tilde{\varphi}_{ij}(\mathbf{x}+ 1/2\varepsilon_1\mathbf{v}) - \tilde{\psi}_{ij}(\mathbf{y})}{1/2 \varepsilon_1} \right)} \mu_{1_j}(\mathbf{y}) \nonumber \\
&= p_{ij}(\mathbf{y} | \mathbf{x}+ 1/2\varepsilon_1\mathbf{v})
\end{align}
Thus,
\begin{gather}
    \mathcal{T}_{\mathbf{v}}p_{ij}(\mathbf{y}|\mathbf{x}) =p_{ij}(\mathbf{y}|\mathbf{x} + \frac{1}{2}\varepsilon_1 \mathbf{v})
\end{gather}
\begin{align}
\label{eq:cov_ij_tilt}
    \text{Cov}_{\mathcal{T}_{\boldsymbol{v}} p_{ij}}(Y|X=\mathbf{x}) &= \text{Cov}_{p_{ij}}(Y|X=\mathbf{x} + \frac{1}{2}\varepsilon_1 \mathbf{v}) \nonumber \\
    &= \text{Cov}^{ij}_{\mathbf{y}|\mathbf{x}} 
\end{align}
Finally,
\begin{gather*}
    \| \text{Cov}_{\mathcal{T}_{\boldsymbol{v}} p_{ij}}\|_{\text{op}} \le \displaystyle \sup_{\mathbf{x}} \| \text{Cov}^{ij}_{\mathbf{y}|\mathbf{x}}\|_{\text{op}} \ = c_{ij} \ .
\end{gather*}
\end{proof}
\begin{prop}
\label{prop:app_OMT_tilt_cnt}
 Suppose $\nu_0(\mathbf{x}) \in M_{K_0}(\mathbb{R}^d), \nu_1(\mathbf{y}) \in M_{K_1}(\mathbb{R}^d)$ are mixture models admitting regularity conditions in Definition~\ref{def:regularity_conditions}. Then, the OMT map between these two probability measures is tilt-stable. 
\end{prop}
\begin{proof}
According to the definition of Eq.~\ref{eq:OMT_cpl} and Eq.~\ref{eq:cond_pij}, we have
\begin{align*}
    \pi_{\textsc{omt}}(\mathbf{y}|\mathbf{x}) &= \displaystyle \sum_{i,j} \tilde{\omega}_{ij}(\mathbf{x}) p_{ij}(\mathbf{y} | \mathbf{x})\\
    &= \displaystyle \sum_{i,j} \tilde{\omega}_{ij}(\mathbf{x}) \exp{\left(\frac{\langle \mathbf{x}, \mathbf{y} \rangle - \tilde{\varphi}_{ij}(\mathbf{x}) - \tilde{\psi}_{ij}(\mathbf{y})}{1/2\varepsilon_1}\right)}\mu_{1_j}(\mathbf{y})
\end{align*}
Similar to derivations for Corollary~\ref{coroll:comp_tilt_stability}, we have
\begin{align*}
    \mathbb{E}_{\pi_{\textsc{omt}}(\mathbf{y}|\mathbf{x})}[e^{\langle \mathbf{v}, \mathbf{y} \rangle}] &= \displaystyle \sum_{i,j} \tilde{\omega}_{ij}(\mathbf{x}) \exp{\frac{- 2\tilde{\varphi}_{ij}(\mathbf{x})}{\varepsilon_1}} \int \exp{\left(\frac{\langle \mathbf{x}, \mathbf{y} \rangle + 1/2\varepsilon_1\langle \mathbf{v}, \mathbf{y} \rangle - \tilde{\psi}_{ij}(\mathbf{y})}{1/2\varepsilon_1}\right)} \mu_{1_j}(\mathbf{y}) d\mathbf{y} \\
    &= \displaystyle \sum_{i,j} \tilde{\omega}_{ij}(\mathbf{x}) \exp{\frac{- 2\tilde{\varphi}_{ij}(\mathbf{x})}{\varepsilon_1}} \underbrace{\int \exp{\left(\frac{\langle \mathbf{x} + 1/2\varepsilon_1\mathbf{v}, \mathbf{y} \rangle  - \tilde{\psi}_{ij}(\mathbf{y})}{1/2\varepsilon_1}\right)} \mu_{1_j}(\mathbf{y}) d\mathbf{y}}_{\exp{\left( \frac{\tilde{\varphi}_{ij}(\mathbf{x} + 1/2\varepsilon_1 \mathbf{v})}{1/2 \varepsilon_1} \right)}} \\
    &= \displaystyle \sum_{i,j} \tilde{\omega}_{ij}(\mathbf{x}) \exp{\left( \frac{\tilde{\varphi}_{ij}(\mathbf{x} + 1/2\varepsilon_1 \mathbf{v}) - \tilde{\varphi}_{ij}(\mathbf{x})}{1/2 \varepsilon_1} \right)} \\
    & = \displaystyle \sum_{i,j} \tilde{\omega}_{ij}(\mathbf{x}) h_{{\varepsilon_1}_{ij}}(\mathbf{x}, \mathbf{v}) = \overline{h}_{\varepsilon_1}(\mathbf{x}, \mathbf{v}), \quad \forall \mathbf{v} \in \mathbb{R}^d.
\end{align*}
Based on Definition~\ref{def:tilt}, the tilted coupling measure 
\begin{align}
\label{eq:tilt_v_measure}
    \mathcal{T}_{\mathbf{v}}\pi_{\textsc{omt}}(\mathbf{y}|\mathbf{x}) &= \displaystyle \frac{1}{\overline{h}_{\varepsilon_1}(\mathbf{x}, \mathbf{v})} \sum_{i,j} \tilde{\omega}_{ij}(\mathbf{x}) \exp{\left( \frac{\langle \mathbf{x} + 1/2\varepsilon_1\mathbf{v}, \mathbf{y} \rangle - \tilde{\varphi}_{ij}(\mathbf{x}) - \tilde{\psi}_{ij}(\mathbf{y})}{1/2 \varepsilon_1} \right)} \mu_{1_j}(\mathbf{y}) \nonumber \\
&= \displaystyle \frac{1}{\overline{h}_{\varepsilon_1}(\mathbf{x}, \mathbf{v})} \sum_{i,j} \tilde{\omega}_{ij}(\mathbf{x}) \frac{{h}_{{\varepsilon_1}_{ij}}(\mathbf{x}, \mathbf{v})}{{h}_{{\varepsilon_1}_{ij}}(\mathbf{x}, \mathbf{v})} \exp{\left( \frac{\langle \mathbf{x} + 1/2\varepsilon_1\mathbf{v}, \mathbf{y} \rangle - \tilde{\varphi}_{ij}(\mathbf{x}) - \tilde{\psi}_{ij}(\mathbf{y})}{1/2 \varepsilon_1} \right)} \mu_{1_j}(\mathbf{y}) \nonumber \\
&= \displaystyle \sum_{i,j} \tilde{\omega}_{ij}(\mathbf{x}) \frac{{h}_{{\varepsilon_1}_{ij}}(\mathbf{x}, \mathbf{v})}{\overline{h}_{\varepsilon_1}(\mathbf{x}, \mathbf{v})} \exp{\left( \frac{\langle \mathbf{x} + 1/2\varepsilon_1\mathbf{v}, \mathbf{y} \rangle - \tilde{\varphi}_{ij}(\mathbf{x}+ 1/2\varepsilon_1\mathbf{v}) - \tilde{\psi}_{ij}(\mathbf{y})}{1/2 \varepsilon_1} \right)} \mu_{1_j}(\mathbf{y}) \nonumber \\
&= \displaystyle \sum_{i,j} \tilde{\omega}_{ij}(\mathbf{x}) \frac{{h}_{{\varepsilon_1}_{ij}}(\mathbf{x}, \mathbf{v})}{\overline{h}_{\varepsilon_1}(\mathbf{x}, \mathbf{v})} p_{ij}(\mathbf{y} | \mathbf{x}+ 1/2\varepsilon_1\mathbf{v})
\end{align}
Let $\boldsymbol{\delta} = \frac{1}{2}\varepsilon_1 \mathbf{v}$.
\begin{align}
\label{eq:tilt_omt_delta}
\mathcal{T}_{\mathbf{\boldsymbol{\delta}}} \pi_{\textsc{omt}}(\mathbf{y}|\mathbf{x}) &= \displaystyle \sum_{i,j} \underbrace{\tilde{\omega}_{ij}(\mathbf{x}) \frac{\exp{\left( \frac{\tilde{\varphi}_{ij}(\mathbf{x} + \boldsymbol{\delta}) - \tilde{\varphi}_{ij}(\mathbf{x})}{1/2 \varepsilon_1} \right)}}{\sum_{i,j} \tilde{\omega}_{ij}(\mathbf{x})\exp{\left( \frac{\tilde{\varphi}_{ij}(\mathbf{x} + \boldsymbol{\delta}) - \tilde{\varphi}_{ij}(\mathbf{x})}{1/2 \varepsilon_1} \right)}}}_{\tilde{\omega}_{ij}(\mathbf{x}, \boldsymbol{\delta})} \ p_{ij}(\mathbf{y} | \mathbf{x}+ \boldsymbol{\delta}) \nonumber \\
    &= \displaystyle \sum_{i,j} \tilde{\omega}_{ij}(\mathbf{x}, \boldsymbol{\delta}) \ p_{ij}(\mathbf{y} | \mathbf{x}+ \boldsymbol{\delta})
\end{align}
Similar to the covariance derivation in \ref{sec:app_cov}, the conditional covariance of $\mathcal{T}_{\boldsymbol{\delta}}\pi_{\textsc{omt}}$ is expressed as:
\begin{align}
\label{eq:total_cov_tilt}
\text{Cov}_{\mathcal{T}_{\boldsymbol{\delta}} \pi_{\textsc{omt}}}(Y|X=\mathbf{x}+ \boldsymbol{\delta}) &= \displaystyle \sum_{i,j} \tilde{\omega}_{ij}(\mathbf{x}, \boldsymbol{\delta}) \ \text{Cov}^{ij}_{\mathbf{y}|\mathbf{x}+ \boldsymbol{\delta}}   + \displaystyle \sum_{i,j} \tilde{\omega}_{ij}(\mathbf{x}, \boldsymbol{\delta}) \left(\mathbf{m}^{ij}_{\mathbf{y}|\mathbf{x}+ \boldsymbol{\delta}} -\bar{\mathbf{m}}_{\mathbf{y}|\mathbf{x}+ \boldsymbol{\delta}} \right)\left(\mathbf{m}^{{ij}}_{\mathbf{y}|\mathbf{x}+ \boldsymbol{\delta}} - \bar{\mathbf{m}}_{\mathbf{y}|\mathbf{x}+ \boldsymbol{\delta}} \right)^T
\end{align}
In the entropic transport framework, satisfying the marginal densities requires 
\begin{align}
\label{eq:marginal_const_source}
  \nu_{0}(\mathbf{x}) &= \int_\mathcal{Y} \displaystyle \sum_{i,j} \omega_{ij} p_{ij}(\mathbf{x}, \mathbf{y}) d\mathbf{y} \nonumber\\
  &= \int_{\mathcal{Y}}\sum_{i,j} {\omega}_{ij}\exp{\left(\frac{\langle \mathbf{x}, \mathbf{y} \rangle - \tilde{\varphi}_{ij}(\mathbf{x}) - \tilde{\psi}_{ij}(\mathbf{y})}{1/2\varepsilon_1}\right)}\mu_{0_i}(\mathbf{x}) \mu_{1_j}(\mathbf{y})  d\mathbf{y} \nonumber \\
  &= \sum_{i,j} {\omega}_{ij} \mu_{0_i}(\mathbf{x}) \exp{\left(\frac{-\tilde{\varphi}_{ij}(\mathbf{x})}{1/2\varepsilon_1}\right)} \underbrace{\int_{\mathcal{Y}} \exp{\left(\frac{\langle \mathbf{x}, \mathbf{y} \rangle - \tilde{\psi}_{ij}(\mathbf{y})}{1/2\varepsilon_1}\right)} \mu_{1_j}(\mathbf{y})  d\mathbf{y}}_{f_{ij}(\mathbf{x}) \hspace{.1in} \text{(kernel integral)}} \nonumber \\
  1 & = \sum_{i,j} {\omega}_{ij}  \frac{\mu_{0_i}(\mathbf{x})}{\nu_0{\mathbf{x}}}\exp{\left(\frac{-\tilde{\varphi}_{ij}(\mathbf{x})}{1/2\varepsilon_1}\right)} {f_{ij}(\mathbf{x})} \nonumber \\
  \sum_{i,j} {\omega}_{ij} &= \sum_{i,j} {\omega}_{ij} \frac{\mu_{0_i}(\mathbf{x})}{\nu_0(\mathbf{x})}\exp{\left(\frac{-\tilde{\varphi}_{ij}(\mathbf{x})}{1/2\varepsilon_1}\right)} {f_{ij}(\mathbf{x})} 
\end{align}
which implies 
\begin{align}
\label{eq:phi_mu_ij}
    \exp(2\tilde{\varphi}_{ij}(\mathbf{x})/\varepsilon_1) = \displaystyle f_{ij}(\mathbf{x}) \frac{\mu_{0_i}(\mathbf{x})}{\nu_0(\mathbf{x})}
\end{align}
Accordingly, in Eq.~\ref{eq:tilt_omt_delta}, the growth of the Brenier potential, i.e., $\Delta \tilde{\varphi}_{ij} := \tilde{\varphi}_{ij}(\mathbf{x} + \boldsymbol{\delta}) - \tilde{\varphi}_{ij}(\mathbf{x})$ can be stated as:
\begin{align*}
    \exp\left(\frac{\Delta \tilde{\varphi}_{ij}}{1/2\varepsilon_1}\right) & = \left(\frac{\mu_{0_i}(\mathbf{x}+\boldsymbol{\delta})}{\nu_{0}(\mathbf{x} + \boldsymbol{\delta})}\right) \left(\frac{\nu_{0}(\mathbf{x})}{\mu_{0_i}(\mathbf{x})}\right) \left( \frac{f_{ij}(\mathbf{x} + \boldsymbol{\delta})}{f_{ij}(\mathbf{x})}\right) \\
    &= \left(\frac{\omega_{ij} \mu_{0_i}(\mathbf{x}+\boldsymbol{\delta})}{\nu_{0}(\mathbf{x} + \boldsymbol{\delta})}\right) \left(\frac{\nu_{0}(\mathbf{x})}{\omega_{ij} \mu_{0_i}(\mathbf{x})}\right) \left( \frac{f_{ij}(\mathbf{x} + \boldsymbol{\delta})}{f_{ij}(\mathbf{x})}\right) \\
    &= \left(\frac{\tilde{\omega}_{ij}(\mathbf{x} + \boldsymbol{\delta})}{\tilde{\omega}_{ij}(\mathbf{x})} \right) \left( \frac{f_{ij}(\mathbf{x} + \boldsymbol{\delta})}{f_{ij}(\mathbf{x})}\right)
\end{align*}
Bounding the kernel ratio by its maximum over all $(i,j)$, denoted $s(\mathbf{x}, \boldsymbol{\delta})$, yields
\begin{align*}
    \exp\left(\frac{\Delta \tilde{\varphi}_{ij}}{1/2\varepsilon_1}\right) \leq s(\mathbf{x}, \boldsymbol{\delta}) \frac{\tilde{\omega}_{ij}(\mathbf{x} + \boldsymbol{\delta})}{\tilde{\omega}_{ij}(\mathbf{x})}
\end{align*}
Since the tilted weights satisfy
\begin{align*}
\tilde{\omega}_{ij}(\mathbf{x}, \boldsymbol{\delta}) 
&= \displaystyle \frac{ \tilde{\omega}_{ij}(\mathbf{x}) 
\exp \left( \displaystyle \frac{\Delta \tilde{\varphi}_{ij}(\mathbf{x})}{\tfrac{1}{2} \varepsilon_1} \right)}
{ \displaystyle \sum_{r,k} \tilde{\omega}_{rk}(\mathbf{x})  
\exp \left( \frac{\Delta \tilde{\varphi}_{rk}(\mathbf{x})}{\tfrac{1}{2} \varepsilon_1} \right)} \ ,
\end{align*}
we obtain the bound
\begin{align*}
\tilde{\omega}_{ij}(\mathbf{x}, \boldsymbol{\delta})  \leq \displaystyle \frac{ s(\mathbf{x}, \boldsymbol{\delta}) \tilde{\omega}_{ij}(\mathbf{x} + \boldsymbol{\delta}) }
{ \displaystyle \sum_{r,k} s(\mathbf{x}, \boldsymbol{\delta}) \tilde{\omega}_{rk}(\mathbf{x} + \boldsymbol{\delta})} = \tilde{\omega}_{ij}(\mathbf{x} + \boldsymbol{\delta}) \ .
\end{align*}
Let ${\mathbf{x}'} := \mathbf{x} + \boldsymbol{\delta}$. Then the conditional covariance under the tilted measure satisfies
\begin{align}
\label{eq:total_cov_tilt_2}
\text{Cov}_{\mathcal{T}_{\boldsymbol{\delta}} \pi_{\textsc{omt}}}(Y|X={\mathbf{x}'}) &= \displaystyle \sum_{i,j} \tilde{\omega}_{ij}(\mathbf{x}, \boldsymbol{\delta}) \ \text{Cov}^{ij}_{\mathbf{y}|{\mathbf{x}'}}   + \displaystyle \sum_{i,j} \tilde{\omega}_{ij}(\mathbf{x}, \boldsymbol{\delta}) \left(\mathbf{m}^{ij}_{\mathbf{y}|\tilde{\mathbf{x}}} -\bar{\mathbf{m}}_{\mathbf{y}|{\mathbf{x}'}} \right)\left(\mathbf{m}^{{ij}}_{\mathbf{y}|{\mathbf{x}'}} - \bar{\mathbf{m}}_{\mathbf{y}|{\mathbf{x}'}} \right)^T \nonumber \\
& \leq \displaystyle \sum_{i,j} \tilde{\omega}_{ij}({\mathbf{x}'}) \ \text{Cov}^{ij}_{\mathbf{y}|{\mathbf{x}'}}   + \displaystyle \sum_{i,j} \tilde{\omega}_{ij}({\mathbf{x}'}) \left(\mathbf{m}^{ij}_{\mathbf{y}|{\mathbf{x}'}} -\bar{\mathbf{m}}_{\mathbf{y}|{\mathbf{x}'}} \right)\left(\mathbf{m}^{{ij}}_{\mathbf{y}|{\mathbf{x}'}} - \bar{\mathbf{m}}_{\mathbf{y}|{\mathbf{x}'}} \right)^T \nonumber \\
& = \text{Cov}_{\pi_{\textsc{omt}}}(Y|X={\mathbf{x}'})
\end{align}
Importantly, these stability results hold for all $\mathbf{x} \in \mathcal{X} \subset \mathbb{R}^d$. For any tilted value ${\mathbf{x}'} = \mathbf{x} + \boldsymbol{\delta}$ that falls outside this subset, the density $\nu_0(\mathbf{x}')$ is not well-defined, and thus the stability of the transport map is not defined.

Finally, 
\begin{gather}
\label{eq:total_tilt_cnt}
\displaystyle \|\text{Cov}_{\mathcal{T}_{\boldsymbol{\delta}}\pi_{\textsc{omt}}}(Y|X=\mathbf{x}+ \frac{1}{2} \varepsilon_1\boldsymbol{v})\|_{\text{op}} 
     \le u_{\text{max}} + v_{\text{max}}=C_{\pi_{\textsc{omt}}} \ < \infty \ .
\end{gather}
where $C_{\pi_{\textsc{omt}}}$ is defined in Eq.~\ref{eq:omt_sup_cov}.
\end{proof}
\begin{lemma}[Lemma 3.21, \cite{bauerschmidt2024stochastic}]
\label{lem:gen_tilt_stable}
Let $\nu$ be a probability measure on $\mathbb{R}^d$ with finite second moment. Suppose for every $\mathbf{x} \in \mathbb{R}^d$, $\nu$ is tilt-stable with constant $\alpha$. Then for any probability measure $\rho$ on $\mathbb{R}^d$ with finite second moment, 
$$ \| \mathbb{E}_{\rho}[\mathbf{x}] - \mathbb{E}_{\nu}[\mathbf{x}] \|^2 \le 2\alpha D_{KL}(\rho \| \nu) \ . $$
\end{lemma}
\vspace{.1in}
\begin{prop}
\label{prop:gen_OMT_tilt_stability}
    Let $T^{\nu_0 \to \nu_1}_{\textsc{omt}}$ denote the OMT map between two probability measures $\nu_0 \in M_{K_0}(\mathbb{R}^d)$ and $\nu_1 \in M_{K_1}(\mathbb{R}^d)$, as defined in \ref{eq:Reg-OMT_2} with parameter $\varepsilon_1$. If the mixture models satisfy regularity conditions in Definition~\ref{def:regularity_conditions}, there exist constants $a'_0, b'_0 \in \mathbb{R}_{+}$ such that, for any probability measure $\rho \in M_{K_0}(\mathbb{R}^d)$, 
$$ \displaystyle \mathbb{E}_{\pi_{\textsc{omt}}(\mathbf{x}, \mathbf{z})} [\| T^{\rho \rightarrow \nu_1}_{\textsc{omt}}(\mathbf{z}) - T^{\nu_0 \rightarrow \nu_1}_{\textsc{omt}}(\mathbf{x}) \| ] \le  a'_0 \mathcal{W}_2(\nu_0, \rho) + b'_0 \mathcal{W}^{1/2}_2(\nu_0, \rho)$$
\end{prop}
\begin{proof}
Given $\pi_{\textsc{omt}}(\mathbf{y} | \mathbf{x}) = \displaystyle \sum_{i,j} \tilde{\omega}_{ij}(\mathbf{x}) p_{ij}(\mathbf{y} | \mathbf{x})$ and $\pi_{\textsc{omt}}(\mathbf{y} | \mathbf{z}) = \displaystyle \sum_{r, k} \tilde{\omega}_{rk}(\mathbf{z}) p_{rk}(\mathbf{y} | \mathbf{z})$, we can obtain the following upper bound:
\begin{align}
\label{eq:mean_KL_rho_nu}
    D_{KL}(\pi_{\textsc{omt}}(\mathbf{y} | \mathbf{z}) \| \pi_{\textsc{omt}}(\mathbf{y} | \mathbf{x})) &\le \displaystyle \sum_{i,j} \sum_{r,k} \tilde{\omega}_{ij}(\mathbf{x}) \tilde{\omega}_{rk}(\mathbf{z}) D_{KL}(p_{rk}(\mathbf{y|\mathbf{z}}) \| p_{ij}(\mathbf{y|\mathbf{x}})) 
\end{align}
We now compute the pairwise KL divergence.
\begin{align}
D_{KL}(p_{rk}(\mathbf{y|\mathbf{z}}) \| p_{ij}(\mathbf{y|\mathbf{x}}))
    = \displaystyle  \int &\log{\frac{\exp{\left(\frac{\langle \mathbf{z}, \mathbf{y} \rangle - \tilde{\varphi}_{rk}(\mathbf{z}) - \tilde{\psi}_{rk}(\mathbf{y})}{1/2 \varepsilon_1} \right)} \mu_{1_{k}}(\mathbf{y})}{\exp{\left(\frac{\langle \mathbf{x}, \mathbf{y} \rangle - \tilde{\varphi}_{ij}(\mathbf{x}) - \tilde{\psi}_{ij}(\mathbf{y})}{1/2 \varepsilon_1} \right)}\mu_{1_{j}}(\mathbf{y})}} p_{rk}(\mathbf{y} | \mathbf{z}) d\mathbf{y} \nonumber \\
    = \displaystyle  \int 2 \varepsilon^{-1}_1  \Big( \langle \mathbf{z} - \mathbf{x}, \mathbf{y} \rangle + \tilde{\varphi}_{ij}(\mathbf{x}) - \tilde{\varphi}_{rk}(\mathbf{z}) + \tilde{\psi}_{ij}&(\mathbf{y}) - \tilde{\psi}_{rk}(\mathbf{y}) \Big) p_{rk}(\mathbf{y} | \mathbf{z}) d\mathbf{y}+ \int \left(\log{\mu_{1_{k}}(\mathbf{y})} - \log{\mu_{1_{j}}(\mathbf{y})} \right) p_{rk}(\mathbf{y} | \mathbf{z}) d\mathbf{y} \nonumber \\
    = 2 \varepsilon^{-1}_1 \left(\langle \mathbf{z} - \mathbf{x}, T_{rk}(\mathbf{z}) \rangle  + \tilde{\varphi}_{ij}(\mathbf{x}) - \tilde{\varphi}_{rk}(\mathbf{z}) \right) +&
    2\varepsilon^{-1}_1 \mathbb{E}_{p_{rk}(\mathbf{y} | \mathbf{z})} [\tilde{\psi}_{ij}(\mathbf{y}) - \tilde{\psi}_{rk}(\mathbf{y})] + \mathbb{E}_{p_{rk}(\mathbf{y} | \mathbf{z})} [\log{\mu_{1_{k}}(\mathbf{y})} - \log{\mu_{1_{j}}(\mathbf{y})} ]
    \label{eq:KL_y_z_x}
\end{align}
Defining $h_{ij}(\mathbf{y}) = 2\varepsilon^{-1}_1 \tilde{\psi}_{ij}(\mathbf{y}) - \log{\mu_{1_j}(\mathbf{y})}$, the KL divergence in Eq.~\ref{eq:KL_y_z_x} is simplified to
\begin{align}
  D_{KL}(p_{rk}(\mathbf{y|\mathbf{z}}) \| p_{ij}(\mathbf{y|\mathbf{x}})) &=   2 \varepsilon^{-1}_1 \left(\langle \mathbf{z} - \mathbf{x}, T_{rk}(\mathbf{z}) \rangle  + \tilde{\varphi}_{ij}(\mathbf{x}) - \tilde{\varphi}_{rk}(\mathbf{z}) \right) + \mathbb{E}_{p_{rk}(\mathbf{y} | \mathbf{z})} [h_{ij}(\mathbf{y}) - h_{rk}(\mathbf{y})]
\end{align}    
Similarly, the reverse direction of the KL divergence calculation can be expressed as
\begin{align}
\label{eq:kl_ij_rk}
D_{KL}(p_{ij}(\mathbf{y|\mathbf{x}}) \| p_{rk}(\mathbf{y|\mathbf{z}}))
    =& 2 \varepsilon^{-1}_1 \left(\langle \mathbf{x} - \mathbf{z}, T_{ij}(\mathbf{x}) \rangle  + \tilde{\varphi}_{rk}(\mathbf{z}) - \tilde{\varphi}_{ij}(\mathbf{x}) \right) + \mathbb{E}_{p_{ij}(\mathbf{y} | \mathbf{x})} [h_{rk}(\mathbf{y}) - h_{ij}(\mathbf{y})] 
\end{align}
Summing the two KL divergences yields
\begin{align}
\label{eq:KL_ijrk_upb2}
   D_{KL}(p_{ij}(\mathbf{y|\mathbf{x}}) \| p_{rk}(\mathbf{y|\mathbf{z}})) + D_{KL}(p_{rk}(\mathbf{y|\mathbf{z}}) \| p_{ij}(\mathbf{y|\mathbf{x}})) 
   =& 2 \varepsilon^{-1}_1 \langle \mathbf{x} - \mathbf{z}, T_{ij}(\mathbf{x}) - T_{rk}(\mathbf{z}) \rangle + \nonumber \\ 
\mathbb{E}_{p_{rk}(\mathbf{y} | \mathbf{z})} \underbrace{[h_{ij}(\mathbf{y}) - h_{rk}(\mathbf{y})]}_{f_{ijrk}(\mathbf{y})} &- \mathbb{E}_{p_{ij}(\mathbf{y} | \mathbf{x})} [h_{ij}(\mathbf{y}) - h_{rk}(\mathbf{y})] 
\end{align}
where $f_{ijrk} \equiv 0$ when $i=r$ and $j=k$. Thus, for any pair of components $(i,j)$ and $(r,k)$,
\begin{align}
\label{eq:KL_ijrk_upb3}
   D_{KL}(p_{ij}(\mathbf{y|\mathbf{x}}) \| p_{rk}(\mathbf{y|\mathbf{z}})) &\le 2 \varepsilon^{-1}_1 \langle \mathbf{x} - \mathbf{z}, T_{ij}(\mathbf{x}) - T_{rk}(\mathbf{z}) \rangle + 
\left(\mathbb{E}_{p_{rk}(\mathbf{y} | \mathbf{z})} [f_{ijrk}(\mathbf{y})] - \mathbb{E}_{p_{ij}(\mathbf{y} | \mathbf{x})} [f_{ijrk}(\mathbf{y})]\right) (1 - \delta_{ir} \delta_{jk})
\end{align}
Substituting the component-wise bound in \ref{eq:KL_ijrk_upb3} into the mixture expression in Eq.~\ref{eq:mean_KL_rho_nu} gives
\begin{align}
    D_{KL}(\pi_{\textsc{omt}}(\mathbf{y} | \mathbf{z}) \| \pi_{\textsc{omt}}(\mathbf{y} | \mathbf{x})) &\le \displaystyle \sum_{i,j} \sum_{r,k} \tilde{\omega}_{ij}(\mathbf{x}) \tilde{\omega}_{rk}(\mathbf{z}) \Big(2 \varepsilon^{-1}_1 \langle \mathbf{x} - \mathbf{z}, T_{ij}(\mathbf{x}) - T_{rk}(\mathbf{z}) \rangle + \nonumber \\
&\left(\mathbb{E}_{p_{rk}(\mathbf{y} | \mathbf{z})} [f_{ijrk}(\mathbf{y})] - \mathbb{E}_{p_{ij}(\mathbf{y} | \mathbf{x})} [f_{ijrk}(\mathbf{y})]\right) (1 - \delta_{ir} \delta_{jk}) \Big) \ .
\end{align}
By the triangle inequality
\begin{align}
     \sum_{i,j,rk} \tilde{\omega}_{ij}(\mathbf{x}) \tilde{\omega}_{rk}(\mathbf{z}) \left(\mathbb{E}_{p_{rk}(\mathbf{y} | \mathbf{z})} [f_{ijrk}(\mathbf{y})] - \mathbb{E}_{p_{ij}(\mathbf{y} | \mathbf{x})} [f_{ijrk}(\mathbf{y})]\right) \le& \sum_{i,j,rk} \Big| \mathbb{E}_{p_{rk}(\mathbf{y} | \mathbf{z})} [\tilde{\omega}_{ij}(\mathbf{x}) \tilde{\omega}_{rk}(\mathbf{z}) f_{ijrk}(\mathbf{y})] - \nonumber \\
     &\mathbb{E}_{p_{ij}(\mathbf{y} | \mathbf{x})} [\tilde{\omega}_{ij}(\mathbf{x}) \tilde{\omega}_{rk}(\mathbf{z}) f_{ijrk}(\mathbf{y})] \Big|
\end{align}
Differentiating $f_{ijrk}(\mathbf{y})$, 
\begin{gather}
\label{eq:deriv_f_ijrk}
    \nabla_{\mathbf{y}} f_{ijrk}(\mathbf{y}) = 2\varepsilon^{-1}_1 \nabla_{\mathbf{y}}  \tilde{\psi}_{ij}(\mathbf{y}) - \nabla_{\mathbf{y}} \log{\mu_{1_j}(\mathbf{y})} - 2\varepsilon^{-1}_1 \nabla_{\mathbf{y}}  \tilde{\psi}_{rk}(\mathbf{y}) + \nabla_{\mathbf{y}} \log{\mu_{1_k}(\mathbf{y})} 
\end{gather}
where 
\begin{align*}
    \nabla_{\mathbf{y}}  \tilde{\psi}_{ij}(\mathbf{y}) &= \displaystyle \nabla_{\mathbf{y}} \frac{1}{2}\varepsilon_1 \log{\left(\int \exp{\left(\frac{\langle \mathbf{x},\mathbf{y} \rangle - \tilde{\varphi}_{ij}(\mathbf{x})}{1/2 \varepsilon_1} \right)} \mu_{0_i}(\mathbf{x}) d\mathbf{x}\right)} \\[.1in]
    & = \displaystyle \frac{1}{2}\varepsilon_1 \frac{\nabla_{\mathbf{y}}\int \exp{\left(\frac{\langle \mathbf{x},\mathbf{y} \rangle - \tilde{\varphi}_{ij}(\mathbf{x})}{1/2 \varepsilon_1} \right)} \mu_{0_i}(\mathbf{x}) d\mathbf{x}}{\int \exp{\left(\frac{\langle \mathbf{x}',\mathbf{y} \rangle - \tilde{\varphi}_{ij}(\mathbf{x}')}{1/2 \varepsilon_1} \right)} \mu_{0_i}(\mathbf{x}') d\mathbf{x}'} \\[.1in]
    &= \displaystyle \frac{1}{2}\varepsilon_1 \frac{\int  2\varepsilon^{-1}_1 \exp{\left(\frac{\langle \mathbf{x},\mathbf{y} \rangle - \tilde{\varphi}_{ij}(\mathbf{x})}{1/2 \varepsilon_1} \right)} \mu_{0_i}(\mathbf{x}) \mathbf{x} d\mathbf{x}}{\int \exp{\left(\frac{\langle \mathbf{x}',\mathbf{y} \rangle - \tilde{\varphi}_{ij}(\mathbf{x}')}{1/2 \varepsilon_1} \right)} \mu_{0_i}(\mathbf{x}') d\mathbf{x}'} \\[.1in]
    &= \displaystyle \int \frac{ \exp{\left(\frac{\langle \mathbf{x},\mathbf{y} \rangle - \tilde{\varphi}_{ij}(\mathbf{x})}{1/2 \varepsilon_1} \right)} \mu_{0_i}(\mathbf{x}) \mathbf{x} d\mathbf{x}}{\int \exp{\left(\frac{\langle \mathbf{x}',\mathbf{y} \rangle - \tilde{\varphi}_{ij}(\mathbf{x}')}{1/2 \varepsilon_1} \right)} \mu_{0_i}(\mathbf{x}') d\mathbf{x}'} \\[.1in]
    &= \displaystyle \int \frac{ \exp{\left(\frac{\langle \mathbf{x},\mathbf{y} \rangle - \tilde{\varphi}_{ij}(\mathbf{x}) - \tilde{\psi}_{ij}(\mathbf{y})}{1/2 \varepsilon_1} \right)} \mu_{0_i}(\mathbf{x}) \mu_{1_j}(\mathbf{y}) \mathbf{x} d\mathbf{x}}{\int \exp{\left(\frac{\langle \mathbf{x}',\mathbf{y} \rangle - \tilde{\varphi}_{ij}(\mathbf{x}') - \tilde{\psi}_{ij}(\mathbf{y})}{1/2 \varepsilon_1} \right)} \mu_{0_i}(\mathbf{x}') \mu_{1_j}(\mathbf{y}) d\mathbf{x}'} \\[.1in]
    & = \displaystyle \int \frac{p_{ij}(\mathbf{x}, \mathbf{y})}{\int p_{ij}(\mathbf{x}', \mathbf{y}) d\mathbf{x}'} \mathbf{x} d\mathbf{x} = \displaystyle \int p_{ij}(\mathbf{x} | \mathbf{y}) \mathbf{x} d\mathbf{x}
\end{align*}
Accordingly, the derivation in Eq.~\ref{eq:deriv_f_ijrk} can be simplified as
\begin{align}
    \nabla_{\mathbf{y}} f_{ijrk}(\mathbf{y}) 
    & = 2\varepsilon^{-1}_1 \left( \mathbb{E}_{p_{ij}(\mathbf{x}|\mathbf{y})} [\mathbf{x}] - \mathbb{E}_{p_{rk}(\mathbf{z}|\mathbf{y})} [\mathbf{z}]\right) - \left( \nabla_{\mathbf{y}} \log{\mu_{1_j}(\mathbf{y})} - \nabla_{\mathbf{y}} \log{\mu_{1_k}(\mathbf{y})} \right) \ .
\end{align}
Defining $\tilde{\omega}_{ijrk}(\mathbf{x}, \mathbf{z}) = \tilde{\omega}_{ij}(\mathbf{x}) \tilde{\omega}_{rk}(\mathbf{z})$, we bound the norm of $\tilde{\omega}_{ijrk}(\mathbf{x}, \mathbf{z})\nabla_{\mathbf{y}} f_{ijrk}(\mathbf{y})$ as
\begin{align}
    \| \tilde{\omega}_{ijrk}(\mathbf{x}, \mathbf{z})\nabla_{\mathbf{y}} f_{ijrk}(\mathbf{y}) \|_{\text{op}}
    & \le 2\varepsilon^{-1}_1 \| \tilde{\omega}_{ij}(\mathbf{x}) \tilde{\omega}_{rk}(\mathbf{z})\left( \mathbb{E}_{p_{ij}(\mathbf{x}|\mathbf{y})} [\mathbf{x}] - \mathbb{E}_{p_{rk}(\mathbf{z}|\mathbf{y})} [\mathbf{z}]\right) - \nonumber \\
    & \tilde{\omega}_{ij}(\mathbf{x}) \tilde{\omega}_{rk}(\mathbf{z})\left( \nabla_{\mathbf{y}} \log{\mu_{1_j}(\mathbf{y})} - \nabla_{\mathbf{y}} \log{\mu_{1_k}(\mathbf{y})} \right) \|_{\text{op}}\ . \nonumber \\[.1in]
    & \le 2\varepsilon^{-1}_1 \| \tilde{\omega}_{ij}(\mathbf{x}) \tilde{\omega}_{rk}(\mathbf{z})\left( \mathbb{E}_{p_{ij}(\mathbf{x}|\mathbf{y})} [\mathbf{x}] - \mathbb{E}_{p_{rk}(\mathbf{z}|\mathbf{y})} [\mathbf{z}]\right) \| + \nonumber \\
    & \| \tilde{\omega}_{rk}(\mathbf{z}) \frac{\tilde{\omega}_{ij}(\mathbf{x})}{{\theta}_{ij}(\mathbf{y})}{\theta}_{ij}(\mathbf{y}) \nabla_{\mathbf{y}} \log{\mu_{1_j}(\mathbf{y})} - \tilde{\omega}_{ij}(\mathbf{x}) \frac{\tilde{\omega}_{rk}(\mathbf{z})}{{\theta}_{rk}(\mathbf{y})}{\theta}_{rk}(\mathbf{y})\nabla_{\mathbf{y}} \log{\mu_{1_k}(\mathbf{y})} \|_{\text{op}}\
\end{align}
where ${\theta}_{ij}(\mathbf{y}) := \displaystyle \omega_{ij} \frac{\mu_{1_j}(\mathbf{y})}{\nu_1(\mathbf{y})}$,  $\tilde{\omega}_{ij}(\mathbf{x}) := \displaystyle \omega_{ij} \frac{\mu_{0_i}(\mathbf{x})}{\nu_0(\mathbf{x})}$, and consequently, $ \displaystyle \frac{\tilde{\omega}_{ij}(\mathbf{x})}{{\theta}_{ij}(\mathbf{y})} = \frac{\mu_{0_i}(\mathbf{x})}{\mu_{1_j}(\mathbf{y})} \frac{\nu_{1}(\mathbf{y})}{\nu_{0}(\mathbf{x})}$.

Under the regularity conditions in Definition~\ref{def:regularity_conditions}, the conditional distribution has a finite, bounded second moment, which guarantees that the mass of $p_{ij}(\mathbf{y} | \mathbf{x})$ is concentrated in a high-probability, 
$$\displaystyle P_{p_{ij}(\cdot | \mathbf{x})} \Big( \|\mathbf{y} - T({\mathbf{x}})\| \ge R \Big) \le \frac{\text{Tr}(\text{Cov}^{ij}_{\mathbf{y} | \mathbf{x}})}{R^2} \ \   \text{(Chebyshev's inequality).} $$

Therefore, the following ratio attains a finite, local maximum parameterised by $\mathbf{x}$:
$$\sup_{\mathbf{y} \in \text{supp}(\mathbf{x})} \displaystyle \frac{\tilde{\omega}_{ij}(\mathbf{x})}{{\theta}_{ij}(\mathbf{y})} = \sup_{\mathbf{y} \in \text{supp}(\mathbf{x})} \frac{\mu_{0_i}(\mathbf{x})}{\nu_{0}(\mathbf{x})} \frac{\nu_{1}(\mathbf{y})}{\mu_{1_j}(\mathbf{y}) } = M(\mathbf{x}) < \infty$$

Therefore, for given $\mathbf{x}, \mathbf{z}$, both $\| \frac{\tilde{\omega}_{ij}(\mathbf{x})}{{\theta}_{ij}(\mathbf{y})}\|$ and $\|{\theta}_{ij}(\mathbf{y}) \nabla_{\mathbf{y}} \log{\mu_{1_j}(\mathbf{y})}\|_{\text{op}}$ are well-defined and uniformly bounded, for all $i,j$. Consequently, the function
\begin{gather*}
    |\tilde{\omega}_{ijrk}(\mathbf{x}, \mathbf{z})f_{ijrk}(\mathbf{y}) - \tilde{\omega}_{ijrk}(\mathbf{x}, \mathbf{z}) f_{ijrk}(\mathbf{y}') | \le l_{ijrk}(\mathbf{x}, \mathbf{z}) \ \| \mathbf{y} - \mathbf{y}' \| .
\end{gather*}
where
\begin{gather}
\label{eq:l_ijrk}
    l_{ijrk}(\mathbf{x}, \mathbf{z}) = \displaystyle \sup_{\mathbf{y} \in \text{supp}(\mathbf{x}, \mathbf{z})} \| \tilde{\omega}_{ijrk}(\mathbf{x}, \mathbf{z}) \nabla_{\mathbf{y}} f_{ijrk}(\mathbf{y}) \|_{\text{op}} < \infty 
\end{gather}
Let us define $\tilde{f}_{ijrk}(\mathbf{y}) := l_{ijrk}^{-1} (\mathbf{x}, \mathbf{z})\tilde{\omega}_{ijrk}(\mathbf{x}, \mathbf{z}) f_{ijrk}(\mathbf{y})$, which is 1-Lipschitz. By the Kantorovich--Rubinstein theorem~\cite{villani2021topics},
\begin{align}
\label{eq:exp_diff_ijrk}
    |\mathbb{E}_{p_{rk}(\mathbf{y} | \mathbf{z})}[\tilde{\omega}_{ijrk}(\mathbf{x}, \mathbf{z})f_{ijrk}(\mathbf{y})] - \mathbb{E}_{p_{ij}(\mathbf{y} | \mathbf{x})}[\tilde{\omega}_{ijrk}(\mathbf{x}, \mathbf{z})f_{ijrk}(\mathbf{y})]| &= l_{ijrk} (\mathbf{x}, \mathbf{z}) \ | \int_{\mathcal{Y}} \tilde{f}_{ijrk}(\mathbf{y}) p_{rk}(\mathbf{y} | \mathbf{z}) d\mathbf{y} - \int_{\mathcal{Y}} \tilde{f}_{ijrk}(\mathbf{y}') p_{ij}(\mathbf{y} | \mathbf{x}) d\mathbf{y}'| \nonumber \\
    \le l_{ijrk} (\mathbf{x}, \mathbf{z}) \sup_{\|\tilde{f}_{ijrk}(\mathbf{y})\|_L < 1} & \{ \mathbb{E}_{p_{rk}(\mathbf{y} | \mathbf{z})}[\tilde{f}_{ijrk}(\mathbf{y})] - \mathbb{E}_{p_{ij}(\mathbf{y} | \mathbf{x})}[\tilde{f}_{ijrk}(\mathbf{y})] \} \nonumber \\
    \le l_{ijrk} (\mathbf{x}, \mathbf{z}) \sup_{\|\tilde{f}_{ijrk}(\mathbf{y})\|_L < 1} \sup_{\mathbf{x}'} \{\| \nabla_{\mathbf{x}'} & \mathbb{E}_{p_{ij}(\mathbf{y}|\mathbf{x}')} [\tilde{f}_{ijrk}(\mathbf{y})]\|\} \ \| \mathbf{x} - \mathbf{z} \|
\end{align}
where
\begin{align*}
    \nabla_{\mathbf{x}'} \mathbb{E}_{p_{ij}(\mathbf{y}|\mathbf{x}')} [\tilde{f}_{ijrk}(\mathbf{y})] &= \int \tilde{f}_{ijrk}(\mathbf{y}) \nabla_{\mathbf{x}'} p_{ij}(\mathbf{y}|\mathbf{x}') d\mathbf{y} \\
    &= \int 2\varepsilon^{-1}_1  \tilde{f}_{ijrk}(\mathbf{y})  (\mathbf{y} - T_{ij}(\mathbf{x}'))^T p_{ij}(\mathbf{y}|\mathbf{x}') d\mathbf{y} \hspace{.2in} \text{(Eq.~\ref{eq:nabla_p_ij})} \\
    & = 2\varepsilon^{-1}_1 \mathbb{E}_{p_{ij}(\mathbf{y} | \mathbf{x}')} [\tilde{f}_{ijrk}(\mathbf{y})  (\mathbf{y} - T_{ij}(\mathbf{x}'))^T]
\end{align*}
Using the centering property,  $\mathbb{E}_{p_{ij}(\mathbf{y} | \mathbf{x})} [  (\mathbf{y} - T_{ij}(\mathbf{x}'))^T]=0$, 
\begin{align*}
     \mathbb{E}_{p_{ij}(\mathbf{y} | \mathbf{x}')} [\tilde{f}_{ijrk}(\mathbf{y})  (\mathbf{y} - T_{ij}(\mathbf{x}'))^T] &= \mathbb{E}_{p_{ij}(\mathbf{y} | \mathbf{x}')} [\tilde{f}_{ijrk}(\mathbf{y})  (\mathbf{y} - T_{ij}(\mathbf{x}'))^T] - \tilde{f}_{ijrk}(T_{ij}(\mathbf{x}')) \mathbb{E}_{p_{ij}(\mathbf{y} | \mathbf{x}')} [  (\mathbf{y} - T_{ij}(\mathbf{x}'))^T] \\
     &= \mathbb{E}_{p_{ij}(\mathbf{y} | \mathbf{x}')} [\left(\tilde{f}_{ijrk}(\mathbf{y}) - \tilde{f}_{ijrk}(T_{ij}(\mathbf{x}'))\right) \left(\mathbf{y} - T_{ij}(\mathbf{x}')\right)^T]
\end{align*}
Since $\tilde{f}_{ijrk}(\mathbf{y})$ is 1-Lipschitz, $|\tilde{f}_{ijrk}(\mathbf{y}) - \tilde{f}_{ijrk}(T_{ij}(\mathbf{x}'))| \le \|\mathbf{y} - T_{ij}(\mathbf{x}')\|$. 
Applying Jensen’s and Cauchy–Schwarz inequalities gives
\begin{align*}
    \| \mathbb{E}_{p_{ij}(\mathbf{y} | \mathbf{x}')} [\tilde{f}_{ijrk}(\mathbf{y})  (\mathbf{y} - T_{ij}(\mathbf{x}'))^T] \| & \le \mathbb{E}_{p_{ij}(\mathbf{y} | \mathbf{x})} [\|\mathbf{y} - T_{ij}(\mathbf{x}')\|^2] \le d \| \text{Cov}^{ij}_{\mathbf{y} | \mathbf{x}'} \|_{\text{op}} \hspace{.2in} (\mathbf{y} \in \mathbb{R}^d)
\end{align*}
Therefore, 
\begin{align*}
     \sup_{i,j,\mathbf{x}'} \{\| \nabla_{\mathbf{x}'} \mathbb{E}_{p_{ij}(\mathbf{y}|\mathbf{x}')} [\tilde{f}_{ijrk}(\mathbf{y})]\|\} = 2\varepsilon^{-1}_1  \sup_{i,j,\mathbf{x}'} \| \mathbb{E}_{p_{ij}(\mathbf{y} | \mathbf{x}')} [\tilde{f}_{ijrk}(\mathbf{y})  (\mathbf{y} - T_{ij}(\mathbf{x}'))^T] \| &\le 2d\varepsilon^{-1}_1 \sup_{i,j,\mathbf{x}'}\| \text{Cov}^{ij}_{\mathbf{y} | \mathbf{x}'} \|_{\text{op}} \\
     &= 2d\varepsilon^{-1}_1 c_{ij} \hspace{.1in} \text{(Eq.~\ref{eq:comp_tilt_cnt})}
\end{align*}
Accordingly, the gap between the cross-component terms in Eq.~\eqref{eq:exp_diff_ijrk} is bounded as follows.
\begin{align}
\label{eq:sup_exp_diff}
     |\mathbb{E}_{p_{rk}(\mathbf{y} | \mathbf{z})}[f_{ijrk}(\mathbf{y})] - \mathbb{E}_{p_{ij}(\mathbf{y} | \mathbf{x})}[f_{ijrk}(\mathbf{y})]| &\le 2 d \varepsilon^{-1}_1 l_{ijrk} c_{ij} (1 - \delta_{ir} \delta_{jk}) \|\mathbf{x} - \mathbf{z} \|
\end{align}
%

Substituting \eqref{eq:sup_exp_diff} into \ref{eq:KL_ijrk_upb2} returns
\begin{align}
    D_{KL}(\pi_{\textsc{omt}}(\mathbf{y} | \mathbf{z}) \| \pi_{\textsc{omt}}(\mathbf{y} | \mathbf{x})) &\le \displaystyle \sum_{i,j,r,k} 2 \varepsilon^{-1}_1 \tilde{\omega}_{ij}(\mathbf{x}) \tilde{\omega}_{rk}(\mathbf{z})  \langle \mathbf{x} - \mathbf{z}, T_{ij}(\mathbf{x}) - T_{rk}(\mathbf{z}) \rangle  + \nonumber \\
    & \sum_{i,j,r,k}  2 \varepsilon^{-1}_1 d \ l_{ijrk} c_{ij}  (1 - \delta_{ir} \delta_{jk}) \|\mathbf{x} - \mathbf{z} \|
\end{align}
%
Applying Cauchy--Schwarz gives
\begin{align}
\label{eq:mean_KL_ihrk_upb_}
    D_{KL}(\pi_{\textsc{omt}}(\mathbf{y} | \mathbf{z}) \| \pi_{\textsc{omt}}(\mathbf{y} | \mathbf{x}))  \le & \displaystyle \sum_{i,j,r,k} 2 \varepsilon^{-1}_1 \tilde{\omega}_{ij}(\mathbf{x}) \tilde{\omega}_{rk}(\mathbf{z})  \| \mathbf{x} - \mathbf{z}\| \| T_{ij}(\mathbf{x}) - T_{rk}(\mathbf{z}) \|  + \nonumber \\
    & 2 \varepsilon^{-1}_1 \|\mathbf{x} - \mathbf{z} \|\underbrace{\sum_{i,j,r,k}  d \ l_{ijrk} c_{ij}  (1 - \delta_{ir} \delta_{jk})}_{\mathcal{D}_0}  \nonumber \\
    &=  2 \varepsilon^{-1}_1 \| \mathbf{x} - \mathbf{z}\| \ \| T^{\rho \rightarrow \nu_1}_{\textsc{omt}}(\mathbf{z}) - T^{\nu_0 \rightarrow \nu_1}_{\textsc{omt}}(\mathbf{x}) \| + 2 \varepsilon^{-1}_1 \mathcal{D}_0 \|\mathbf{x} - \mathbf{z} \| 
\end{align}
From the tilt-stability condition (Proposition~\ref{prop:app_OMT_tilt_cnt}, Lemma~\ref{lem:gen_tilt_stable}), we further obtain
\begin{gather}
     \| T^{\rho \rightarrow \nu_1}_{\textsc{omt}}(\mathbf{z}) - T^{\nu_0 \rightarrow \nu_1}_{\textsc{omt}}(\mathbf{x}) \|^2  \le   2 C_{\pi_{\textsc{omt}}} D_{KL}(\pi_{\textsc{omt}}(\mathbf{y} | \mathbf{z}) \| \pi_{\textsc{omt}}(\mathbf{y} | \mathbf{x}))
     \label{eq:mean_KL_ihrk_lowb}
\end{gather}
Combining this with the upper bound in Eq.~\eqref{eq:mean_KL_ihrk_upb_} yields
\begin{align}
\label{eq:mix_exp_upb}
     \| T^{\rho \rightarrow \nu_1}_{\textsc{omt}}(\mathbf{z}) - T^{\nu_0 \rightarrow \nu_1}_{\textsc{omt}}(\mathbf{x}) \|^2  \le 4 \varepsilon^{-1}_1 C_{\pi_{\textsc{omt}}} &  \| \mathbf{x} - \mathbf{z}\| \ \| T^{\rho \rightarrow \nu_1}_{\textsc{omt}}(\mathbf{z}) - T^{\nu_0 \rightarrow \nu_1}_{\textsc{omt}}(\mathbf{x}) \| + 4 \varepsilon^{-1}_1 C_{\pi_{\textsc{omt}}}\mathcal{D}_0 \|\mathbf{x} - \mathbf{z} \| 
 \end{align}
The inequality in Eq.~\ref{eq:mix_exp_upb} can be further simplified as
\begin{align}
   \| T^{\rho \rightarrow \nu_1}_{\textsc{omt}}(\mathbf{z}) - T^{\nu_0 \rightarrow \nu_1}_{\textsc{omt}}(\mathbf{x}) \|  \le 4 \varepsilon^{-1}_1 C_{\pi_{\textsc{omt}}} &  \| \mathbf{x} - \mathbf{z}\| + \left(4 \varepsilon^{-1}_1 C_{\pi_{\textsc{omt}}}\mathcal{D}_0\right)^{1/2} \|\mathbf{x} - \mathbf{z} \|^{1/2}   \hspace{.2in} \text{(derivations in \ref{sec:add_deriv})}
\end{align}
Taking expectation w.r.t.\ $\pi_{\textsc{omt}}(\mathbf{x},\mathbf{z})$, we obtain
\begin{align}
   \mathbb{E}_{\pi_{\textsc{omt}}(\mathbf{x}, \mathbf{z})} [\| T^{\rho \rightarrow \nu_1}_{\textsc{omt}}(\mathbf{z}) - T^{\nu_0 \rightarrow \nu_1}_{\textsc{omt}}(\mathbf{x}) \|]  \le 4 \varepsilon^{-1}_1 C_{\pi_{\textsc{omt}}} &  \mathbb{E}_{\pi_{\textsc{omt}}(\mathbf{x}, \mathbf{z})} [\| \mathbf{x} - \mathbf{z}\|] + \left(4 \varepsilon^{-1}_1 C_{\pi_{\textsc{omt}}}\mathcal{D}_0\right)^{1/2} \mathbb{E}_{\pi_{\textsc{omt}}(\mathbf{x}, \mathbf{z})} [\|\mathbf{x} - \mathbf{z} \|^{1/2}]  
\end{align}
Applying Jensen’s inequality,
$$\mathbb{E}_{\pi_{\textsc{omt}}(\mathbf{x}, \mathbf{z})} [\| \mathbf{x} - \mathbf{z}\|] \leq \mathbb{E}_{\pi_{\textsc{omt}}(\mathbf{x}, \mathbf{z})}[ \| \mathbf{x} - \mathbf{z}\|^2]^{1/2}=\mathcal{W}_2(\nu_0, \rho) , $$
$$
\mathbb{E}_{\pi_{\textsc{omt}}(\mathbf{x}, \mathbf{z})} [\| \mathbf{x} - \mathbf{z}\|^{1/2}] \leq \mathbb{E}_{\pi_{\textsc{omt}}(\mathbf{x}, \mathbf{z})}[ \| \mathbf{x} - \mathbf{z}\|]^{1/2} \le \mathcal{W}^{1/2}_2(\nu_0, \rho) , 
$$
we arrive at
\begin{align}
\label{eq:final_rho_nu_bound}
    \displaystyle \mathbb{E}_{\pi_{\textsc{omt}}(\mathbf{x}, \mathbf{z})} [\| T^{\rho \rightarrow \nu_1}_{\textsc{omt}}(\mathbf{z}) - T^{\nu_0 \rightarrow \nu_1}_{\textsc{omt}}(\mathbf{x}) \| ] &\le  4 \varepsilon^{-1}_1 C_{\pi_{\textsc{omt}}} \mathcal{W}_2(\nu_0, \rho) + \left(4 \varepsilon^{-1}_1 C_{\pi_{\textsc{omt}}}\mathcal{D}_0\right)^{1/2} \mathcal{W}^{1/2}_2(\nu_0, \rho) \nonumber \\
    &= a'_0 \mathcal{W}_2(\nu_0, \rho) + b'_0 \mathcal{W}^{1/2}_2(\nu_0, \rho)
\end{align}
where 
\begin{gather}
\label{eq:upbound_a_b}
    a'_0 = 4 \varepsilon^{-1}_1 C_{\pi_{\textsc{omt}}} , \hspace{.15in}
    b'_0 =  (a'_0 D_0)^{1/2} \ .
\end{gather} 
%

\end{proof}
\begin{theorem}[Stability of OMT under perturbation]
\label{theorem:stability}
Let $\nu_0 \in M_{K_0}(\mathbb{R}^d)$, $\nu_0^{\prime} \in M_{K^{\prime}_0}(\mathbb{R}^d)$, and $\nu_1 \in M_{K_1}(\mathbb{R}^d)$ be mixture probability densities that admit density regularity conditions in Definition~\ref{def:regularity_conditions}. Then there exist constants $a_0, b_0 \in \mathbb{R}_+$ such that
\begin{align*}
    \mathbb{E}_{\nu_0} \left[\|T^{\nu_0 \rightarrow \nu_1}_{\textsc{omt}}(\mathbf{x}) - T^{\nu^{\prime}_0 \rightarrow \nu_1}_{\textsc{omt}}(\mathbf{x}) \| \right] & \le a_0 \mathcal{W}_2 (\nu_{0}, \nu^{\prime}_{0}) + b_0 \mathcal{W}^{1/2}_2 (\nu_{0}, \nu^{\prime}_{0}) \ .
\end{align*}
\end{theorem}
\begin{proof}
Let $(\Omega^{\prime}, P^{\prime})$ denote the transport weights and the set couplings associated with the OMT between $\nu_0$ and $\nu^{\prime}_0$, where $\nu'_0$ denotes the perturbed counterpart of $\nu_0$, for all $\mathbf{x}, \mathbf{x}^{\prime} \in \mathbb{R}^d$.

For brevity, we define $T^{\nu_0 \rightarrow \nu_1}_{\textsc{omt}} := T^{\nu_0}_{\textsc{omt}}$. We begin by quantifying the average norm of deviation between the two optimal transport maps:
\begin{align*}
\int \|T^{\nu_0}_{\textsc{omt}}(\mathbf{x}) - T^{{\nu}^{\prime}_{0}}_{\textsc{omt}}(\mathbf{x}) \| \nu_{0}(\mathbf{x}) d\mathbf{x} &= \iint \|T^{\nu_0}_{\textsc{omt}}(\mathbf{x}) - T^{{\nu}^{\prime}_{0}}_{\textsc{omt}}(\mathbf{x}) \| \pi^{\prime}(\mathbf{x}, \mathbf{x}^{\prime}) d\mathbf{x} d\mathbf{x}^{\prime}  \\
&\le \iint \|T^{\nu_0}_{\textsc{omt}}(\mathbf{x}) - T^{{\nu}^{\prime}_{0}}_{\textsc{omt}}(\mathbf{x}^{\prime})\| \pi^{\prime}(\mathbf{x}, \mathbf{x}^{\prime}) d\mathbf{x} d\mathbf{x}^{\prime} \\
&\quad + \iint \|T^{\nu^{\prime}_0}_{\textsc{omt}}(\mathbf{x}) - T^{{\nu}^{\prime}_{0}}_{\textsc{omt}}(\mathbf{x}^{\prime}) \| \pi^{\prime}(\mathbf{x}, \mathbf{x}^{\prime}) d\mathbf{x} d\mathbf{x}^{\prime} \ .
\end{align*}
According to Proposition~\ref{prop:local_lipsch_OMT}, $T^{\nu^{\prime}_0}_{\textsc{omt}}$ is Lipschitz continuous with constant $L_{\nu^{\prime}_0}$ (Eq.~\ref{eq:l_nu}). Then
\begin{align}
\label{eq:bound_1}
 \iint \|T^{\nu^{\prime}_0}_{\textsc{omt}}(\mathbf{x}) - T^{{\nu}^{\prime}_{0}}_{\textsc{omt}}(\mathbf{x}^{\prime}) \| \pi^{\prime}(\mathbf{x}, \mathbf{x}^{\prime}) d\mathbf{x} d\mathbf{x}^{\prime} & \leq L_{\nu^{\prime}_0}  \iint \| \mathbf{x} - \mathbf{x}^{\prime} \| \pi^{\prime}(\mathbf{x}, \mathbf{x}^{\prime}) d\mathbf{x} d\mathbf{x}^{\prime}\\ 
 & \leq L_{\nu^{\prime}_0} \mathcal{W}_2(\nu_{0}, \nu^{\prime}_{0}) \ .
\end{align}
Using  the results of Proposition~\ref{prop:gen_OMT_tilt_stability} and derivation in Eq.~\ref{eq:final_rho_nu_bound}, we have
\begin{align}
\label{eq:bound_2}
   \mathbb{E}_{\pi^{\prime}} \left[ \|T^{\nu_0}_{\textsc{omt}}(\mathbf{x}) - T^{{\nu}^{\prime}_{0}}_{\textsc{omt}}(\mathbf{x}^{\prime}) \| \right] &\le a'_0 \mathcal{W}_2(\nu_0, \nu^{\prime}_0) + b'_0 \mathcal{W}^{1/2}_2(\nu_0, \nu^{\prime}_0)
\end{align}
Finally, combining the bounds derived in Eqs.~\ref{eq:bound_1} and \ref{eq:bound_2}, we obtain
\begin{align*}
    \int \|T^{\nu_0}_{\textsc{omt}}(\mathbf{x}) - T^{{\nu}^{\prime}_{0}}_{\textsc{omt}}(\mathbf{x}) \| d\nu_{0}(\mathbf{x}) &\leq (a'_0 + L_{\nu^{\prime}_0} )\mathcal{W}_2(\nu_0, \nu^{\prime}_0)  + (b'_0 \mathcal{W}_2(\nu_0, \nu^{\prime}_0))^{1/2}
\end{align*}
Setting $a_0 := a'_0 + L_{\nu^{\prime}_0} < \infty$ and $b_0 = b'_0 $ (defined in Eq.~\ref{eq:upbound_a_b}) completes the proof.
\end{proof}
\subsection{Additional Derivations}
\label{sec:add_deriv}
Considering $x =  \| T^{\rho \rightarrow \nu_1}_{\textsc{omt}}(\mathbf{z}) - T^{\nu_0 \rightarrow \nu_1}_{\textsc{omt}}(\mathbf{x}) \|$, the inequality in Eq.~\ref{eq:mix_exp_upb} can be simplified as 
\begin{gather*}
    x^2 - b x - c \le 0
\end{gather*}
Therefore
\begin{gather*}
    x \le b/2  + 1/2\sqrt{b^2 + 4c}
\end{gather*}
Using the property $\sqrt{f^2 + g^2} \le f + g$, $f \ge 0$, and $g \ge 0$:
\begin{gather*}
    x \le \displaystyle b + c^{1/2} \ .
\end{gather*}
\setcounter{coroll}{0}
\subsection{Gaussian OMT}
\label{sec:app_GOMT}
\begin{coroll}
Let \( \nu_0 \in G_{K_0}(\mathbb{R}^d) \) and \( \nu_1 \in G_{K_1}(\mathbb{R}^d) \) be two Gaussian mixture models (GMMs) in \( \mathbb{R}^d \) with \( K_0 \) and \( K_1 \) components, respectively. Then, the optimal mixture transport map between \( \nu_0 \) and \( \nu_1 \) is itself a Gaussian mixture model with $K$ components, where $K = K_0K_1$.
\end{coroll}
\begin{proof}
According to {Theorem}~\ref{theorem:app_uniqGOP}, the optimal mixture transport policy between each pair $\mu^i_0(\mathbf{x}), \mu^j_1(\mathbf{y})$, is independent of the weight variable and is given by:
\[
 dp_{ij}^*(\mathbf{x}, \mathbf{y}) = \displaystyle{\exp{\left(\frac{\varphi_{ij}(\mathbf{x}) + \psi_{ij}(\mathbf{y}) - \| \mathbf{x}-\mathbf{y} \|^2_2}{\varepsilon_1} \right)}} d\mu_{0_i}(\mathbf{x}) d\mu_{1_j}(\mathbf{y}),
\]
where $\phi_{ij}$ and $\psi_{ij}$ are Lagrange multipliers defined by:
\begin{eqnarray}
\varphi_{ij}(\mathbf{x}) &=& \displaystyle{- \varepsilon_1 \log{\int_{\mathcal{Y}} \exp{\left(\frac{\psi_{ij}(\mathbf{y}) - \| \mathbf{x}-\mathbf{y} \|^2_2}{\varepsilon_1} \right)}d\mu_{1_j}(\mathbf{y})}}, \nonumber \\
 \psi_{ij}(\mathbf{y}) &=& \displaystyle{- \varepsilon_1 \log{\int_{\mathcal{X}} \exp{\left(\frac{\varphi_{ij}(\mathbf{x}) - \| \mathbf{x}-\mathbf{y} \|^2_2}{\varepsilon_1} \right)}d\mu_{0_i}(\mathbf{x})}}. \nonumber 
\end{eqnarray}
For $\mu^i_0(\mathbf{x}) = \mathcal{N}(\mathbf{x} | \mathbf{m}_{i_{\mathbf{x}}}, \Sigma_{i_{\mathbf{x} \mathbf{x}}})$ and $\mu^j_1(\mathbf{y}) = \mathcal{N}(\mathbf{y} | \mathbf{m}_{j_y}, \Sigma_{j_{\mathbf{y} \mathbf{y}}})$, it was shown in \cite{janati2020entropic} that $\varphi_{ij}$ and $\psi_{ij}$ admit closed-form solutions in the form of quadratic functions as follows ( {Proposition 1} in \cite{janati2020entropic}).
\begin{eqnarray}
\label{eq:dual_form_phi_psi}
    \varphi_{ij}(\mathbf{x}) &=& - (\mathbf{x} - \mathbf{m}_{i_{\mathbf{x}}})^T U_{ij} (\mathbf{x} - \mathbf{m}_{i_{\mathbf{x}}}), \hspace{.2in} U_{ij} = \Sigma_{j_{\mathbf{y} \mathbf{y}}} \left(\Sigma^{\varepsilon_1}_{ij} + \varepsilon_1 \mathbf{I}_d \right)^{-1} - \mathbf{I}_d \nonumber \\
     \psi_{ij}(\mathbf{y}) &=& - (\mathbf{y} - \mathbf{m}_{j_y})^T V_{ij} (\mathbf{y} - \mathbf{m}_{j_y}), \hspace{.2in} V_{ij} = \left(\Sigma^{\varepsilon_1}_{ij} + \varepsilon_1 \mathbf{I}_d \right)^{-1} \Sigma_{i_{\mathbf{x} \mathbf{x}}} - \mathbf{I}_d 
\end{eqnarray}
where 
\begin{gather}
\label{eq:sigma_eps_ij}
\Sigma^{\varepsilon_1}_{ij} = \displaystyle{ \frac{1}{2}\Sigma_{i_{\mathbf{x} \mathbf{x}}}^{\frac{1}{2}} \Gamma^{\varepsilon_1}_{ij} \Sigma_{i_{\mathbf{x} \mathbf{x}}}^{-\frac{1}{2}} -\frac{\varepsilon_1}{4} I_d} \ ,  \\
\Gamma^{\varepsilon_1}_{ij} = \displaystyle{(4\Sigma_{i_{\mathbf{x} \mathbf{x}}}^{\frac{1}{2}}\Sigma_{j_{\mathbf{y} \mathbf{y}}}\Sigma_{i_{\mathbf{x} \mathbf{x}}}^{\frac{1}{2}} + \frac{\varepsilon_{1}^2}{4}I_d)^{\frac{1}{2}}}. 
\end{gather}
Accordingly, the closed-form unique solution for $dp^*_{ij}$ can be obtained as:
\begin{gather}
\label{eq:jointPDF_opt}
    p^*_{ij}(\mathbf{x}, \mathbf{y}) = \mathcal{N}\left(\begin{bmatrix}
\mathbf{x}\\
\mathbf{y}
\end{bmatrix} \arrowvert \begin{bmatrix}
\mathbf{m}_{i_{\mathbf{x}}}\\
\mathbf{m}_{j_{\mathbf{y}}}
\end{bmatrix}, \begin{bmatrix}
\Sigma_{i_{\mathbf{x} \mathbf{x}}} & \Sigma^{\varepsilon_1}_{ij}\\
\Sigma^{{\varepsilon_1}^T}_{ij} & \Sigma_{j_{\mathbf{y} \mathbf{y}}}
\end{bmatrix} \right)
\end{gather}
Therefore, the optimal mixture transport policy is itself a GMM, given by:
\begin{eqnarray}
\label{eq:omt_finalform}
\pi_{\textsc{omt}}(\mathbf{x}, \mathbf{y}) = \displaystyle{\sum^K_{i,j}} \omega_{ij} p_{ij}(\mathbf{x}, \mathbf{y}) = 
\displaystyle{\sum_{i,j} \omega_{ij} \mathcal{N}\left(\begin{bmatrix}
\mathbf{x}\\
\mathbf{y}
\end{bmatrix} \arrowvert \begin{bmatrix}
\mathbf{m}_{i_{\mathbf{x}}}\\
\mathbf{m}_{i_{\mathbf{y}}}
\end{bmatrix}, \begin{bmatrix}
\Sigma_{i_{\mathbf{x} \mathbf{x}}} & \Sigma^{\varepsilon_1}_{ij}\\
\Sigma^{{\varepsilon_1}^T}_{ij} & \Sigma_{j_{\mathbf{y} \mathbf{y}}}
\end{bmatrix} \right)}.
\end{eqnarray}
where $\omega_{ij}$ is given by Eq.~\ref{eq:omega_opt}.
\end{proof}

%% file: appendix/2-expfamily.tex
\section{Transport maps for exponential-family mixtures}
\label{sec:mix_exp_derivations}
This section provides explicit closed-form derivations for the exponential family components utilized in OMT.
\subsection{Transport map in Laplace regime (p=q=1)}
Consider univariate probability distributions under a quadratic cost function, for which the optimal transport map is given by $T_{ij}(x) = F_j^{-1}(F_i(x))$. Following the formulation in Eq.~\ref{eq:exp_family}, the univariate exponential components, $\mu_i(x)$ and $\mu_j(y)$, with shape parameters $p_i=q_i=p_j=q_j=1$, are defined as
\begin{align}
\label{eq:uni_double_exp}
\mu_{i}(x) = \begin{cases} a_{i} e^{- (x - b_i)c_i}  & \text{if } x \geq b_i \\
a_{i} e^{(x - b_i)d_i}  & \text{otherwise}
\end{cases} \ ,  \hspace{.25in} 
\mu_{j}(y) = \begin{cases} a_{j} e^{- (y - b_j)c_j}  & \text{if } y \geq b_j \\
a_{j} e^{(y - b_j)d_j}  & \text{otherwise}
\end{cases} \ .
\end{align}
To ensure that $\mu_i(x)$ integrates to 1 and defines a valid probability density function, the normalization constant $a_i$ must satisfy $\int_{-\infty}^{\infty} \mu_i(x) dx = 1$, which yields:\begin{align*}a_i = \frac{c_i d_i}{c_i + d_i} \ .\end{align*}The associated cumulative distribution function (CDF), $F_i(x)$, is given by:
\begin{align*}
    F_i(x) = \begin{cases} 1 - \frac{a_i}{c_i}e^{-c_i(x - b_i)} & \text{if } x \geq b_i  \\ \frac{a_i}{d_i} e^{d_i(x - b_i)} & \text{otherwise} \end{cases}
\end{align*}
which can be inverted to obtain the quantile function for $u \in [0, 1]$:
\begin{align*}
    F_i^{-1}(u) = \begin{cases} b_i - {c_i}^{-1} \ln\left(\frac{c_i(1 - u)}{a_i}\right) & \text{if } u \geq \frac{a_i}{d_i}  \\ b_i + d_i^{-1}\ln\left(\frac{d_i u}{a_i}\right) & \text{otherwise} \end{cases} \ .
\end{align*}
Consequently, by composing the quantile function of the target component with the source CDF, the optimal transport map $T_{ij}(x)$ between the distributions $\mu_i(x)$ and $\mu_j(y)$ is expressed piecewise as:
\begin{align}
\label{eq:map_exp_ij_onedim}
T_{ij}(x) = F^{-1}_j(F_i(x)) = \begin{cases} b_j + \frac{c_i}{c_j}(x - b_i) - c_j^{-1} \ln{\frac{a_i c_j}{a_j c_i}}  & \text{if } x \geq \max\{b_i, b_i - c_i^{-1} \ln{\frac{a_j c_i}{a_i c_j}}\}  \\ 
b_j + d_j^{-1} \ln\left(\frac{d_j (1 - \frac{a_i}{c_i} e^{- c_i(x - b_i)})}{a_j}\right) & \text{if } b_i \leq x < b_i - c_i^{-1} \ln{\frac{a_j c_i}{a_i c_j}}  \\
b_j - c_j^{-1} \ln\left(\frac{c_j (1 - \frac{a_i}{d_i} e^{d_i(x - b_i)})}{a_j}\right) & \text{if } b_i + d_i^{-1} \ln{\frac{a_j d_i}{d_j a_i}} \leq x < b_i  \\
b_j + \frac{d_i}{d_j}(x - b_i) + d_j^{-1} \ln{\frac{a_i d_j}{a_j d_i}} & \text{if } x<\min\{b_i, b_i+ d_i^{-1} \ln{\frac{a_j d_i}{d_j a_i}}\} \end{cases} \ .
\end{align}
By aggregating the transport maps between all pairs of components, the overall OMT map for the mixture distribution is formulated as:
\begin{align}
T^{\nu_0 \rightarrow \nu_1}_{\textsc{omt}}(\mathbf{x}) = \displaystyle\sum_{i,j} \tilde{\omega}_{ij}(\mathbf{x}) F^{-1}_{1_j}(F_{0_i}(\mathbf{x})) \ .
\end{align}
\subsection{Transport map in quadratic regime ($p=q=2$)}
For mixtures involving Gaussian-like components with quadratic exponential decay, the univariate density functions are defined as:
\begin{align}
\label{eq:quad_double_exp}
\mu_{i}(x) = \begin{cases} a_{i} e^{- (x - b_i)^2c^2_i}  & \text{if } x \geq b_i \\
a_{i} e^{-(x - b_i)^2d^2_i}  & \text{otherwise}
\end{cases} \ ,  
\end{align}
where the normalization constant to ensure $\int_{-\infty}^{\infty} \mu_i(x) dx = 1$ is given by:
\begin{align*}a_i = \frac{2 c_i d_i}{\sqrt{\pi}(c_i + d_i)} \ .\end{align*}Integrating the density function yields the CDF $F_i(x)$ involving the error function $\text{erf}(x) = \frac{2}{\sqrt{\pi}} \int^x_0 e^{-t^2} dt$:
\begin{align*}
F_i(x) = \begin{cases} \frac{\sqrt{\pi} a_i}{2} (\frac{1}{d_i} + \frac{1}{c_i} \text{erf}\left( \left(x - b_i \right) c_i \right)) & \text{if } x \geq b_i  \\ \frac{\sqrt{\pi} a_i}{2} \left(\frac{1}{d_i} + \frac{1}{d_i}\text{erf}\left(\left(x - b_i\right) d_i \right) \right) & \text{otherwise} \end{cases}, 
\end{align*}
and the associated quantile function $F^{-1}_i(u)$ is defined as:
\begin{align*}
    F^{-1}_i(u) = \begin{cases} b_i + c_i^{-1} \text{erf}^{-1} \left(\frac{2 c_i}{\sqrt{\pi} a_i} u - \frac{c_i}{d_i}\right) & \text{if } u \geq \frac{\sqrt{\pi} a_i}{2 d_i}  \\ 
    b_i + d_i^{-1} \text{erf}^{-1}\left(\frac{2 d_i}{\sqrt{\pi} a_i} u - 1 \right) & \text{otherwise} \end{cases}.
\end{align*}
Composing these functions yields the optimal transport map $T{ij}(x)$ between two Gaussian-like distributions:
\begin{align}
\label{eq:map_gau_ij_onedim}
T_{ij}(x) = \begin{cases} b_j + c_j^{-1} \text{erf}^{-1} \left( \frac{a_i c_j}{a_j c_i} \text{erf}(c_i(x-b_i)) + \frac{a_i c_j}{a_j d_i} - \frac{c_j}{d_j} \right) & \text{if } x \geq \max\left\{b_i, b_i + c_i^{-1} \text{erf}^{-1} \left( \frac{a_j c_i}{a_i c_j} - \frac{c_i}{d_i} \right) \right\} \\ b_j + d_j^{-1} \text{erf}^{-1} \left( \frac{a_i d_j}{a_j c_i} \text{erf}(c_i(x-b_i)) + \frac{a_i d_j}{a_j d_i} - 1 \right) & \text{if } b_i \leq x < b_i + c_i^{-1} \text{erf}^{-1} \left( \frac{a_j c_i}{a_i c_j} - \frac{c_i}{d_i} \right) \\ b_j + c_j^{-1} \text{erf}^{-1} \left( \frac{a_i c_j}{a_j d_i} \text{erf}(d_i(x-b_i)) + \frac{a_i c_j}{a_j d_i} - \frac{c_j}{d_j} \right) & \text{if } b_i + d_i^{-1} \text{erf}^{-1} \left( \frac{a_j d_i}{a_i d_j} - 1 \right) \leq x < b_i \\ b_j + d_j^{-1} \text{erf}^{-1} \left( \frac{a_i d_j}{a_j d_i} \text{erf}(d_i(x-b_i)) + \frac{a_i d_j}{a_j d_i} - 1 \right) & \text{if } x < \min\left\{b_i, b_i + d_i^{-1} \text{erf}^{-1} \left( \frac{a_j d_i}{a_i d_j} - 1 \right) \right\} \end{cases} \ .
\end{align}
In the symmetric case where $c_i = d_i$ and $c_j = d_j$, the transport map simplifies to the linear relation:
\begin{align}T_{ij}(x) = b_j + \frac{c_i}{c_j} (x - b_i), \hspace{.5in} \forall x \in \mathbb{R} \ .
\end{align}
\subsection{Transport between Laplace and quadratic forms}
Consider a pair of components in which $\mu_i(x)$ follows a double-exponential (Laplace-type) form, while $\mu_j(y)$ is characterized by a Gaussian-like quadratic exponential form. Their corresponding density functions are given by
\begin{align}
\mu_{i}(x) = \begin{cases} a_{i} e^{- (x - b_i)c_i}  & \text{if } x \geq b_i \\
a_{i} e^{(x_i - b_i)d_i}  & \text{otherwise}
\end{cases} \ ,  \hspace{.25in}
\mu_{j}(y) = \begin{cases} a_{j} e^{- (x - b_j)^2c^2_j}  & \text{if } y \geq b_j \\
a_{j} e^{-(y - b_j)^2d^2_j}  & \text{otherwise}
\end{cases} \  .
\end{align}
The forward map from the $\mu_i(x)$ to $\mu_j(y)$, can be formulated as
\begin{align}
\label{eq:map_ij_exp_gau}
T_{ij}(x) = \begin{cases}
b_j + c_j^{-1} \text{erf}^{-1}\left( \frac{2c_j}{\sqrt{\pi} a_j} \left(1 - \frac{a_i}{c_i} e^{-c_i(x - b_i)}\right) - \frac{c_j}{d_j} \right) & \text{if } x \geq \max\{b_i, b_i - c_i^{-1} \ln \frac{\sqrt{\pi} a_j c_i}{2 a_i c_j}\} \\
b_j + d_j^{-1} \text{erf}^{-1}\left( \frac{2d_j}{\sqrt{\pi} a_j} \left(1 - \frac{a_i}{c_i} e^{-c_i(x - b_i)}\right) - 1 \right) & \text{if } b_i \leq x < b_i - c_i^{-1} \ln \frac{\sqrt{\pi} a_j c_i}{2 a_i c_j} \\
b_j + c_j^{-1} \text{erf}^{-1}\left( \frac{2c_j a_i}{\sqrt{\pi} a_j d_i} e^{d_i(x - b_i)} - \frac{c_j}{d_j} \right) & \text{if } b_i + d_i^{-1} \ln \frac{\sqrt{\pi} a_j d_i}{2 a_i d_j} \leq x < b_i \\
b_j + d_j^{-1} \text{erf}^{-1}\left( \frac{2d_j a_i}{\sqrt{\pi} a_j d_i} e^{d_i(x - b_i)} - 1 \right) & \text{if } x < \min\{b_i, b_i + d_i^{-1} \ln \frac{\sqrt{\pi} a_j d_i}{2 a_i d_j}\} \ .
\end{cases}
\end{align}
Conversely, the backward map, i.e., $T_{ji}(y)$, is given by
\begin{align}
\label{eq:map_ij_gau_exp}
   T_{ji}(y) = \begin{cases} b_i - c_i^{-1} \ln\left( \frac{c_i \sqrt{\pi} a_j (1 - \text{erf}((y-b_j)c_j))}{2 c_j a_i} \right) & \text{if } y \geq \max\left\{b_j, b_j + c_j^{-1} \text{erf}^{-1}\left( \frac{2 c_j a_i}{\sqrt{\pi} a_j d_i} - \frac{c_j}{d_j} \right)\right\} \\ 
   b_i + d_i^{-1} \ln\left( \frac{d_i \sqrt{\pi} a_j ( \frac{1}{d_j} + \frac{1}{c_j} \text{erf}((y-b_j)c_j))}{2 a_i} \right) & \text{if } b_j \leq y < b_j + c_j^{-1} \text{erf}^{-1}\left( \frac{2 c_j a_i}{\sqrt{\pi} a_j d_i} - \frac{c_j}{d_j} \right) \\ 
   b_i - c_i^{-1} \ln\left( \frac{c_i(1 - \frac{\sqrt{\pi} a_j}{2 d_j}(1 + \text{erf}((y-b_j)d_j)))}{a_i} \right) & \text{if } b_j + d_j^{-1} \text{erf}^{-1}\left( \frac{2 d_j a_i}{\sqrt{\pi} a_j d_i} - 1 \right) \leq y < b_j \\ 
   b_i + d_i^{-1} \ln\left( \frac{d_i \sqrt{\pi} a_j (1 + \text{erf}((y-b_j)d_j))}{2 d_j a_i} \right) & \text{if } y < \min\left\{b_j, b_j + d_j^{-1} \text{erf}^{-1}\left( \frac{2 d_j a_i}{\sqrt{\pi} a_j d_i} - 1 \right)\right\} \end{cases}
\end{align}
Simply, when the scale parameters are symmetric ($c_j = d_j$), the transport maps in Eqs.~\ref{eq:map_ij_exp_gau} and \ref{eq:map_ij_gau_exp} simplify to
\begin{align}
\label{eq:map_ij_exp_gau_sym}
    T_{ij}(x) = \begin{cases} b_j + c_j^{-1} \text{erf}^{-1}\left( 1 - \frac{2 a_i}{c_i} e^{-c_i(x - b_i)} \right) & \text{if } x \geq b_i \\
    b_j + c_j^{-1} \text{erf}^{-1}\left( \frac{2 a_i}{d_i} e^{d_i(x - b_i)} - 1 \right) & \text{otherwise} \end{cases} \ ,
\end{align}
\begin{align}
\label{eq:map_ij_gau_sym_exp}
    T_{ji}(y) = \begin{cases} b_i - c_i^{-1} \ln\left( \frac{c_i \left[ 1 - \text{erf}(c_j(y - b_j)) \right]}{2 a_i} \right) & \text{if } y \geq b_j + {c_j}^{-1} \text{erf}^{-1}\left( \frac{2a_i}{d_i} - 1 \right) \\ 
    b_i + d_i^{-1} \ln\left( \frac{d_i \left[ 1 + \text{erf}(c_j(y - b_j)) \right]}{2 a_i} \right) & \text{otherwise} \end{cases} \ .
\end{align}

%% file: appendix/3-moreresults.tex
\section{Synthetic datasets}
\label{sec:Appsynth_data}
\ExpVsGauss
\PathShape
\section{W2-Benchmark Task}
\label{sec:AppBenchmark}
\subsection{Experiment}
For the continuous Wasserstein-2 benchmark task, we adapted the experimental setup from the publicly available repository of \cite{korotin2021neural}. We evaluated all models across a range of dimensions ($d$) with corresponding training sample sizes ($n$). For each configuration, performance was assessed on a separate test set of $10,000$ samples. To ensure statistical robustness, every experiment was repeated 10 times with different random initializations. These same settings were also used for the ablation study on the impact of noise.

The specific dimension and sample size pairs were as follows:
\begin{itemize}
    \item for $d=2$ with $n=10,000$,
    \item for $d=16$ with $n=20,000$,
    \item for $d=64$ with $n=30,000$, 
    \item for $d \in \{128, 256\}$ with $n=50,000$.
\end{itemize}
\subsection{Stability analysis}
\label{sec:ablation}
To evaluate the stability of OT solvers, we conducted an ablation study examining the sensitivity of transport maps to input perturbations. All solvers were first optimized on noisy samples and then evaluated on the original, unperturbed test set. Noise was added to the source distribution samples during training, while the evaluation used the original test set ($N=10,000$ samples).

Here, the source samples were subjected to two distinct forms of stochastic perturbation:
\begin{itemize}
    \item Additive Gaussian (White) Noise: We added noise sampled from $\mathcal{N}(0, \sigma^2 I)$ to the source input, varying the standard deviation $\sigma \in \{0.1, 0.25, 0.5, 1.0\}$.
    \item Dropout Noise: To simulate missing data, a common issue in single-cell genomics data (e.g., gene dropout), we applied a random drop-out operation where input features were zeroed out with a probability $p \in \{0.05, 0.1, 0.25, 0.5\}$.
\end{itemize}
The performance variations shown in Figure~\ref{fig:Noise} were quantified as the relative percentage difference in MSE compared to the noise-free baseline.
\subsection{Training Configurations}
\label{sec:AppTrainingBenchmark}

All methods were implemented and trained using their official implementations within the OTT-JAX toolbox~\citep{cuturi2022optimal}, following the recommended configurations provided in the official tutorials and Wasserstein-2 benchmark repositories.

\begin{itemize}
    \item \textbf{EOT:} The entropy regularization parameter was set to $\varepsilon = 0.1$, with a maximum of $10^6$ training iterations.
    \item \textbf{W2-OT:} We employed the recommended three-layer MLP architectures for the dual potentials and used regression-based amortization combined with L-BFGS fine-tuning. Models were trained for a maximum of $10^5$ iterations, following the configuration reported in Table~4 of~\cite{amos2022amortizing}.
    \item \textbf{P{\scriptsize{ROG}}OT:} Training followed the recommended scheduler settings with $K=4$ refinement steps.
    \item \textbf{ExNOT:} We adopted the recommended architecture consisting of five-layer MLPs for each potential function, with 128 hidden units per layer. Training was performed for up to $10^5$ iterations.
    \item \textbf{ENOT:} We used the original implementation released by the authors for this benchmark and retained the same hyperparameter configuration reported in the corresponding work.
    \item \textbf{OMT ($p=q=1$):} The number of source and target mixture components was set to $K=10$ and $K=30$, respectively. Asymmetric factorized double-exponential components were independently fitted to the source and target distributions using the expectation-maximization algorithm. The transport model was then trained for a maximum of $10^5$ iterations using entropy regularization parameter $\varepsilon_2=0.01$.
    \item \textbf{OMT ($p=q=2$):} The number of source and target components was set to $K=5$ and $K=15$, respectively. Gaussian mixture components were independently fitted to the source and target data using the \texttt{scikit-learn} Python package, employing full covariance matrices for dimensions $d \leq 64$. The model was trained for up to $10^5$ iterations with entropy regularization parameters $\varepsilon_1=\varepsilon_2=0.01$.
    \item \textbf{GMM-OT:} The training procedure followed the same setup as OMT ($p=q=2$), except that entropy regularization was disabled ($\varepsilon=0$). The mixing-component coupling matrix $\Omega$ was optimized directly through Wasserstein minimization.
\end{itemize}
\AppBenchMark
\RunTime
\subsection{OMT vs. GMM-OT}
\label{sec:OMTvsGMMOT}
OMT and GMM-OT ($G\mathcal{W}_2$)~\citep{delon2020wasserstein} both operate on pairs of GMMs, but they fundamentally differ in their optimization landscapes and solution guarantees. GMM-OT fits a GMM to the coupling function to minimize the Wasserstein distance. This formulation is non-convex and may lack a unique solution, meaning multiple distinct transport plans can yield identical costs, thereby hampering interpretability. In contrast, OMT utilizes a dual-entropic regularization scheme, applying regularization to both the component transport $p_{ij}$ and the mixing weights matrix $\Omega$. This formulation offers significant computational and theoretical advantages:
\begin{itemize}
    \item Closed-form and faster updates: The regularization on component transportation allows for analytical closed-form solutions for $p_{ij}$ (Eq.~\ref{eq:p_opt}) and the transport potentials $\varphi_{ij}$ (Eq.~\ref{eq:phi_opt}) and $\psi_{ij}$ (Eq.~\ref{eq:psi_opt}). This simplifies the optimization, reducing the problem to optimizing only the mixing parameters, which can be solved efficiently using convex solvers including Sinkhorn algorithm.
    \item Uniqueness and Smoothness: The entropic term makes the objective strictly biconvex (Lemma~\ref{lemma:biconvex}). Consequently, the solution is unique (Theorem~\ref{theorem:uniqGOP}) and differentiable with respect to the inputs.
    \item Gradient Stability: Regularization discourages degenerate or overly \textit{peaky} transport plans, preventing overfitting to empirical distributions with limited samples. This ensures stable gradients with respect to the cost and input distributions. 
\end{itemize}
To empirically validate the stability offered by this regularization, we evaluate the relative changes to the transport plan under input perturbation using the $W_2$-benchmark task. We used the same source and target training samples ($\mathbf{x}_{train}$, $\mathbf{y}_{train}$) and identical GMMs (same number of components and parameters) for both methods. For each solver, we then optimize to obtain a baseline coupling map $\Omega$. Following the ablation study in ~\ref{sec:ablation}, we perturb the input data by adding Gaussian noise $\mathcal{N}(0, \sigma^2 I)$ to $\mathbf{x}_{train}$, with $\sigma \in \{0.1, 0.25, 0.5, 1.0\}$, and re-optimize the transport plans ($\tilde{\Omega}$). Here, we quantify the robustness of the methods using two metrics: (i) the percentage change in Transport Cost ($T_c$), and (ii) the relative Transport Map Deviation (TMD), defined as:
\begin{gather*}
    TMD := \displaystyle \frac{\|\tilde{\Omega} - \Omega \|_F}{\|\Omega\|_F}
\end{gather*}
Figure~\ref{fig:OMTGMMOT} shows that OMT exhibits smaller changes in the transport plan than GMM-OT, highlighting the benefit of regularization in the mixture transportation problem.
\OMTGMMOT
\section{Single-cell Data Analysis}
\label{sec:AppscRNAeq}
\subsection{Data availability}

\begin{itemize}
\item sci-Plex3 data can be downloaded from NCBI GEO ($\#$GSE139944).
 \item MERFISH Mouse Brain Receptor data is available on the Vizgen website (https://info.vizgen.com/mouse-brain-map). 
Map data release (https://info.vizgen.com/mouse-brain-map)
 \item Mouse developmental data is available through Neuroscience Multi-omic Data Archive (NeMO), (RRID:SCR-016152). The 10x scRNA-seq dataset is available at https://assets.nemoarchive.org/dat-0oyried.
 \item Mouse aging scRNA-seq is also available through NeMO, https://nemoarchive.org/, and can be accessed at https://assets.nemoarchive.org/dat-61kfys3.
\end{itemize}

\subsection{Preprocessing}
For human scRNA-seq data, sci-Plex,  we followed the same processing steps recommended in \cite{cuturi2023monge}. Genes which appear in less than $20 $cells, and cells with less that $20$ gene expressed are excluded. Then we normalized gene expression, by first normalized to counts per million (CPM) and then transformed using the formula $\log{(CPM + 1)}$. Then we whiten the data and apply PCA.

For the mouse scRNA-seq datasets, we preprocessed the raw count matrix. First, we performed library size normalization by converting counts to counts per million (CPM), followed by $\log$-transformation. For feature selection, we chose a subset of highly variable genes combined with a list of known marker genes from the mouse brain atlas.

\subsection{OMT training}
\label{sec:scTraining}
For each dataset, we first performed dimensionality reduction and then trained the OMT model on the resulting low-dimensional embeddings. The specific hyperparameters were tailored to each dataset.

\textbf{sci-Plex Dataset.} As previously described, we used PCA for dimensionality reduction. On the resulting PCA embeddings, we trained the OMT model with the number of source components set to $K_s=3$ and target components to $K_t=5$. The entropy regularization parameter for both $(\Omega, P)$ was set to $0.01$.

\textbf{MERFISH Mouse Brain Data.} For the Vizgen MERFISH brain data, similar to \cite{halmos2025hierarchical}, we only focused on transporting cell spatial coordinates from the source slice to the target one in 2-dimensional space $(x,y)$. We first fit GMMs to the source and target cells' location independently. Using the resulting components, we minimized the OMT loss to transport cells across slices. Figure~\ref{fig:AppMERFISHGMMs} shows the impact of number of GMM components on transport accuracy. Due to the dense distribution of cells, using a small number of components results in a poor alignment between the transported cells and the ground truth. Conversely, increasing the number of components, i.e., $K_s=K_t=1000$, provides significantly better alignment between the transported cells and the target.

\textbf{Mouse scRNA-seq Data.} For the mouse scRNA-seq data, we first trained a variational autoencoder (VAE) to learn a compressed cellular representation in a latent space of dimension $d_z=10$. We then trained the OMT model on these VAE embeddings. The number of components was set within a range of $5$ to $25$, with the specific value chosen based on the biological context; we typically used approximately twice the number of known cell types present at the analyzed timepoints.

\subsection{Additional results for sci-Plex}
\label{sec:AppscSciPlexFigs}
%
%
\AppsciplexPCA
%
\subsection{Additional results for MERFISH}
\label{sec:AppscMERFISHFigs}
This section provides additional qualitative evaluations for the Vizgen MERFISH Mouse Brain dataset, complementing the quantitative benchmarks presented in the main text. Figure~\ref{fig:AppMERFISHGMMs} displays the impact of varying the number of mixture components ($K_s$ and $K_t$, ranging from 5 to 1000) across the source and target slices, alongside the corresponding transported sample distributions. Figure~\ref{fig:AppMERFISHgenes} shows a detailed visual comparison of the imputed spatial gene expression against the ground-truth target expressions for five key marker genes (\textit{Slc17a7}, \textit{Gad1}, \textit{Grm4}, \textit{Olig1}, and \textit{Peg10}). 
\AppMERFISHGMMs
\AppMERFISHgenes
\subsection{Additional results for Mouse scRNA-seq data}
\label{sec:AppscMouseFigs}
This section presents additional qualitative evaluations on the developmental and aging mouse datasets, further supporting the results demonstrated in Figure~\ref{fig:scRNAseq} of the main text.
\AppDevPerformance
\AppDevGenes
\AppAgePerformance
\AppAgeGenes
\clearpage
\section{Image Datasets}
\label{sec:AppImageTasks}
\subsection{Translation Task}
To further demonstrate the applicability of the the proposed OMT framework beyond tabular data, we apply it to an unpaired image-to-image translation task using two benchmark datasets: MNIST~\citep{lecun1998mnist} and CIFAR-10~\citep{krizhevskycifar}. In MNIST, the task involves translating images of one digit into another (e.g., learning transport maps such as $T: 1 \to 7$). Similarly, in CIFAR-10, the goal is to translate images from one semantic class (e.g., airplane) into another (e.g., bird). Although OMT can in principle be applied directly to raw image data, the resulting mappings are not semantically meaningful and fail to capture class-level translations. 

Similar to our approach with scRNA-seq data, our method for image datasets in translation tasks involves a two-step process. We first train a deep neural network to learn a low-dimensional representation of the images, and then train the OMT model on these resulting embeddings. 
For the MNIST dataset, we employed a convolutional VAE featuring a dual-decoder design, where each decoder reconstructs images for the source and target domains, respectively. The latent dimension for this network was set to $d_z=10$. The full architecture is detailed in Table~\ref{tab:NetworkMNIST}.

For the CIFAR-10 dataset, we utilized the DoubleRessNet architecture, described in Table~\ref{tab:NetworkCIFAR}, with a latent space dimension of $d_z=32$.
For all experiments on these datasets, the subsequent OMT model was trained using $10$ components for both the source and target measure and $\varepsilon=0.01$. We found the number of components choice to be robust, as preliminary experiments with other values did not yield significant changes in the final results.

Figure~\ref{fig:AppImagTran} illustrate representative examples of these class-to-class translations for test images in MNIST and CIFAR-10, respectively. Quantitative evaluation of the generated translations is reported in Table~\ref{tab:fid} using the widely adopted Fréchet Inception Distance (FID). These experiments highlight that, OMT can be effectively extended to image-based applications as well.
For context and to benchmark our performance against established OT based approaches, we also report the FID scores for WGAN~\citep{arjovsky2017wasserstein} and WGAN-GP~\citep{gulrajani2017improved}. The results show that OMT performs in a similar range to WGAN on CIFAR-10, while outperforming both WGAN and WGAN-GP on the MNIST dataset.
See Appendix~\ref{sec:AppImageTasks} for further implementation details, including the autoencoder architectures used for dimensionality reduction and the hyperparameters for OMT. 

\subsection{Alignment Task}
\label{sec:AppImageNet}
To demonstrate the applicability of OMT to large-scale, high-dimensional data, we perform an alignment task on ImageNet~\citep{deng2009imagenet}, following \cite{halmos2025hierarchical}. 
For the sake of comparability, we follow the experimental setting provided in \cite{halmos2025hierarchical}, without using any dimensionality reduction. We selected ImageNet data with $1,281,166$ images (resized to $224 \times 224$) from the ImageNet ILSVRC dataset~\cite{deng2009imagenet}. The alignment was performed on high-dimensional embeddings obtained using a pretrained ResNet-50~\citep{he2016deep} (accessible through PyTorch). We use 2048-d embeddings from a pretrained ResNet-50~\citep{he2016deep} and construct source and target subsets via random sampling. The dataset was randomly split into two sets of $640,500$ images for the source and target. We then optimized OMT with $1,000$ components on each domain to learn the alignment function between source and target images. Since the mixture transport function is optimized in a high-dimensional embedding space, and learning full covariance is challenging, we used a diagonal covariance structure for all components.
\NetworkMNIST
\NetworkCIFAR
\AppImagTran
\FID